\documentclass{article}
\usepackage{ijcai24}

\usepackage{times}
\usepackage{soul}
\usepackage{url}
\usepackage[hidelinks]{hyperref}
\usepackage[utf8]{inputenc}
\usepackage[small]{caption}
\usepackage{graphicx}
\usepackage{amsmath}
\usepackage{amsthm}
\usepackage{booktabs}
\usepackage{algorithm}
\usepackage{algorithmic}
\usepackage[switch]{lineno}
\usepackage{amssymb}
\usepackage{mathtools}
\usepackage{subfigure}
\usepackage[capitalize,noabbrev]{cleveref}
\usepackage{bm}
\usepackage[table]{xcolor} 
\usepackage{multirow}
\usepackage{multicol,lipsum}
\newtheorem{example}{Example}
\newtheorem{theorem}{Theorem}
\newtheorem{lemma}{Lemma}
\newtheorem{assumption}{Assumption}
\title{Global Optimality of Single-Timescale Actor-Critic under Continuous State-Action Space: A Study on Linear Quadratic Regulator}

\author{
Xuyang Chen$^1$
\and
Jingliang Duan$^2$\and
Lin Zhao$^{1}$\\
\affiliations
$^1$National University of Singapore\\
$^2$University of Science and Technology Beijing\\
\emails
chenxuyang@u.nus.edu,
duanjl@ustb.edu.cn,
elezhli@nus.edu.sg
}

\begin{document}

\maketitle

\begin{abstract}
Actor-critic methods have achieved state-of-the-art performance in various challenging tasks. However, theoretical understandings of their performance remain elusive and challenging. Existing studies mostly focus on practically uncommon variants such as double-loop or two-timescale stepsize actor-critic algorithms for simplicity. These results certify local convergence on finite state- or action-space only. We push the boundary to investigate the classic single-sample single-timescale actor-critic on continuous (infinite) state-action space, where we employ the canonical linear quadratic regulator (LQR) problem as a case study. We show that the popular single-timescale actor-critic can attain an epsilon-optimal solution with an order of epsilon to -2 sample complexity for solving LQR on the demanding continuous state-action space. Our work provides new insights into the performance of single-timescale actor-critic, which further bridges the gap between theory and practice.
% , marking an inaugural achievement in the community.
\end{abstract}

\section{Introduction}\label{intro}

\begin{table*}[t]
\centering
\begin{tabular}{c|cc|c|c}
\hline
\multirow{2}{*}{Reference} & \multicolumn{2}{c|}{Setting}    & \multirow{2}{*}{Optimality}  & \multirow{2}{*}{Sample Complexity} \\ \cline{2-3}
                  & \multicolumn{1}{c|}{State Space} & action space &  &                   \\ \hline
   \cite{chen2021closing}           & \multicolumn{1}{c|}{infinite} & finite &  local &    $\mathcal{O}(\epsilon^{-2})$               \\ \hline
     \cite{olshevsky2023small}             & \multicolumn{1}{c|}{finite} & finite &  local &  $\mathcal{O}(\epsilon^{-2})$                 \\ \hline
 \cite{chen2022finite}               & \multicolumn{1}{c|}{infinite} & finite &  local &  $\mathcal{O}(\epsilon^{-2})$                 \\ \hline
\rowcolor{blue!15}
     This Paper          & \multicolumn{1}{c|}{infinite} & infinite &  global & $\mathcal{O}(\epsilon^{-2})$                  \\ \hline
\end{tabular}
\caption{Comparison with other single-sample single-timescale actor-critic algorithms}
\label{table1}
\end{table*}

Actor-critic (AC) methods achieved substantial success in solving many difficult reinforcement learning (RL) problems \cite{lecun2015deep,mnih2016asynchronous,silver2017mastering}. 
% Compared to actor-only methods such as REINFORCE \cite{williams1992simple} estimate policy gradient by sample trajectory rollouts, 
In addition to a policy update, AC methods employ a parallel critic update to bootstrap the Q-value for policy gradient estimation, which often enjoys reduced variance and fast convergence in training.

Despite the empirical success, theoretical analysis of AC in the most practical form remains challenging. Existing works mostly focus on either the double-loop or the two-timescale variants. In double-loop AC, the actor is updated in the outer loop only after the critic takes sufficiently many steps to have an accurate estimation of the Q-value in the inner loop \cite{yang2019provably,kumar2019sample,wang2019neural}. Hence, the convergence of the critic is decoupled from that of the actor. The analysis is separated into a policy evaluation sub-problem in the inner loop and a perturbed gradient descent in the outer loop. In two-timescale AC, the actor and the critic are updated simultaneously in each iteration using stepsizes of different timescales. The actor stepsize (denoted by $\alpha_t$ in the sequel) is typically smaller than that of the critic (denoted by $\beta_t$ in the sequel), with their ratio going to zero as the iteration number goes to infinity (i.e., $\lim_{t\rightarrow \infty}{\alpha_t}/{\beta_t} = 0$). The two-timescale allows the critic to approximate the correct Q-value asymptotically. This special stepsize design essentially decouples the analysis of the actor and the critic.

The aforementioned AC variants are considered mainly for the ease of analysis, which, however, are uncommon in practical implementations. In practice, the single-timescale AC, where the actor and the critic are updated simultaneously using constantly proportional stepsizes (i.e., with ${\alpha_t}/{\beta_t} = c>0$), is more favorable due to its simplicity of implementation and empirical sample efficiency~\cite{schulman2015trust,mnih2016asynchronous}. For online learning, 
the actor and the critic update only once with a single sample in each iteration using proportional stepsizes. This single-sample single-timescale AC is the most classic AC algorithm extensively discussed in the literature and introduced in \cite{sutton2018reinforcement}. However, its analysis is significantly more difficult than other variants, primarily due to the more inaccurate value estimation of the critic update and the stronger coupling between critic and actor. More recent works \cite{chen2021closing,olshevsky2023small,chen2022finite} investigated its local convergence and on the finite state- or action-space only. Given that most practical applications in real world are of continuous state-action space, it is demanding to ask the following challenging question:

{\em Can the classic single-sample single-timescale AC find a global optimal policy on continuous state-action space?}

To this end, we take a first step to consider the Linear Quadratic Regulation (LQR), a fundamental continuous state-action space control problem that is commonly employed to study the performance and the limits of RL algorithms   \cite{fazel2018global,yang2019provably,tu2018least,duan2023optimization}. We analyze the same classic single-sample single-timescale AC algorithm as those studied in the references listed in \Cref{table1}. As compared in \Cref{table1}, our result is the first to show the global optimality on continuous (infinite) state-action space, while achieving the sample complexity as the previous studies.

Specifically, we consider the time-average cost, which is a more common case for LQR formulation and more difficult to analyze than the discounted cost. The single-sample single-timescale AC algorithm for solving LQR consists of three parallel updates in each iteration: the cost estimator, the critic, and the actor. Unlike the aforementioned double-loop or two-timescale, there is no specialized design in single-sample single-timescale AC that facilitates a decoupled analysis of its three interconnected updates. In fact, it is both conservative and difficult to bound the three iterations separately. Moreover, the existing perturbed gradient analysis can no longer be applied to establish the convergence of the actor either.

To tackle these challenges in analysis, we instead directly bound the overall interconnected iteration system altogether, without resorting to conservative decoupled analysis. In particular, despite the inaccurate estimation in all three updates, we prove the estimation errors diminish to zero if the (constant) ratio of the stepsizes between the actor and the critic is below a threshold. The identified threshold provides new insights into the practical choices of the stepsizes for single-timescale AC.

Compared with other single-sample single-timescale AC (see \Cref{table1}), the state-action space we study is infinite. We emphasize that moving from finite to infinite state-action space is highly nontrivial and requires significant analysis. Existing works \cite{chen2021closing,chen2022finite} derived key intermediate results such as many Lipschitz constants relying on the finite size of the state-action space ($|\mathcal{S}|, |\mathcal{A}|$). These results however become immaterial in the infinite state-action space scenario. Some other analysis \cite{olshevsky2023small} concatenates all state-action pairs to create a finite-dimensional feature matrix. However, this will not be possible when the state-action space is infinite. Consequently, existing analyses are not applicable in our context.

We also distinguish our work from other model-free RL algorithms for solving LQR in \Cref{table2}, in addition to AC methods. The zeroth-order methods and the policy iteration method are included for completeness. In particular, we note that~\cite{zhou2023single} analyzed the single-timescale AC under a multi-sample setting, where the critics are updated by the least square temporal difference (LSTD) estimator. The idea is still to obtain an accurate policy gradient estimation at each iteration by using sufficient samples (in LSTD), and then follow the common perturbed gradient analysis to prove the convergence of the actor, which decouples the convergence analysis of the actor and the critic. Moreover, the analysis requires a strong assumption on the uniform boundedness of the critic parameters. In comparison, our analysis does not require this assumption and considers the more classic and challenging single-sample setting which is also considered by the previous works as listed in~\Cref{table1}.

Overall, our contributions are summarized as follows:
% Moreover, the previous decoupling methods always exhibit a framework for bounding the critic first and then the actor based on it\cite{yang2019provably,wu2020finite,chen2021closing}. Inspired by \cite{olshevsky2023small}, we consider an interconnected iteration system and bounds the critic and the actor simultaneously, without the conservative decoupling methods. We summarize our main contributions as follows:

$\bullet$ Our work furthers the theoretical understanding of AC on continuous state-action space, which represents the most practical usages. We for the first time show that the single-sample single-timescale AC can provably find the $\epsilon$-accurate global optimum with a sample complexity of $\mathcal{O}(\epsilon^{-2})$ for tasks with unbounded continuous state-action space. The previous works consider the more restricted finite state-action space settings with only local convergence guarantee \cite{chen2021closing,olshevsky2023small,chen2022finite}.

$\bullet$ We also contribute to the work of RL on continuous control tasks. It is novel that even with the actor updated by a roughly estimated gradient, the single-sample single-timescale AC algorithm can still find the global optimal policy for LQR, under general assumptions. Compared with all other model-free RL algorithms for solving LQR (see \Cref{table2}), our work adopts the simplest single-sample single-timescale structure, which may serve as the first step towards understanding the limits of AC methods on continuous control tasks. 
In addition, compared with the state-of-the-art double-loop AC for solving LQR~\cite{yang2019provably}, we improve the sample complexity from $\mathcal{O}(\epsilon^{-5})$ to $\mathcal{O}(\epsilon^{-2})$.
We also show the algorithm is much more sample-efficient empirically compared to a few classic works in Experiments, which unveils the practical wisdom of AC algorithm.

% $\bullet$ Technically, we provide a new proof framework that can establish the finite-time convergence for single-timescale AC. In the finite-time analysis of double-loop AC \cite{yang2019provably} and two-timescale AC \cite{wu2020finite}, the previous techniques hinge on decoupling the analysis of actor and critic, establishing the convergence of critic first and then the convergence of actor consequently. The novelty of our proof framework is that we formulate the estimation errors of the time-average cost, the critic, and the natural policy gradient into an interconnected iteration system and establish the convergence for them simultaneously rather than separately. This proof framework may provide new insights for finite-time analysis of other single-timescale algorithms under continuous state-action space. 

\begin{table*}[t]
\centering
\begin{tabular}{c|c|cc}
\hline
Reference & Algorithm & \multicolumn{2}{c}{Structure}                      \\ \hline
\cite{fazel2018global} & zeroth-order & \multicolumn{2}{c}{\multirow{3}{*}{double-loop}}     \\ \cline{1-2}
\cite{malik2019derivative} & zeroth-order & \multicolumn{2}{c}{}                       \\ \cline{1-2}
\cite{yang2019provably} & actor-critic & \multicolumn{2}{c}{}                       \\ \hline
\cite{krauth2019finite} & policy iteration & \multicolumn{2}{c}{multi-sample} \\ \hline 
\cite{zhou2023single} & actor-critic & \multicolumn{1}{c|}{single-timescale}           & multi-sample \\ \hline 
\rowcolor{blue!15}
This paper & actor-critic & \multicolumn{1}{c|}{single-timescale}                   & single-sample \\ \hline
\end{tabular}
\caption{Comparison with other model-free RL algorithms for solving LQR.}
    \label{table2}
\end{table*}

\subsection{Related Work}\label{related}
In this section, we review the existing works that are most relevant to ours. 

\textbf{Actor-Critic methods.} The AC algorithm was proposed by \cite{konda1999actor}.~\cite{kakade2001natural} extended it to the natural AC algorithm. The asymptotic convergence of AC algorithms has been well established in \cite{kakade2001natural,bhatnagar2009natural,castro2010convergent,zhang2020provably}. Many recent works focused on the finite-time convergence of AC methods. Under the double-loop setting, \cite{yang2019provably} established the global convergence of AC methods for solving LQR. \cite{wang2019neural} studied the global convergence of AC methods with both the actor and the critic being parameterized by neural networks. \cite{kumar2019sample} studied the finite-time local convergence of a few AC variants with linear function approximation. Under the two-timescale AC setting, \cite{wu2020finite,xu2020non} established the finite-time convergence to a stationary point at a sample complexity of $\mathcal{O}(\epsilon^{-2.5})$. Under the single-timescale setting, all the related works \cite{chen2021closing,olshevsky2023small,chen2022finite} have been reviewed in the Introduction.
% Under the single-timescale setting, \cite{fu2020single} and \cite{zhou2023single} consider a special single-timescale AC  implementation and obtain a global optimum with LSTD updates for critic. \cite{chen2021closing} analyzed the  can only attain a stationary point, where the analysis in \cite{chen2021closing} highly relies on the Lipschitz continuity assumption on the Jacobian of the stationary distribution which the authors can not give a justification and \cite{olshevsky2023small} only consider the tabular case.

\textbf{RL algorithms for LQR.} RL algorithms in the context of LQR have seen increased interest in the recent years. These works can be mainly divided into two categories: model-based methods \cite{dean2018regret,mania2019certainty,cohen2019learning,dean2020sample} and model-free methods. Our main interest lies in the model-free methods. Notably, \cite{fazel2018global} established the first global convergence result for LQR under the policy gradient method using zeroth-order optimization. \cite{krauth2019finite} studied the convergence and sample complexity of the LSTD policy iteration method under the LQR setting.
On the subject of adopting AC to solve LQR, \cite{yang2019provably} provided the first finite-time analysis with convergence guarantee and sample complexity under the double-loop setting. \cite{zhou2023single} considered the multi-sample (LSTD) and single-timescale setting. For the more practical yet challenging single-sample single-timescale AC, there is no such theoretical guarantee so far, which is the focus of this paper.

\textbf{Notation.} We use non-bold letters to denote scalars and use lower and upper case bold letters to denote vectors and matrices respectively. We also use $\Vert \bm\omega \Vert$ to denote the $\ell_2$-norm of a vector $\bm\omega$, $\Vert \bm A \Vert$ to denote the spectral norm of a matrix $\bm A$, and $\Vert \bm A\Vert_F$ to denote the Frobenius norm of a matrix $\bm A$. We use $\rm Tr(\cdot)$ to denote the trace of a matrix. For any symmetric matrix $\bm M\in \mathbb{R}^{n\times n}$, let $\text{svec}(\bm M)\in\mathbb{R}^{n(n+1)/2}$ denote the vectorization of the upper triangular part of $\bm M$ such that $\Vert \bm M \Vert^2_F = \langle \text{svec}(\bm M),\text{svec}(\bm M) \rangle $. Besides, let $\text{smat}(\cdot)$ denote the inverse of $\text{svec}(\cdot)$ so that $\text{smat}(\text{svec}(\bm M))=\bm M$. Finally, we denote by $\bm A\otimes_s \bm B$ the symmetric Kronecker product \cite{schacke2004kronecker} of two matrices $\bm A$ and $\bm B$.

\section{Preliminaries}
In this section, we introduce the AC algorithm and provide the theoretical background of LQR.
\subsection{Actor-Critic Algorithms}
We consider the reinforcement learning for the standard Markov Decision Process (MDP) defined by $(\mathcal{X},\mathcal{U},\mathcal{P},c)$, where $\mathcal{X}$ is the state space, $\mathcal{U}$ is the action space, $\mathcal{P}(\bm x_{t+1}|\bm x_t,\bm u_t)$ denotes the transition kernel that the agent transits to state $\bm x_{t+1}$ after taking action $\bm u_t$ at current state $\bm x_t$, and $c(\bm x_t,\bm u_t)$ is the running cost. A policy $\pi_{\bm\theta}(\bm u|\bm x)$ parameterized by $\bm\theta$ is defined as a mapping from a given state to a probability distribution over actions. 

In this paper, we aim to find a  policy $\pi_{\bm\theta}$ that minimizes the infinite-horizon time-average cost, which is given by
\begin{align}\label{eq2.1.1}
   J(\bm\theta) := \lim\limits_{T\to\infty}\mathbb{E}_{\bm\theta}\frac{\sum_{t=0}^Tc(\bm x_t,\bm u_t)}{T} =\mathop{\mathbb{E}}_{\bm x\sim \rho_{\bm\theta},\bm u\sim\pi_{\bm\theta}}[c(\bm x,\bm u)],
\end{align}
where $\rho_{\bm\theta}$ denotes the stationary state distribution generated by policy $\pi_{\bm\theta}$. In the time-average cost setting, the state-action value (Q-value) of policy $\pi_{\bm\theta}$ is defined as 
\begin{align*}%\label{eq2.1.2}
    Q_{\bm\theta}(\bm x,\bm u)=\mathbb{E}_{\bm\theta}[\sum\limits_{t=0}^{\infty}(c(\bm x_t,\bm u_t)-J(\bm\theta))|\bm x_0=\bm x,\bm u_0=\bm u],
\end{align*}
which describes the accumulated differences between running costs and average cost for selecting $\bm u$ in
state $\bm x$ and thereafter following policy $\pi_{\bm\theta}$ \cite{sutton2018reinforcement}. Based on this definition, we can use the policy gradient theorem \cite{sutton1999policy} to express the gradient of $J(\bm\theta)$ with respect to $\bm\theta$ as
\begin{align}
\label{eq.vanilla_gradient}
    \nabla_{\bm\theta}J(\bm\theta)=\mathbb{E}_{\bm x\sim\rho_{\bm\theta},\bm u\sim\pi_{\bm\theta}}[\nabla_{\bm\theta}\log\pi_{\bm\theta}(\bm u|\bm x) Q_{\bm \theta}(\bm x,\bm u)].
\end{align}

One can also choose to update the policy using the natural policy gradient \cite{kakade2001natural}, which is given by 
\begin{equation}
\label{eq.natural_pg}
    \nabla_{\bm\theta}^N J(\bm\theta)=F(\bm\theta)^{\dagger}\nabla_{\bm\theta}J(\bm\theta).
\end{equation}
where 
\begin{align*}
    F(\bm\theta)=\mathbb{E}_{\bm x\sim\rho_{\bm\theta},\bm u\sim\pi_{\bm\theta}}[\nabla_{\bm\theta}\log\pi_{\bm\theta}(\bm u|\bm x)\nabla_{\bm\theta}\log\pi_{\bm\theta}(\bm u|\bm x)^\top]
\end{align*}
is the Fisher information matrix and $F(\bm\theta)^{\dagger}$ denotes its Moore Penrose pseudoinverse. 

Optimizing $J(\bm\theta)$ in \eqref{eq2.1.1} with \eqref{eq.vanilla_gradient} requires evaluating the Q-value of the current policy $\pi_{\bm\theta}$, which is usually unknown. AC estimates both the Q-value and the policy. The critic update approximates Q-value towards the actual value of the current policy $\pi_{\bm\theta}$ using temporal difference (TD) learning \cite{sutton2018reinforcement}. The actor improves the policy to reduce the time-average cost $J(\bm\theta)$ via policy gradient descent. Note that the AC with a natural policy gradient is also known as natural AC, which is a variant of AC.

% Assuming the Q-value can be approximated by a linear function 
% \begin{align*}
%     Q_{\bm\theta}(\bm x,\bm u) \approx \bm\phi(\bm x,\bm u)^\top \bm\omega,
% \end{align*}
% where $\bm\phi(\bm x,\bm u)$ is the feature function and $\bm\omega$ is called critic. Then, the single-sample single-timescale AC under the time-average cost setting takes the following updates
% \begin{align*}
%     \textbf{TD error:}\quad  &\delta_t=c_t - n_t +\phi(x,u)^\top \omega_t-\phi(x,u)^\top \omega \\
%     \textbf{Cost estimation:}\quad &\eta_{t+1}= \eta_t+\gamma_t(c_t-\eta_t)\\
%     \textbf{Critic update:}\quad &\omega_{t+1} = \omega_t+\beta_t \delta_t \phi(x_t,u_t)\\
%     \textbf{Actor update:}\quad &\theta_{t+1}=\theta_t+\alpha_t\delta_t\nabla_{\theta}\log \pi_{\theta_t}(a_t|s_t)
% \end{align*}

\subsection{Actor-Critic for Linear Quadratic Regulator}\label{sec2.2}

In this paper, we aim to demystify the convergence property of AC by focusing on the infinite-horizon time-average linear quadratic regulator (LQR) problem:
\begin{equation}\label{eq.LQR}
    \begin{aligned}
        \mathop{\text{minimize}}\limits_{\{ \bm{u}_t\}}\quad &J(\{ \bm{u}_t\}):=\lim\limits_{T\to\infty}\frac{1}{T}\mathbb{E}[\sum\limits_{t=1}^T\bm{x_t}^\top \bm{Qx}_t+\bm{u}_t^\top \bm{Ru}_t]\\
    \text{subject to}\quad & \bm x_{t+1}=\bm{Ax}_t+\bm{Bu}_t+\bm{\epsilon}_t,
    \end{aligned}
\end{equation}
where $\bm x_t\in\mathbb{R}^d$ is the state and $\bm{u}_t\in\mathbb{R}^k$ is the control action at time $t$; $\bm{A}\in \mathbb{R}^{d\times d}$ and $\bm{B}\in \mathbb{R}^{d\times k}$ are system matrices, and the $(\bm{A},\bm{B})$-pair is stabilizable; $\bm{Q}\in\mathbb{S}^{d\times d}$ and $\bm{R}\in\mathbb{S}^{k\times k}$ are symmetric positive definite performance matrices, and hence, the $(\bm{A},\bm{Q}^{1/2})$-pair is immediately observable; $\bm{\epsilon}_t\sim \mathcal{N}(0,\bm{D_0})$ are i.i.d Gaussian random variables with positive definite covariance $\bm{D_0}\succ0$. From the optimal control theory \cite{anderson2007optimal}, the optimal policy of \eqref{eq.LQR} is a linear feedback of the state
\begin{align}\label{eq2.2.1}
    \bm{u_t}=-\bm{K}^\ast \bm x_t,
\end{align}
where $\bm{K}^\ast \in\mathbb{R}^{k\times d}$ is the optimal policy which can be uniquely found by solving an Algebraic Riccati Equation (ARE) \cite{anderson2007optimal} depending on  $\bm{A}$, $\bm{B}$, $\bm{Q}$, $\bm{R}$. This means that finding $\bm{K}^\ast$ using ARE relies on the complete model knowledge.

In the sequel, we pursue finding the optimal policy in a {\em model-free} way by using the AC method, without knowing or estimating $\bm A,\bm B,\bm Q,\bm R$. The structure of the optimal policy in \eqref{eq2.2.1} allows us to reformulate \eqref{eq.LQR} as a static optimization problem over all feasible policy matrix $\bm K\in \mathbb{R}^{k\times d}$. To encourage exploration, we parameterize the policy as
\begin{align}\label{policy}
\{\pi_{\bm K}(\cdot|\bm x)=\mathcal{N}(-\bm{Kx},\sigma^2\bm I_k),\bm K\in \mathbb{R}^{k\times d} \},
\end{align}
where $\mathcal{N}(\cdot,\cdot)$ denotes the Gaussian distribution and $\sigma>0$ is the standard deviation of the exploration noise. In other words, given a state $\bm x_t$, the agent will take an action $\bm u_t$ according to $\bm u_t=-\bm{Kx}_t+\sigma \bm\zeta_t$, where $\bm\zeta_t\sim \mathcal{N}(0,\bm I_k)$.
As a consequence, the optimization problem defined in \eqref{eq.LQR} under policy \eqref{policy} can be reformulated as
\begin{align}
\label{JK}
     \mathop{\text{minimize}}\limits_{\bm K} \ J( \bm{K}):=\lim\limits_{T\to\infty}\frac{1}{T}\mathbb{E}[\sum\limits_{t=1}^T\bm x_t^\top \bm{Qx}_t+\bm u_t^\top \bm{Ru}_t]
\end{align}
subject to
\begin{equation} \label{eq.state_dynamic}
    \begin{aligned}
        \bm{u}_t&=-\bm{Kx}_t+\sigma \bm{\zeta}_t,\\\bm x_{t+1}&=\bm{Ax}_t+\bm{Bu}_t+\bm{\epsilon}_t.
    \end{aligned}
\end{equation}
Therefore, the closed-loop form of system \eqref{eq.state_dynamic} is given by
\begin{align}\label{eq:6}
\bm x_{t+1}=(\bm{A}-\bm{BK})\bm x_t+\bm{\xi}_t,
\end{align}
where $\bm\xi_t=\bm\epsilon_t+\sigma \bm{B\zeta}_t\sim \mathcal{N}(0,\bm{D_{\sigma}})$ with $\bm{D_{\sigma}}=\bm{D_0}+\sigma^2\bm{BB}^\top$. Note that optimizing over the set of stochastic policies \eqref{policy} will lead to the same optimal $\bm{K}^\ast$. From \eqref{eq:6}, a policy $\bm K$ is stabilizing if and only if $\rho(\bm A-\bm{BK})<1$, where $\rho(\cdot)$ denotes the spectral radius. It is well known that if $\bm K$ is stabilizing, the Markov chain in \eqref{eq:6} yields a stationary state distribution $\rho_{\bm K}\sim\mathcal{N}(0, \bm{D_K})$, where $\bm{D_K}$ satisfies the following Lyapunov equation (by taking the variance of \eqref{eq:6})
\begin{align}\label{lyap1}
    \bm{D_K}=\bm{D}_{\sigma}+(\bm A-\bm{BK})\bm{D_K}(\bm A-\bm{BK})^\top.
\end{align}
Similarly, we define $\bm{P_K}$ as the unique positive definite solution to (Bellman equation under $\bm K$)
\begin{align}\label{lyap3}
    \bm{P_K}=\bm{Q}+\bm{K^\top RK}+(\bm{A}-\bm{BK})^\top \bm{P_K} (\bm{A}-\bm{BK}).
\end{align}

Based on $\bm{D_K}$ and $\bm{P_K}$, the following lemma characterizes $J(\bm K)$ and its gradient $\nabla_{\bm K} J(\bm K)$. 
\begin{lemma}[\cite{yang2019provably}]\label{pro1}
For any stabilizing policy $\bm K$, the time-average cost $J(\bm K)$ and its gradient $\nabla_{\bm K} J(\bm K)$ take the following forms
\begin{subequations}
\begin{align} 
\label{eq.cost_formula}
J(\bm K)&= {\rm Tr}(\bm{P_K}\bm{D}_{\sigma})+\sigma^2 {\rm Tr}(\bm R),\\
\nabla_{\bm K} J(\bm K)&=2\bm{E_KD_K},
\label{eq.gradient_formula}
\end{align} 
\end{subequations}
where $\bm{E_K}:=(\bm R+\bm{B}^\top \bm{P_KB})\bm{K}-\bm{B}^\top \bm{P_KA}$.
\end{lemma}
Then, the natural gradient of $J(\bm{K})$ can be calculated as \cite{fazel2018global,yang2019provably}
\begin{align}
\label{eq.natural_formula}
\nabla_{\bm K}^N J(\bm K)=\nabla_{\bm K}J(\bm K)\bm{D_K}^{-1}=\bm{E_K},
\end{align}
which eliminates the burden of estimating $\bm{D_K}$. Note that we omit the constant coefficient since it can be absorbed by the stepsize. 

Calculating the natural gradient $\nabla_{\bm K}^N J(\bm K)$ requires estimating $\bm{P_K}$, which depends on $\bm A, \bm B, \bm Q, \bm R$. To estimate the gradient without the knowledge of the model, we instead directly utilize the Q-value. 
\begin{lemma}[\cite{bradtke1994adaptive,yang2019provably}]\label{pro2}
For any stabilizing policy $\bm K$, the Q-value $Q_{\bm K}(\bm x,\bm u)$ takes the following form
\begin{equation}
\label{eq.Q-structure}
\begin{aligned}
    Q_{\bm K}(\bm x,\bm u)=&\  (\bm{x}^\top,\bm{u}^\top)\bm{\Omega_K}\begin{pmatrix}
    \bm x\\ \bm u
    \end{pmatrix}- \text{\rm Tr}(\bm{P_KD_K})\\
    &-\sigma^2\text{\rm Tr}(\bm R+\bm{P_KBB}^\top),
\end{aligned}
\end{equation}
where 
\begin{align}\label{eq:2}
\bm{\Omega_K}:=
\begin{bmatrix}
\bm{\Omega}^{11}_{\bm K} & \bm{\Omega}^{12}_{\bm{K}} \\ \bm{\Omega}^{21}_{\bm{K}} & \bm{\Omega}^{22}_{\bm{K}}
\end{bmatrix}:=
\begin{bmatrix}
\bm Q+\bm{A}^\top \bm{P_K} \bm A & \bm{A}^\top \bm{P_KB} \\ \bm{B}^\top \bm{P_K} \bm A & \bm R+\bm{B}^\top \bm{P_KB}
\end{bmatrix}.
\end{align}
\end{lemma}
%  $R+B^\top P_KB$ and $B^\top P_K A$ are two sub-matrices of $\Omega_K$, which contain all information needed to estimate $\nabla_{K}^N J(K)$.
Clearly, if we can estimate $\bm{\Omega_K}$, then $\bm{E_K}$ in~\eqref{eq.natural_formula} can be readily estimated by using $\bm{\Omega}^{21}_{\bm{K}}$ and $\bm{\Omega}^{22}_{\bm{K}}$, which represent the bottom left corner block and bottom right corner block of matrix $\bm{\Omega_K}$, respectively.
\section{Single-sample Single-timescale Actor-Critic}
\label{sec:singlesample}
In this section, we describe the single-sample single-timescale AC algorithm for solving LQR. In view of the structure of the Q-value given in \eqref{eq.Q-structure} and the fact that \cite{schacke2004kronecker}
% \begin{align*}
% \phi(x,u)=\text{svec}\left[\begin{pmatrix}
% x\\u
% \end{pmatrix}\begin{pmatrix}
% x \\ u
% \end{pmatrix}^\top \right].
% \end{align*}
% From the fact that \cite{schacke2004kronecker} 
\begin{align}
    (\bm x^\top,\bm u^\top)\ \bm{\Omega_K}
\begin{pmatrix}
\bm x \\ \bm u
\end{pmatrix}=\phi(\bm x,\bm u)^\top{\rm svec}(\bm{\Omega_K}),
\end{align}
where
\begin{align}\label{q2}
\phi(\bm x,\bm u):={\rm svec} [\begin{pmatrix}
\bm x\\\bm u
\end{pmatrix}\begin{pmatrix}
\bm x \\ \bm u
\end{pmatrix}^\top]
\end{align}
and $\text{svec}(\cdot)$ denotes the vectorization of the upper triangular part of a symmetric matrix as defined in \cite{schacke2004kronecker}.
We can then parameterize the Q-estimator (critic) by
\begin{align*}
\hat{Q}_{\bm K}(\bm x,\bm u;\bm\omega,b)=\bm{\phi}(\bm x,\bm u)^\top \bm\omega + b,
\end{align*}
where $\phi(\bm x,\bm u)$ defined in \eqref{q2} is the feature function and $\bm{\omega}$ is the critic. Using the TD(0) learning, the critic update is followed by
\begin{equation}
\label{eq.TD_update}
\begin{aligned}
\bm{\omega}_{t+1}=&\ \bm\omega_{t} + \beta_t [(c_{t}-J(\bm K)+\bm \phi(\bm x_{t+1},\bm u_{t+1})^\top \bm\omega_{t}\\
&+b-\bm\phi(\bm x_{t},\bm u_{t})^\top\bm\omega_{t}-b)]\bm\phi(\bm x_{t},\bm u_{t}),
% &=\omega_{t} + \beta_{t} [(c_{t}-J(K))\phi(x_{t},u_{t})-\phi(x_{t},u_{t})(\phi(x_{t},u_{t})
% -\phi(x_{t+1},u_{t+1}))^\top)\omega_{t}],
\end{aligned}
\end{equation}
where $\beta_t$ is the stepsize of the critic and $\bm K$ denotes the policy under which the state-action pairs are sampled. Note that the constant $b$ is not required for updating the linear coefficient $\bm\omega$.

Taking the expectation of $\bm\omega_{t+1}$  in \eqref{eq.TD_update} with respect to the stationary distribution, conditioned on $\bm\omega_t$, the expected subsequent critic can be written as 
\begin{align}
\label{eq.expected_update}
\mathbb{E}[\bm\omega_{t+1}|\bm\omega_{t}]=\bm\omega_{t}+\beta_t(\bm{b_K}-\bm{A_K\omega}_{t}),
\end{align}
where 
\begin{equation}
    \begin{aligned}
\label{ak}
\bm{A_K}=&\ \mathbb{E}_{(\bm x,\bm u)}[\bm\phi(\bm x,\bm u)(\bm\phi(\bm x,\bm u)-\bm\phi(\bm x',\bm u'))^\top], \\
\bm{b_K}=&\ \mathbb{E}_{(\bm x,\bm u)}[(c(\bm x,\bm u)-J(\bm K))\bm\phi(\bm x,\bm u)].
\end{aligned}
\end{equation}

Note that for ease of exposition, we denote $(\bm x',\bm u')$ as the next state-action pair after $(\bm x,\bm u)$ and abbreviate $\mathbb{E}_{\bm x\sim\rho_{\bm K},\bm u\sim\pi_{\bm K}(\cdot|\bm x)}$ as $\mathbb{E}_{(\bm x,\bm u)}$. 

\begin{assumption}\label{a1}
We consider the policy class $\mathbb{K}$ such that $\forall \bm K\in\mathbb{K}$, $\bm K$ is norm bounded and the spectral radius satisfies $\rho(\bm A-\bm{BK})\leq\lambda$ for some constant $\lambda\in(0,1)$.
\end{assumption}
The above assumes the uniform boundedness of the policy (actor) parameter $\bm K$, which is common in the literature of actor-critic algorithms ~\cite{karmakar2018two,barakat2022analysis,zhou2023single}. One potential approach to address the boundedness assumption involves formulating a projection map capable of diminishing the magnitude of $\|\bm K\|$ when it exceeds the specified boundary \cite{konda1999actor,bhatnagar2009natural}, which is deferred to future research endeavors.

As previously discussed, a policy $\bm K$ is considered stabilizing if and only if $\rho(\bm A-\bm{BK})<1$. Therefore, Assumption \ref{a1} also implies the stability of policy $\bm K$, which is equivalent to assuming the existence of $\bm{A_K}$ due to the expectation being taken over the stationary distribution. Such assumption is standard in the literature \cite{wu2020finite,chen2021closing,olshevsky2023small}. Without loss of generality, we slightly strengthen the requirement to $\rho(\bm A-\bm{BK})\leq \lambda$ for some constant $\lambda\in(0,1)$. This is made to avoid tedious computation of the probability of bounded learning trajectories. It is worth noting that one could alternatively assume $\rho(\bm A-\bm{BK})<1$ and deduce that the same results presented in the sequel with additional high probability characterization. 

We then provide the coercive property of cost function $J(\bm K)$, illustrating that $J(\bm K)$ tends towards infinity as $\|\bm K\|$ approaches infinity or when $\rho(\bm A-\bm{BK})$ approaches 1.
\begin{lemma}[Coercive Property]\label{coe}
    The cost function $J(\bm K)$ defined in \eqref{JK} is coercive, that is, for any sequence $\{\bm K_i\}_{i=1}^\infty$ of stabilizing policies, we have
    \begin{align*}
        J(\bm K_i)\to +\infty , \quad \text{if} \ \|\bm K_i\|\to +\infty \ \  \text{or} \ \ \rho(\bm A-\bm{BK}_i)\to 1.
    \end{align*}
\end{lemma}
\Cref{coe} demonstrates the safety of boundary cutting ($\|\bm K_i\|\to +\infty$, $\rho(\bm A-\bm{BK}_i)\to 1$), ensuring that the optimal $\bm K^\ast$ that minimizes $J(\bm K)$ resides within the class $\mathbb{K}$, thereby justifying \Cref{a1}. Additionally, we present some numerical examples in \Cref{exp} to support this assumption.

As the existence of $\bm{A_K}$ and $\bm{b_K}$ are ensured by \Cref{a1}, given a policy $\pi_{\bm K}$, it is not hard to show that if the update in \eqref{eq.expected_update} has converged to some limiting point $\bm{\omega}^\ast_{\bm K}$, i.e., $\lim_{t\rightarrow \infty}\bm\omega_t=\bm{\omega}^\ast_{\bm K}$, $\bm{\omega}^\ast_{\bm K}$ must be the solution of $\bm{A_K\omega}=\bm{b_K}$.
\begin{lemma}\label{pro3}
Suppose $K \in \mathbb{K}$. Then the matrix $\bm{A_K}$ defined in \eqref{ak} is invertible and $\bm{A_K\omega}=\bm{b_K}$ has a unique solution $\bm{\omega_K}^\ast$ that satisfies
\begin{align}
\label{eq.limiting_point}
    \bm{\omega}^\ast_{\bm K} ={\rm svec}(\bm{\Omega_K}).
\end{align}
where $\bm{\Omega_K}$ is defined in \eqref{eq:2}.
\end{lemma}
Since $\text{smat}(\cdot)$ represents the inverse of $\text{svec}(\cdot)$, it follows that $\bm{\Omega_K}$ can be expressed as $\text{smat}(\bm{\omega_K}^\ast)$, thereby completing the estimation of $\bm{\Omega_K}$.

Combining \eqref{eq.natural_formula}, \eqref{eq:2},  and \eqref{eq.limiting_point}, we can express the natural gradient of $J(\bm K)$ using 
$\bm \omega^\ast_{\bm K}$:
\begin{align*}
\nabla_{\bm K}^N J(\bm K)=\bm\Omega^{22}_{\bm K} \bm K-\bm{\Omega}^{21}_{\bm K} = \text{smat}(\bm{\omega}^\ast_{\bm K})^{22}\bm K-\text{smat}(\bm{\omega}^\ast_{\bm K})^{21},\end{align*}
where $\text{smat}(\bm{\omega}^\ast_{\bm K})^{21}$ and $\text{smat}(\bm{\omega}^\ast_{\bm K})^{22}$ represent the bottom left corner block and bottom right corner block of matrix $\text{smat}(\bm{\omega}^\ast_{\bm K})$, respectively.

This allows us to estimate the natural policy gradient using the critic parameters $\bm\omega_{t}$, and then update the actor in a model-free manner
\begin{align}
\label{eq.policy_update}
\bm K_{t+1}=\bm K_{t}-\alpha_t\widehat{\nabla_{\bm K_t}^N J(\bm K_t)},
\end{align}
where $\alpha_t$ is the actor stepsize and $\widehat{\nabla_{\bm K_t}^N J(\bm K_t)}$ is the natural gradient estimation depending on $\bm \omega_{t}$:
\begin{align}
\label{eq.gradient_estimation}
\widehat{\nabla_{\bm K_t}^N J(\bm K_t)}=\text{smat}(\bm\omega_{t})^{22}\bm K_{t}-\text{smat}(\bm\omega_{t})^{21}.
\end{align}

Furthermore, we introduce a cost estimator $\eta_t$ to estimate the time-average cost $J(\bm K_t)$. Combining the critic update \eqref{eq.TD_update} and the actor update \eqref{eq.policy_update}-\eqref{eq.gradient_estimation}, the single-sample single-timescale AC for solving LQR is listed below.
\begin{algorithm}[H]
\caption{Single-Sample Single-timescale Actor-Critic for Linear Quadratic Regulator}\label{alg1}            
\begin{algorithmic}[1]
\STATE \textbf{Input} initialize actor parameter $\bm K_0 \in \mathbb{K}$, critic parameter $\bm\omega_0$, time-average cost $\eta_0$, stepsizes $\alpha_t$ for actor, $\beta_t$ for critic, and $\gamma_t$ for cost estimator.
%\STATE Sample $x_0$ from initial distribution. Take action $u_0\sim \pi_{K_0}(\cdot| x_0)$ and observe the reward $r_0$ and next state $x_1$. 
\FOR{$t=0,1,2,\cdots,T-1$}
    \STATE Sample $\bm x_t$ from the stationary distribution $\rho_{\bm K_t}$. 
    \STATE Take action $\bm u_t\sim \pi_{\bm K_t}(\cdot| \bm x_t)$ and receive cost $c_t=c(\bm x_t,\bm u_t)$ and the next state $\bm x'_t$.
    \STATE Obtain $\bm u'_t\sim \pi_{\bm K_t}(\cdot| \bm x'_t)$.
    \STATE $\delta_t = c_{t} - \eta_{t}+\bm\phi(\bm x'_{t}, \bm u'_{t})^\top \bm\omega_{t}-\bm\phi(\bm x_{t},\bm u_{t})^\top\bm\omega_{t}$
    \STATE $\eta_{t+1}=\textit{proj}_{\mathcal{B}_{\bar{\eta}}}(\eta_{t}+\gamma_t(c_{t}-\eta_{t}))$
    \STATE $\bm\omega_{t+1}=\textit{proj}_{\mathcal{B}_{\bar{\omega}}}(\bm\omega_{t} + \beta_t \delta_t \bm\phi(\bm x_{t},\bm u_{t}))$
    \STATE $\bm K_{t+1}=\bm K_{t}-\alpha_t(\text{smat}(\bm\omega_{t})^{22}\bm K_{t}-\text{smat}(\bm\omega_{t})^{21})$
\ENDFOR
% \RETURN{ $\eta_{T}$ and $K_{T}$ as the approximation of $J(K)$ and $K$ respectively}
\end{algorithmic} 
\end{algorithm}
Note that \textit{single-sample} refers to the fact that only one sample is used to update the critic per actor step. Line 3 of Algorithm \ref{alg1} samples from the stationary distribution induced by the policy $\pi_{\bm K_t}$, which is a mild requirement in the analysis of uniformly ergodic Markov chain, such as in the LQR problem~\cite{yang2019provably}. It is only made to simplify the theoretical analysis. Indeed, as shown in \cite{tu2018least}, when $\bm K \in \mathbb{K}$, \eqref{eq:6} is geometrically $\beta$-mixing and thus its distribution converges to the stationary distribution exponentially. In practice, one can run the Markov chain in \eqref{eq:6}  a sufficient number of steps and sample one state from the last step to approximate the stationary distribution. In addition, \textit{single-timescale} refers to the fact that the stepsizes for the critic and the actor updates are constantly proportional. 

Since the update of the critic parameter in \eqref{eq.TD_update} requires the time-average cost $J(\bm K_t)$,  Line 7 provides an estimation of it. Besides, on top of~\eqref{eq.TD_update}, we additionally introduce a projection in Line 8 and Line 9 to keep the critic norm-bounded. The projection follows the standard definition, i.e., $\textit{proj}_{\mathcal{B}_{y}}(\bm x)$ means project $\bm x$ to the set $\mathcal{B}_{y}:=\{\bm x|\|\bm x\|\leq y\}$. This is common in the literature~\cite{wu2020finite,yang2019provably,chen2022finite}. In our analysis, the projection is relaxed using its nonexpansive property.

\section{Main Theory}
In this section, we establish the global optimality and analyze the finite-time performance of \Cref{alg1}. All the proofs can be found in the Supplementary Material. 
\begin{theorem}\label{t0}
Suppose that Assumptions \ref{a1} hold and choose $\alpha_t=\frac{c}{\sqrt{T}}, \beta_t=\gamma_t=\frac{1}{\sqrt{T}}$, where $c$ is a small positive constant. It holds that
\begin{align*}
    \frac{1}{T}\sum\limits_{t=0}^{T-1} \mathbb{E}(\eta_t-J(\bm K_t))^2=&\ \mathcal{O}(\frac{1}{\sqrt{T}}), \\ \frac{1}{T}\sum\limits_{t=0}^{T-1} \mathbb{E}\Vert \bm\omega_t-\bm\omega_{\bm K_t}^\ast\Vert^2=&\ \mathcal{O}(\frac{1}{\sqrt{T}}),\\ \mathop{\min}\limits_{0\leq t< T}\mathbb{E}[J(\bm K_t)-J(\bm K^\ast)]=&\ \mathcal{O}(\frac{1}{\sqrt{T}}).
\end{align*}
\end{theorem}
The theorem shows that the cost estimator, the critic, and the actor all converge at a sub-linear rate of $\mathcal{O}(T^{-\frac{1}{2}})$. The $\mathcal{O}$ notation hides the polynomials of the dependence parameters. Note that we have explicitly characterized all the necessary problem parameters in the proofs before the last step of the analysis of the interconnected system. One can easily keep all the problem parameters in the interconnected system analysis and get the order for all parameters. To focus on the key factors and for ease of comprehension, we only show the convergence rate in terms of the iteration number.

Correspondingly, to obtain an $\epsilon$-optimal policy, the required sample complexity is $\mathcal{O}(\epsilon^{-2})$. 
This order is consistent with the existing results on single-sample single-timescale AC \cite{chen2021closing,olshevsky2023small,chen2022finite}. Nevertheless, our result is the first finite-time analysis of the single-sample single-timescale AC with a global optimality guarantee and considers the challenging continuous state-action space.

\subsection{Proof Sketch}
The main challenge in the finite-time analysis lies in that the estimation errors of the time-average cost, the critic, and the natural policy gradient are strongly coupled. To overcome this issue, we view the propagation of these errors  as an interconnected system and analyze them comprehensively. To see the merit of our analysis framework, we sketch the main proof steps of~\Cref{t0} in the following. The supporting lemmas and theorems mentioned below can be found in the Supplementary Material. 

We define three measures $A_T, B_T, C_T$ which denote average values of the cost estimation error, the critic error, and the square norm of natural policy gradient, respectively: 
\begin{align*}
    A_T:=\frac{\sum\limits_{t=0}^{T-1} \mathbb{E}y_t^2}{T}, B_T:=\frac{\sum\limits_{t=0}^{T-1} \mathbb{E}\Vert \bm z_t\Vert^2}{T}, C_T:=\frac{\sum\limits_{t=0}^{T-1} \mathbb{E}\Vert \bm{E}_{\bm K_t}\Vert^2}{T},
\end{align*} where 
$y_t:=\eta_t-J(\bm K_t)$ is the cost estimation error and $\ \bm z_t:=\bm\omega_t-\bm\omega_{t}^\ast$ with $\bm\omega_{t}^\ast:=\bm\omega_{\bm K_t}^\ast$ is the critic error. Note that $\bm E_{\bm K_t}=\nabla_{\bm K_t}^N J(\bm K_t)$ is the natural policy gradient according to \eqref{eq.natural_formula}.

We first derive implicit (coupled) upper bounds for the cost estimation error $y_t$, the critic error $\bm z_t$, and the natural gradient $\bm E_{\bm K_t}$, respectively. After that, we solve an interconnected system of inequalities in terms of $A_T,\; B_T,\; C_T$ to establish the finite-time convergence.

\textbf{Step 1: Cost estimation error analysis.} From the cost estimator update rule (Line 7 of Algorithm \ref{alg1}), we decompose the cost estimation error into (neglecting the projection for the time being):
\begin{equation}
    \begin{aligned}\label{s1}
    y_{t+1}^2=&\ (1-2\gamma_t)y^2_t+2\gamma_t y_t(c_t-J(\bm K_t))\\
    &\ +2y_t(J(\bm K_t)-J(\bm K_{t+1}))\\
    &\ +[J(\bm K_t)-J(\bm K_{t+1})+\gamma_t(c_t-\eta_t)]^2.
\end{aligned}
\end{equation}
The second term on the right hand side of \eqref{s1} is a noise term introduced by random sampling of state-action pairs, which reduces to 0 after taking the expectations. The third term is the variation of the moving targets $J(\bm K_t)$ tracked by cost estimator. It is bounded by $y_t, \bm z_t, \bm E_{\bm K_t}$ utilizing the Lipschitz continuity of $J(\bm K_t)$ (Lemma \ref{p1}), the actor update rule~\eqref{eq.gradient_estimation}, and the Cauchy-Schwartz inequality. The last term reflects the variance in cost estimation, which is bounded by $\mathcal{O}(\gamma_t)$.

\textbf{Step 2: Critic error analysis}. By the critic update rule (Line 8 of Algorithm \ref{alg1}), we decompose the squared error by (neglecting the projection for the time being)
\begin{equation}
    \begin{aligned}\label{s2}
    \Vert \bm z_{t+1}\Vert^2=&\Vert \bm z_t\Vert^2+2\beta_t\langle \bm z_t,\bar{\bm h}(\bm\omega_t,\bm K_t)\rangle+2\beta_t\bm\Lambda(O_t,\bm\omega_t,\bm K_t)\\
    &+2\beta_t\langle \bm z_t,\Delta \bm h(\bm O_t,\eta_t,\bm K_t)\rangle+2\langle \bm z_t,\bm\omega^\ast_{t}-\bm\omega^\ast_{t+1}\rangle\\
    &+\Vert \beta_t(\bm h(\bm O_t,\bm\omega_t,\bm K_t)+\Delta \bm h(\bm O_t,\eta_t,\bm K_t))\\
    &+(\bm\omega^\ast_{t}-\bm\omega^\ast_{t+1})\Vert^2,
\end{aligned}
\end{equation}
where the definitions of $\bm h,\bar{\bm h},\Delta{\bm h},\bm\Lambda$, and $\bm O_t$ can be found in \eqref{notation2} in the Supplementary Material. The second term on the right hand side of \eqref{s2} is bounded by $-\mu \Vert \bm z_t\Vert^2$, where $\mu$ is a lower bound of $\sigma_{\min}(A_{\bm K_t})$ proved in Lemma~\ref{p2}. The third term is a random noise introduced by sampling, which reduces to 0 after taking expectation. The fourth term is caused by inaccurate cost and critic estimations, which can be bounded by the norm of $y_t$ and $\bm z_t$. The fifth term tracks the difference between the drifting critic targets. We control it by the Lipschitz continuity of the critic target established in~\Cref{p3}. The last term reflects the variances of various estimations, which is bounded by $\mathcal{O}(\beta_t)$.

\textbf{Step 3: Natural gradient norm analysis}. From the actor update rule (Line 9 of~\Cref{alg1}) and the almost smoothness property of LQR (Lemma \ref{l6}), we derive
\begin{equation}
    \begin{aligned}\label{s3}
    &2\text{Tr}(\bm D_{\bm K_{t+1}}\bm E_{\bm K_t}^\top \bm E_{\bm K_t})=\frac{1}{\alpha_t}[J(\bm K_{t})-J(\bm K_{t+1})]\\
    &-2\text{Tr}(\bm D_{\bm K_{t+1}}(\hat{\bm E}_{\bm K_t}-\bm E_{\bm K_t})^\top \bm E_{\bm K_t})\\
    &+\alpha_t\text{Tr}(\bm D_{\bm K_{t+1}}\hat{\bm E}^\top_{\bm K_t}(\bm R+\bm B^\top \bm P_{\bm K_t} \bm B)\hat{\bm E}_{\bm K_t}),
\end{aligned}
\end{equation}
where $\hat{\bm E}_{\bm K_t}$ denotes the estimation of the natural gradient $\bm E_{\bm K_t}$. The first term on the left hand side of \eqref{s3} can be considered as the scaled square norm of the natural gradient. The first term on the right hand side compares the actor's performances between consecutive updates, which is bounded via Abel summation by parts. The second term evaluates the inaccurate natural gradient estimation, which is then bounded by the critic error $\bm z_t$ and the natural gradient $\bm E_{\bm K_t}$. The last term can be considered as the variance of the perturbed natural gradient update, which is bounded by $\mathcal{O}(\alpha_t)$. 

\textbf{Step 4: Interconnected iteration system analysis.} Taking expectation and summing \eqref{s1}, \eqref{s2}, \eqref{s3} from $0$ to $T-1$, we obtain the following interconnected iteration system:
\begin{align}\label{interconnected}
    A_T\leq &\mathcal{O}(\frac{1}{\sqrt{T}})+h_2B_T+h_2C_T,\nonumber \\
    B_T\leq &\mathcal{O}(\frac{1}{\sqrt{T}})+h_4\sqrt{A_TB_T}+h_5C_T,\\
    C_T\leq &\mathcal{O}(\frac{1}{\sqrt{T}})+h_7\sqrt{B_TC_T}\nonumber,
\end{align}
where $h_2, h_4, h_5$, and $h_7$ are positive constants defined in \eqref{constants}. By solving the above inequalities, we further prove that if $h_2h_4^2+h_2h_4^2h_7^2+2h_5h_7^2<1$, then $A_T,B_T,C_T$ converge at a rate of $\mathcal{O}(T^{-\frac{1}{2}})$. This condition can be easily satisfied by choosing the stepsize ratio $c$ to be smaller than a threshold defined in \eqref{threshold}.
% \begin{align*}
%     A_T=\mathcal{O}(\frac{1}{\sqrt{T}}),\ B_T=\mathcal{O}(\frac{1}{\sqrt{T}}),\ C_T=\mathcal{O}(\frac{1}{\sqrt{T}}),
% \end{align*}
% which indicates the convergence of cost estimation error, critic error, and natural gradient norm. 

\textbf{Step 5: Global convergence analysis.} To prove the global optimality, we utilize the gradient domination condition of LQR (\Cref{lem:l7}):
\begin{align*}
    J(\bm K)-J(\bm K^\ast)\leq \frac{1}{\sigma_{\text{min}}(\bm R)}\Vert \bm D_{\bm K^\ast}\Vert \text{Tr}(\bm E_{\bm K}^\top \bm E_{\bm K}).
\end{align*}
This property shows that the actor performance error can be bounded by the norm of the natural gradient ($\text{Tr}(\bm E_{\bm K}^\top \bm E_{\bm K})$). Since we have proved the average natural gradient norm $C_T$ converges to zero, summation over both sides of the above inequality yields
\begin{align*}
    \mathop{\min}\limits_{0\leq t< T}\mathbb{E}[J(\bm K_t)-J(\bm K^\ast)]=&\mathcal{O}(\frac{1}{\sqrt{T}}),
\end{align*}
which is the convergence of the actor performance error. We thus complete the proof of~\Cref{t0}.
\section{Experiments}\label{exp}
\begin{figure}[t]
\centering
\vspace{-1em}
	\subfigure[Learning results of Algorithm \ref{alg1}]{
		\begin{minipage}[b]{0.49\textwidth}
			\includegraphics[width=0.495\textwidth]{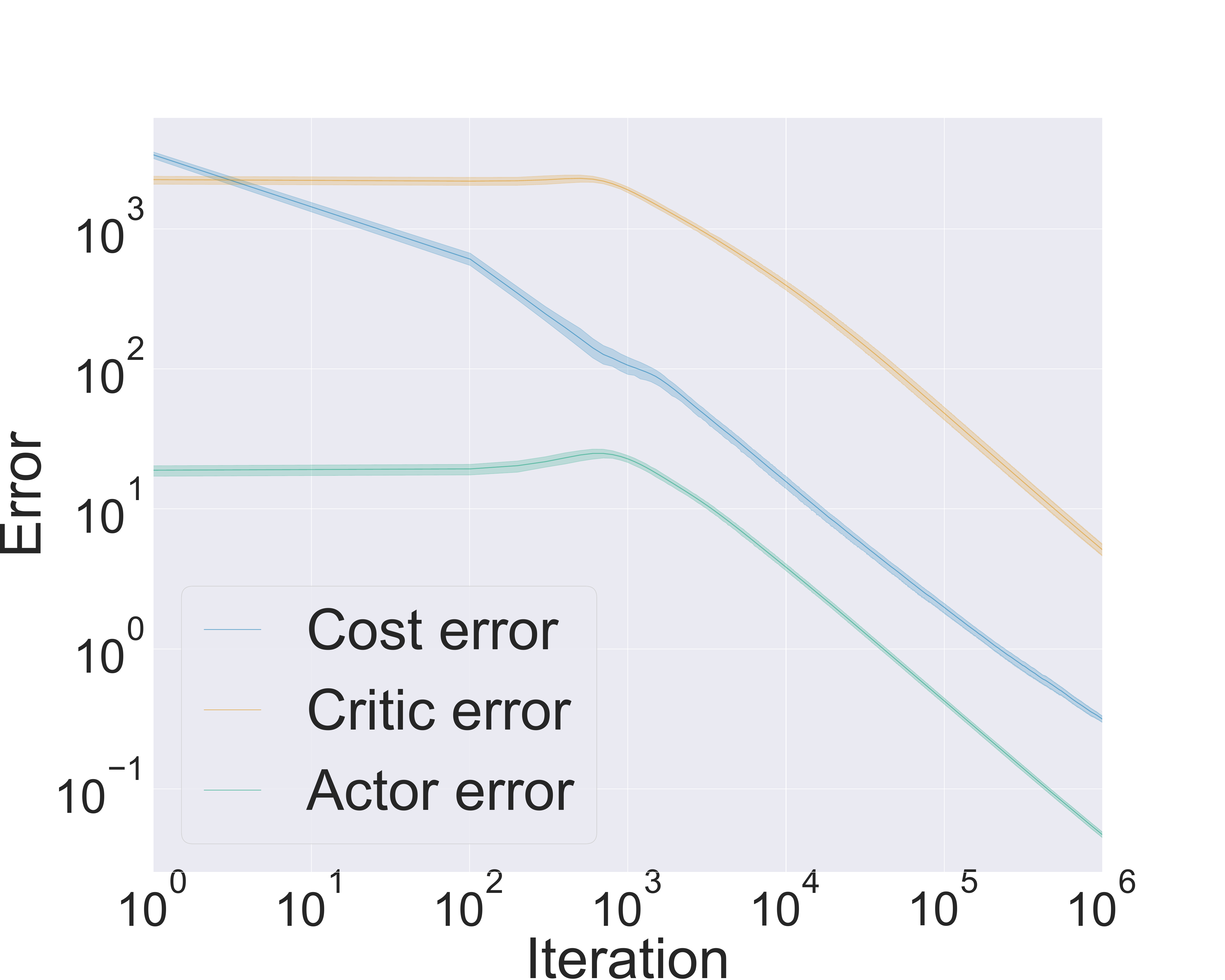} 
			\includegraphics[width=0.495\textwidth]{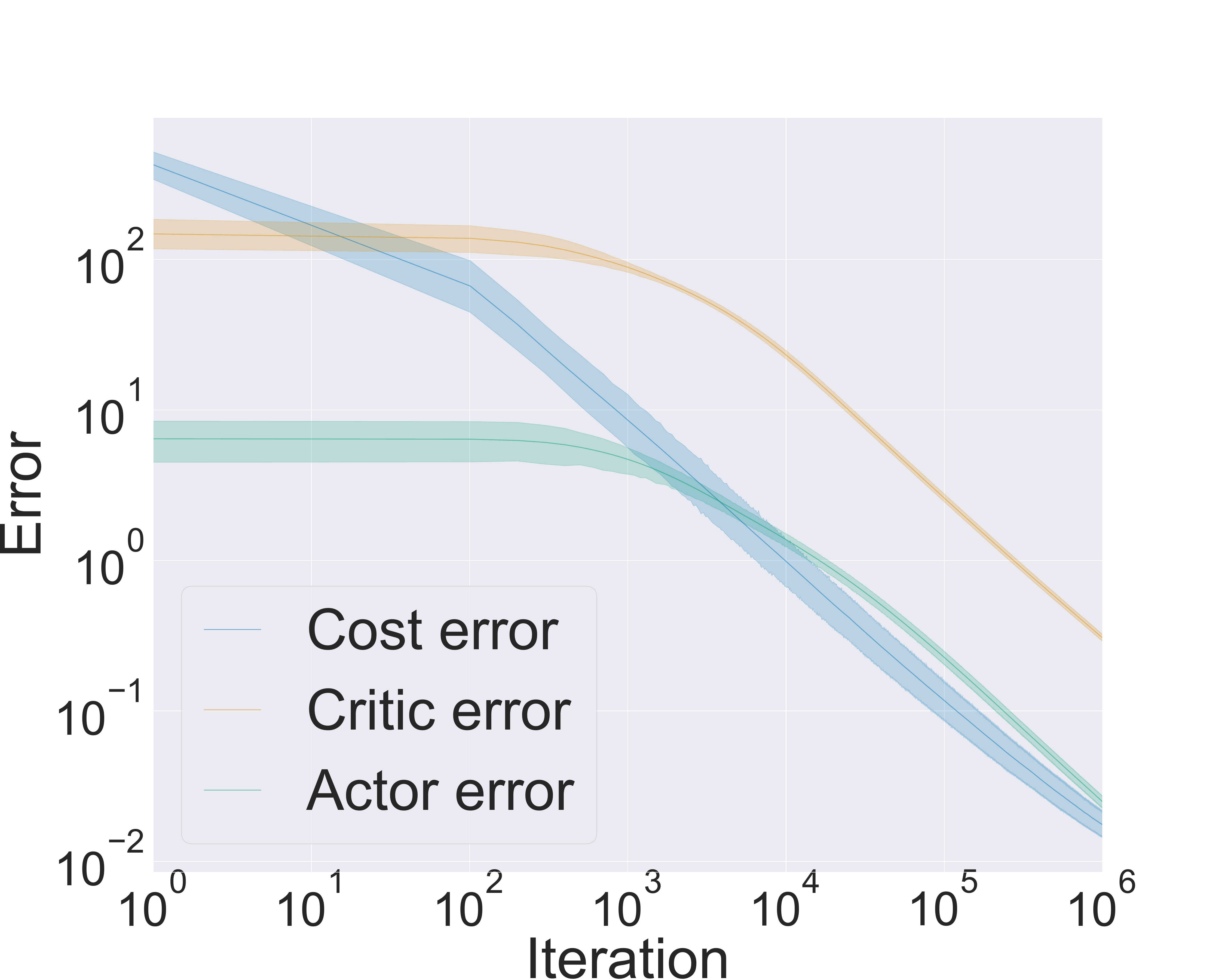}
		\end{minipage}
		\label{fig1}
	}
\vspace{-0.8em}
    	\subfigure[Comparison of~\Cref{alg1} with two other algorithms]{
    		\begin{minipage}[b]{0.49\textwidth}
   		 	\includegraphics[width=0.48\textwidth]{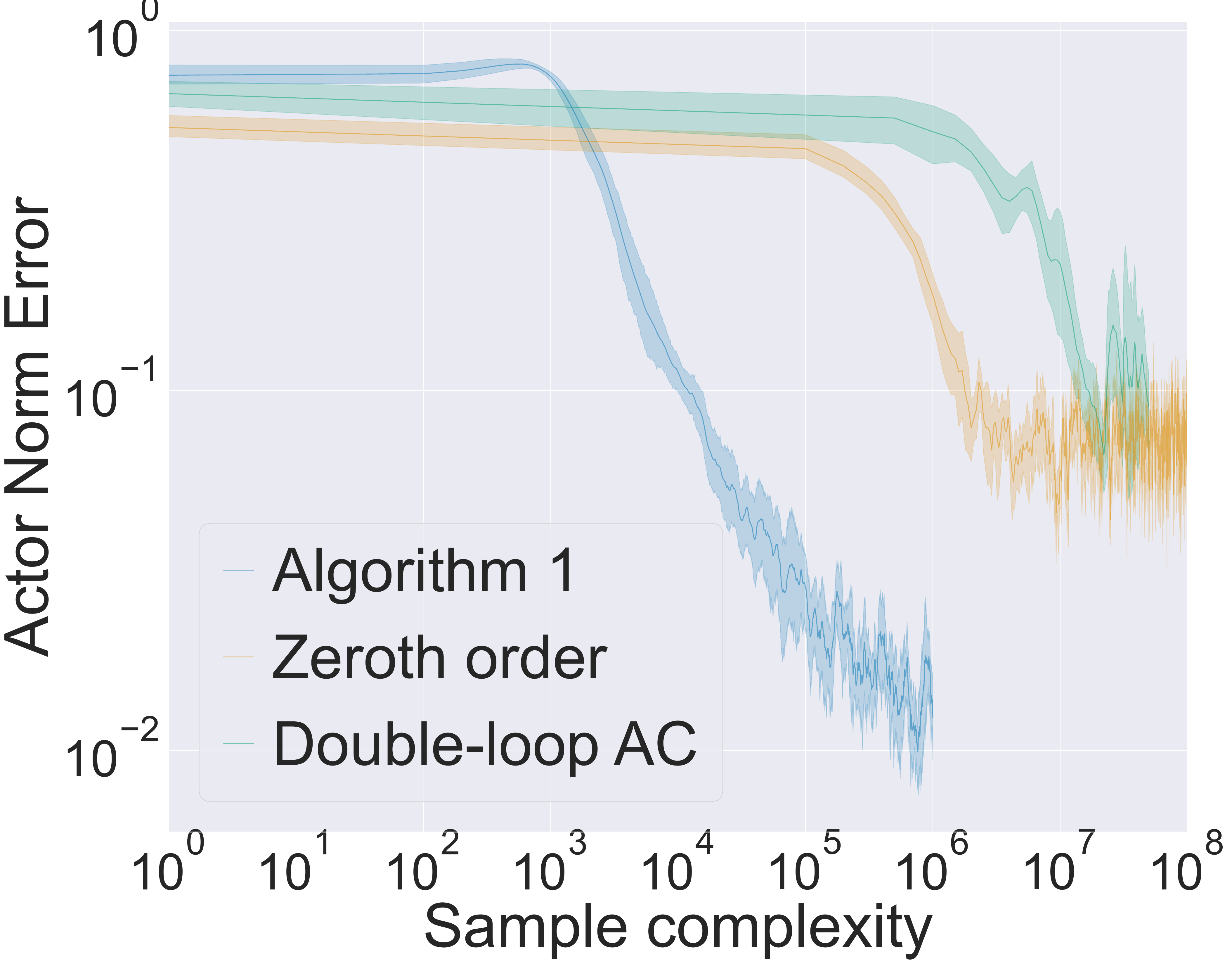}
		 	\includegraphics[width=0.5\textwidth]{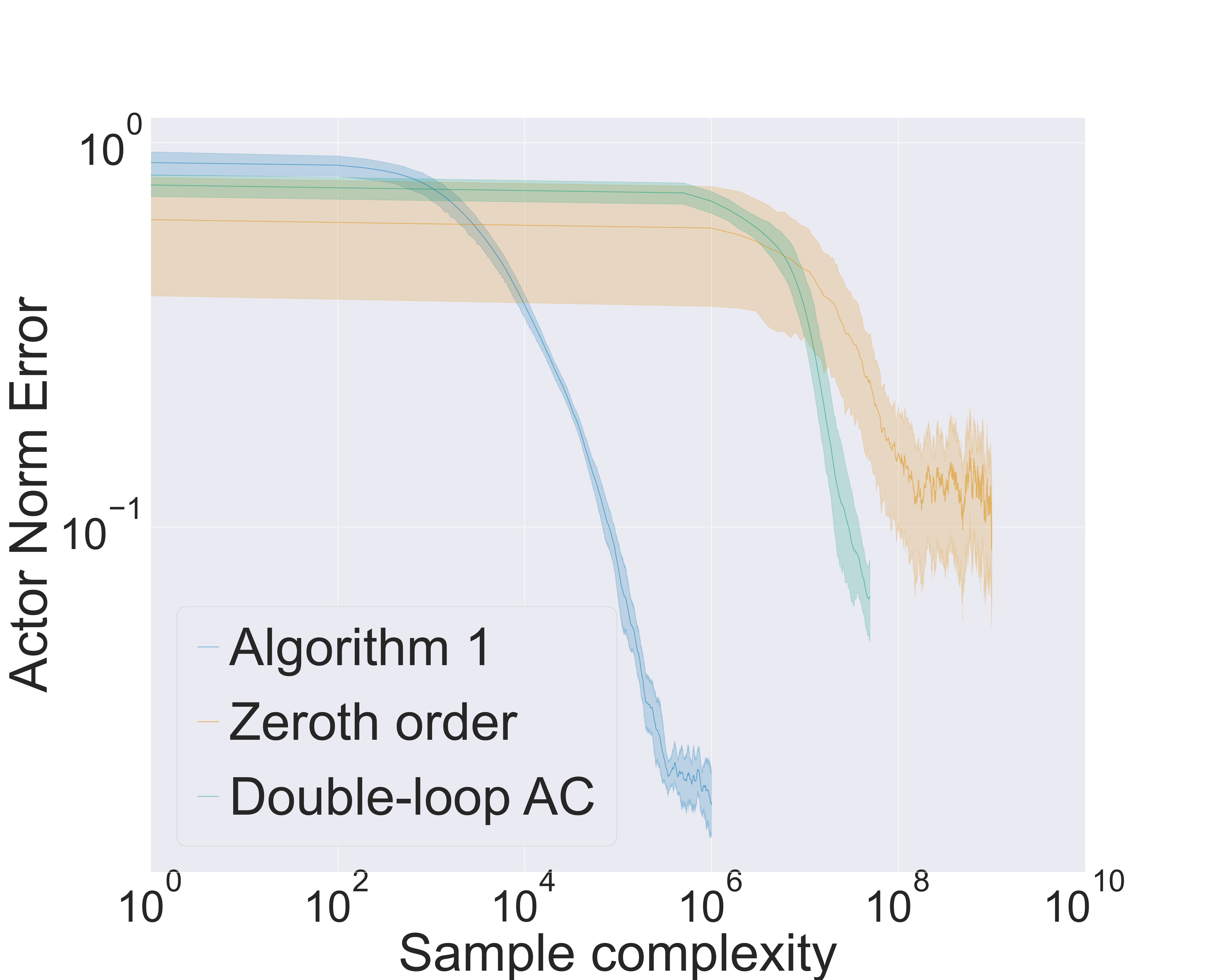}
    		\end{minipage}
    	\label{fig2}
    	}
	\caption{\textbf{(a)} Learning results of Algorithm \ref{alg1}. In the figure, the cost error refers to $\frac{1}{T}\sum_{t=0}^{T-1}(\eta_t-J(\bm K_t))^2$, Critic error refers to $\frac{1}{T}\sum_{t=0}^{T-1}\Vert \bm\omega_t-\bm\omega_{\bm K_t}^\ast\Vert^2$, and the Actor error refers to $\frac{1}{T}\sum_{t=0}^{T-1}[J(\bm K_t)-J(\bm K^\ast)]$, corresponding to the conclusion in~\Cref{t0} empirically.\\ \textbf{(b)} Comparison of~\Cref{alg1} with two other algorithms. The actor norm error refers to $\Vert \bm K-\bm K^* \Vert_F$. In this figure, the solid lines correspond to the mean and the shaded regions correspond to 95\% confidence interval over 10 independent runs.}
	\label{fig}
\vspace{-1em}
\end{figure}
While our main contribution lies in the theoretical analysis, we also present several examples to validate the efficiency of \Cref{alg1}. We provide two examples to illustrate our theoretical results. The first example (first column in \Cref{fig}) is a two-dimensional system and the second example (second column in \Cref{fig}) is a four-dimensional system. The detailed parameters are shown in Supplementary Material.

The performance of \Cref{alg1} is shown in \Cref{fig}, where the left column corresponds to the two-dimensional system and the right column to the four-dimensional system. The solid lines plot the mean values and the shaded regions denote the 95\% confidence interval over 10 independent runs. Consistent with our theorem, \Cref{fig1} shows that the cost estimation error, the critic error, and the actor performance error all diminish at a rate of at least $\mathcal{O}(T^{-\frac{1}{2}})$. The convergence also suggests that the intermediate closed-loop linear systems during iteration are uniformly stable.

We compare \Cref{alg1} with the zeroth-order method \cite{fazel2018global} and the double-loop AC algorithm \cite{yang2019provably} (listed in Algorithm 2 and Algorithm 3 respectively, in Supplementary Material). We plotted the relative errors of the actor parameters for all three methods in Figure \ref{fig2}. As it can be seen that~\Cref{alg1} demonstrates superior sample efficiency compared to the other two algorithms.
\section{Conclusion and Discussion}
In this paper, we establish the finite-time analysis for the single-sample single-timescale AC method under the LQR setting. We for the first time show that this method can find a global optimal policy under the general continuous state-action space, which contributes to understanding the limits of the AC on continuous control tasks. 
\section*{Acknowledgements}
This work was supported by the Singapore Ministry of Education Tier 1 Academic Research Fund (22-5460-A0001) and Tier 2 AcRF (T2EP20123-0037). (Corresponding author: Lin Zhao.) The work of J. Duan was supported in part by the NSF China under Grant 52202487 and in part by the Young Elite Scientists Sponsorship Program by CAST under Grant 2023QNRC001.
\bibliographystyle{named}
\bibliography{ijcai24}
\clearpage
\appendix
\onecolumn
 \begin{center}
    \textbf{\huge{Supplementary Material}}
\end{center}
\vspace{2em}

\noindent\textbf{\large{Table of Contents}}

\noindent\rule{\linewidth}{1pt}
\input{content.toc}
\noindent\rule{\linewidth}{1pt}

\section{Proof of Main Theorems}\label{Supplementary Material1}
We choose stepsizes $\alpha_t=\frac{c}{\sqrt{T}}, \beta_t=\gamma_t=\frac{1}{\sqrt{T}}$. Additional constant multipliers $c_{\beta}, c_{\gamma}$ can be considered in a similar way. Before proceeding, we define the following notations for the ease of presentation:
\begin{equation}
\begin{aligned}\label{notation2}
\omega^\ast_t:=&\omega^\ast_{K_t}, \\
y_t:=&\eta_t-J(K_t), \\
z_t:=&\omega_t-\omega^\ast_t, \\
O_t:=&(x_t,u_t,x'_{t},u'_{t}), \\
\hat{E}_{K_t}:=&\widehat{\nabla_{K_t}^N J(K_t)},\\
\Delta h(O,\eta,K):=&[J(K)-\eta]\phi(x,u), \\
h(O,\omega,K):=&[c(x,u)-J(K)+(\phi(x',u')-\phi(x,u))^\top \omega]\phi(x,u), \\
\bar{h}(\omega,K):=&\mathbb{E}_{(x,u)}[[c(x,u)-J(K)+(\phi(x',u')-\phi(x,u))^\top \omega]\phi(x,u)]. \\
\Lambda(O,\omega,K):=&\langle \omega-\omega^\ast_K,h(O,\omega,K)-\bar{h}(\omega,K)\rangle.
\end{aligned}
\end{equation}
In the sequel, we establish implicit (coupled) upper bounds for the cost estimator, the critic, and the actor in Theorem \ref{tt1}, Theorem \ref{tt2}, and Theorem \ref{tt3}, respectively. Then we prove the main \Cref{t0} by solving an interconnected system of inequalities in Supplementary Material \ref{inter-system}.
\subsection{cost estimation error analysis}
In this section, we establish an implicit upper bound for the cost estimator $\eta_t$, in terms of the critic error and the natural gradient norm. We project $\eta$ into a ball of radius $U$ and project $\omega$ into a ball of radius $\Bar{\omega}$. We use $\Bar{K}$ to denote the upper bound of norm $\|K\|$ for any $K\in\mathbb{K}$.

We first give an uniform upper bound for the covariance matrix $D_{K_t}$.
\begin{lemma}\label{pa1.2.1}
(Upper bound for covariance matrix). Suppose that Assumption \ref{a1} holds. The covariance matrix of the stationary distribution $\mathcal{N}(0,D_{K_t})$ induced by the Markov chain in \eqref{eq:6} can be upper bounded by
\begin{align}\label{updk}
    \Vert D_{K_t}\Vert \leq \frac{c_1}{1-(\frac{1+\lambda}{2})^2}\Vert D_{\sigma}\Vert\ \text{for all}\ t,
\end{align}
where $c_1$ is a constant.
\end{lemma}
Note that the sampled state-action pair $(x_t,u_t)$ can be unbounded. However, in the following lemma, we show that by taking expectation over the stationary state-action distribution, the expected cost and feature function are all bounded.
\begin{lemma}[Upper bound for reward and feature function]\label{new_bounded}
    For $t=0,1,\cdots, T-1$, we have
    \begin{align*}
        \mathbb{E}[c^2_t]\leq C, \  \mathbb{E}[\Vert \phi(x_t,u_t)\Vert^2]\leq C,
    \end{align*}
where $C$ is a constant.
\end{lemma}
\begin{lemma}[Upper bound for cost function]\label{cost_function}
    For $t=0,1,\cdots, T-1$, we have
    \begin{align*}
        J(K_t)\leq U,
    \end{align*}
where $U:=\Vert Q\Vert_{\text{F}}+d\Bar{K}^2+\Vert R\Vert_{\text{F}}+\sigma^2\text{Tr}(R)+\frac{c_1\sqrt{d}\Vert D_{\sigma}\Vert}{1-(\frac{1+\lambda}{2})^2}$ is a constant.
\end{lemma}

% Note that the distribution of state-action pair is unbounded so that the feature function is also unbounded. We can establish an upper bound for the tail probability of $(x,u)$ by the Hansen-Wright inequality, the proof of which can be found in \cite{rudelson2013hanson}.
% \begin{lemma}\label{l1}
% (Hansen-Wright inequality). For any integer $m>0$, let $A$ be a matrix in $\mathbb{R}^{m\times m}$ and let $\eta\sim \mathcal{N}(0,I_m)$be the standard Gaussian random variable in $\mathbb{R}^m$. Then there exists an absolute constant $\bar{c}>0$ such that, for any $\theta \ge 0$, we have
% \begin{align*}
% \mathbb{P}[|\eta^\top A\eta-\mathbb{E}(\eta^\top A\eta)|>\theta]\leq 2e^{-\bar{c}\cdot \text{min}\{\theta^2\Vert A\Vert_{\text{F}}^{-2},\theta \Vert A\Vert^{-1}\}}.
% \end{align*}
% \end{lemma}
% With this lemma, we can provide an uniform upper bound for the cost under high probability.
% \begin{lemma}\label{p0}
% (Upper bound for cost). With probability at least $1-2T^{-10}$, for $t=0,1,2,\cdots,T-1$, the cost satisfies
% \begin{align*}
%     \Vert x_t\Vert^2+\Vert u_t\Vert^2\leq U_r\log(T),\\
%     c(x_t,u_t)\leq U_r\log(T),
% \end{align*}
% where 
% \begin{align}\label{ur}
%     U_r=&2c_2(\sigma_{\max}(Q)+\sigma_{\max}(R)+1)[\sigma^2+(1+\bar{K}^2)\frac{c_1}{1-(\frac{1+\lambda}{2})^2}\Vert D_{\sigma}\Vert]
% \end{align}
% and $c_2$ is a constant.
% \end{lemma}
% Hereafter, we use $U_r\log (T)$ as an upper bound for all cost $c(x_t,u_t)$. As a consequence, we choose $\eta_0\leq U_r\log (T)$ so that we have $\eta_t\leq U_r\log(T)$ for all $t$.
\begin{lemma}\label{l2}
(Perturbation of $P_K$). Suppose $K'$ is a small perturbation of $K$ in the sense that
\begin{align}
    \Vert K'-K\Vert \leq \frac{\sigma_{\text{min}}(D_0)}{4}\Vert D_K\Vert^{-1}\Vert B\Vert^{-1}(\Vert A-BK\Vert+1)^{-1}. \label{eq24}
\end{align}
Then we have 
\begin{align*}
    \Vert P_{K'}-P_K\Vert \leq &\ 6\sigma_{\text{min}}^{-1}(D_0)\Vert D_K\Vert \Vert K\Vert \Vert R\Vert(\Vert K\Vert \Vert B\Vert\Vert A-BK\Vert +\Vert K\Vert \Vert B\Vert +1)\Vert K-K'\Vert.
\end{align*}
\end{lemma}
\begin{proof}
    See Lemma 5.7 in \cite{yang2019provably} for detailed proof.
\end{proof}
With the perturbation of $P_K$, we are ready to prove the Lipschitz continuous of $J(K)$.
\begin{lemma}\label{p1}
(Local Lipschitz continuity of $J(K)$) Suppose Lemma \ref{l2} holds, for any $K_t,K_{t+1}$, we have
\begin{align*}
    |J(K_{t+1})-J(K_t)|\leq l_1 \Vert K_{t+1}-K_t\Vert,
\end{align*}
where
\begin{equation}
    \begin{aligned}\label{newl1}
    l_1:=&\ 6c_1d\bar{K}\sigma_{\text{min}}^{-1}(D_0)\frac{\Vert D_{\sigma} \Vert^2}{1-(\frac{1+\lambda}{2})^2}\Vert R\Vert(\bar{K}\Vert B\Vert
    (\Vert A\Vert +\bar{K}\Vert B\Vert+1) +1).
\end{aligned}
\end{equation}

\end{lemma}
Equipped with the above lemmas and lemmas, we are able to bound the cost estimation error.

\begin{theorem}\label{tt1}
Suppose that Assumptions \ref{a1} and \ref{a1} hold and choose $\alpha_t=\frac{c}{\sqrt{T}}, \beta_t=\gamma_t=\frac{1}{\sqrt{T}}$, where $c$ is a small positive constant. It holds that
\begin{equation}
    \begin{aligned}\label{average cost}
 \frac{1}{T}\sum\limits_{t=0}^{T-1}\mathbb{E}y^2_t\leq &2(l^2_1(\bar{K}+1)^2\bar{{\omega}}^2c^2+C+3U^2)\frac{1}{\sqrt{T}}+\frac{l_1c_\alpha}{T}\sum\limits_{t=0}^{T-1}\mathbb{E}\Vert z_t\Vert^2+\frac{l_1c_\alpha}{T}\sum\limits_{t=0}^{T-1}\mathbb{E}\Vert E_{K_t}\Vert^2.
\end{aligned}
\end{equation}

\end{theorem}
\begin{proof}
From line 7 of Algorithm \ref{alg1}, we have
\begin{align*}
    \eta_{t+1}-J(K_{t+1})=&\ \Pi_{U}(\eta_t+\gamma_t(c_t-\eta_t))-J(K_{t+1})\\
    =&\ \Pi_{U}(\eta_t+\gamma_t(c_t-\eta_t))-\Pi_{U}(J(K_{t+1})).
\end{align*}
Then, it can be shown that
\begin{align*}
    |y_{t+1}|=&\ |\Pi_{U}(\eta_t+\gamma_t(c_t-\eta_t))-\Pi_{U}(J(K_{t+1}))|\\
    \leq &\  |\eta_t+\gamma_t(c_t-\eta_t)-J(K_{t+1})|\\
    =&\  |y_t+J(K_t)-J(K_{t+1})+\gamma_t(c_t-\eta_t)|.
\end{align*}
Thus we get
\begin{align*}
    y^2_{t+1}\leq&\ (y_t+J(K_t)-J(K_{t+1})+\gamma_t(c_t-\eta_t))^2\\
    \leq&\  y^2_t+2\gamma_ty_t(c_t-\eta_t)+2y_t(J(K_t)-J(K_{t+1}))+2(J(K_t)-J(K_{t+1}))^2+2\gamma^2_t(c_t-\eta_t)^2\\
    =&\ (1-2\gamma_t)y^2_t+2\gamma_t y_t(c_t-J(K_t))+2\gamma_t^2(c_t-\eta_t)^2+2y_t(J(K_t)-J(K_{t+1}))+2(J(K_t)-J(K_{t+1}))^2.
\end{align*}
Taking expectation up to $(x_{t},u_{t})$ for both sides, we have
\begin{align*}
    \mathbb{E}[y_{t+1}^2]\leq&\ (1-2\gamma_t)\mathbb{E}y^2_t+2\gamma_t\mathbb{E}[y_t(c_t-J(K_t))]+2\gamma_t^2\mathbb{E}(c_t-\eta_t)^2 +2\mathbb{E}y_t(J(K_t)-J(K_{t+1}))+2\mathbb{E}(J(K_t)-J(K_{t+1}))^2.
\end{align*}
To compute $\mathbb{E}[y_t(c_t-J(K_t))]$, we use the notation $v_t$ to denote the vector $(x_t,u_t)$ and $v_{0:t}$ to denote the sequence $(x_0,u_0),(x_1,u_1),\cdots,(x_{t},u_{t})$. Hence, we have
\begin{align*}
    \mathbb{E}[y_t(c_t-J(K_t))]=&\mathbb{E}_{v_{0:t}}[y_t(c_t-J(K_t))]\\
    =&\mathbb{E}_{v_{0:t-1}}\mathbb{E}_{v_{0:t}}[y_t(c_t-J(K_t))|v_{0:t-1}].
\end{align*}
Once we know $v_{0:t-1}$, $y_t$ is not a random variable any more. Thus we get
\begin{align*}
    \mathbb{E}_{v_{0:t-1}}\mathbb{E}_{v_{0:t}}[y_t(c_t-J(K_t))|v_{0:t-1}]
    &=\mathbb{E}_{v_{0:t-1}}y_t\mathbb{E}_{v_{0:t}}[(c_t-J(K_t))|v_{0:t-1}]\\
    &=\mathbb{E}_{v_{0:t-1}}y_t\mathbb{E}_{v_t}[c_t-J(K_t)|v_{0:t-1}]\\
    &=0.
\end{align*}
Hereafter, we need to verify Lemma \ref{l2} first and use the local Lipschitz continuous property of $J(K)$ provided by lemma \ref{p1} to bound the cost estimation error. 
Since we have
\begin{align*}
    \Vert K_{t+1}-K_{t}\Vert = \alpha_t\Vert (\text{smat}(\omega_{t})^{22}K_{t}-\text{smat}(\omega_{t})^{21})\Vert,
\end{align*}
to satisfy \eqref{eq24}, we choose a lager $T$ such that
\begin{align}\label{calpha1}
    \frac{1}{\sqrt{T}}\leq \frac{(1-(\frac{1+\lambda}{2})^2)\sigma_{\text{min}}(D_0)}{4c_1\Vert D_{\sigma}\Vert \Vert B\Vert (1+\Vert A\Vert +\bar{K}\Vert B\Vert)(\bar{K}+1)\bar{\omega}}.
\end{align}
Hence, according to the update rule, we have
\begin{align}
    &\Vert K_{t+1}-K_{t}\Vert \nonumber \\
    = &\ \alpha_t\Vert (\text{smat}(\omega_{t})^{22}K_{t}-\text{smat}(\omega_{t})^{21})\Vert \nonumber \\
    \leq &\ \frac{c}{\sqrt{T}}(\bar{K}\Vert \text{smat}(\omega_{t})^{22}\Vert +\Vert \text{smat}(\omega_{t})^{21}\Vert )\nonumber \\
    \leq &\ \frac{c}{\sqrt{T}}(\bar{K}\Vert \omega_{t}\Vert +\Vert \omega_{t}\Vert)\nonumber \\
    \leq &\ \frac{c}{\sqrt{T}}(\bar{K}+1)\bar{\omega} \nonumber \\
    \leq &\ \frac{(1-(\frac{1+\lambda}{2})^2)\sigma_{\text{min}}(D_0)}{4c_1\Vert D_{\sigma}\Vert \Vert B\Vert (1+\Vert A\Vert +\bar{K}\Vert B\Vert)}c_\alpha \nonumber \\
    \leq &\ \frac{\sigma_{\text{min}}(D_0)}{4}\Vert D_{K_t}\Vert^{-1}\Vert B\Vert^{-1}(\Vert A-BK_t\Vert+1)^{-1},\label{ktdiff}
\end{align}
where the last inequality comes from \eqref{updk} and we use fact that $c_\alpha\leq 1$. Thus Lemma \ref{l2} holds for Algorithm \ref{alg1}. As a consequence, lemma \ref{p1} is also guaranteed.

Combining the fact $2\gamma_t\mathbb{E}[y_t(c_t-J(K_t))]=0$, we get
\begin{align*}
    \mathbb{E}[y_{t+1}^2]\leq&\ (1-2\gamma_t)\mathbb{E}y^2_t+2\mathbb{E}y_t(J(K_t)-J(K_{t+1}))+2\mathbb{E}(J(K_t)-J(K_{t+1}))^2+2\gamma_t^2\mathbb{E}(c_t-\eta_t)^2\\
    \leq &\ (1-2\gamma_t)\mathbb{E}y^2_t+2\mathbb{E}|y_t||J(K_t)-J(K_{t+1})|+2\mathbb{E}(J(K_t)-J(K_{t+1}))^2+2\gamma_t^2\mathbb{E}(c_t-\eta_t)^2\\
    \leq &\ (1-2\gamma_t)\mathbb{E}y^2_t+2l_1\mathbb{E}|y_t|\Vert K_t-K_{t+1}\Vert+2\mathbb{E}(J(K_t)-J(K_{t+1}))^2+2\gamma_t^2\mathbb{E}(c_t-\eta_t)^2\\
    \leq &\ (1-2\gamma_t)\mathbb{E}y^2_t+2l_1\alpha_t\mathbb{E}|y_t|\Vert \widehat{E}_{K_t}\Vert+2\mathbb{E}(J(K_t)-J(K_{t+1}))^2+2\gamma_t^2\mathbb{E}(c_t-\eta_t)^2\\
    \leq &\ (1-2\gamma_t)\mathbb{E}y^2_t+2l_1\alpha_t\mathbb{E}|y_t|\Vert \widehat{E}_{K_t}-E_{K_t}+E_{K_t}\Vert+2\mathbb{E}(J(K_t)-J(K_{t+1}))^2+2\gamma_t^2\mathbb{E}(c_t-\eta_t)^2\\
    \overset{(1)}{\leq} &\ (1-2\gamma_t)\mathbb{E}y^2_t+2l_1\alpha_t\mathbb{E}[2(\bar{K}+1)|y_t|\Vert z_t\Vert+|y_t|\Vert E_{K_t}\Vert]+2\mathbb{E}(J(K_t)-J(K_{t+1}))^2+2\gamma_t^2\mathbb{E}(c_t-\eta_t)^2\\
    \leq &\ (1-2\gamma_t)\mathbb{E}y^2_t+2l_1\alpha_t\mathbb{E}[2(\bar{K}+1)^2y_t^2+\Vert z_t\Vert^2/2+y_t^2/2+\Vert E_{K_t}\Vert^2/2]+2\mathbb{E}(J(K_t)-J(K_{t+1}))^2+2\gamma_t^2\mathbb{E}(c_t-\eta_t)^2\\
    \leq &\  (1-(2\gamma_t-2l_1\alpha_t(2(\bar{K}+1)^2+\frac{1}{2})))\mathbb{E}y^2_t+l_1\alpha_t\mathbb{E}\Vert z_t\Vert^2+l_1\alpha_t\mathbb{E}\Vert E_{K_t}\Vert^2+2\mathbb{E}(J(K_t)-J(K_{t+1}))^2+2\gamma_t^2\mathbb{E}(c_t-\eta_t)^2,
\end{align*}
where (1) comes from the fact that
\begin{align*}
    \Vert \widehat{E}_{K_t}-E_{K_t}\Vert\leq 2(\bar{K}+1)\Vert \omega_t-\omega^\ast_t\Vert.
\end{align*}
Choose $c$ small enough such that
\begin{align}\label{calpha2}
    2l_1c(2(\bar{K}+1)^2+\frac{1}{2})\leq 1.
\end{align}
Then we get
\begin{align*}
    \gamma_t\ge 2l_1\alpha_t(2(\bar{K}+1)^2+\frac{1}{2}).
\end{align*}
Thus we have
\begin{align*}
    \mathbb{E}[y_{t+1}^2]\leq&\ (1-\gamma_t)\mathbb{E}y_t^2+l_1\alpha_t\mathbb{E}\Vert z_t\Vert^2+l_1\alpha_t\mathbb{E}\Vert E_{K_t}\Vert^2+2\mathbb{E}(J(K_t)-J(K_{t+1}))^2+2\gamma_t^2\mathbb{E}(c_t-\eta_t)^2.
\end{align*}
Rearranging and summing from $0$ to $T-1$, we have
\begin{align*}
    \sum\limits_{t=0}^{T-1}\mathbb{E}y^2_t\leq&\  \underbrace{\sum\limits_{t=0}^{T-1}\frac{1}{\gamma_t}\mathbb{E}(y^2_t-y^2_{t+1})}_{I_1}+\underbrace{\sum\limits_{t=0}^{T-1}\frac{2}{\gamma_t}\mathbb{E}(J(K_t)-J(K_{t+1}))^2}_{I_2}+\underbrace{\sum\limits_{t=0}^{T-1}2\gamma_t\mathbb{E}(c_t-\eta_t)^2}_{I_3}\\
    &\ +l_1c_\alpha\sum\limits_{t=0}^{T-1}\mathbb{E}\Vert z_t\Vert^2+l_1c_\alpha\sum\limits_{t=0}^{T-1}\mathbb{E}\Vert E_{K_t}\Vert^2.
\end{align*}
In the sequel, we need to control $I_1,I_2,I_3$ respectively. For $I_1$, following Abel summation by parts, we have
\begin{align*}
    I_1&=\sum\limits_{t=0}^{T-1}\frac{1}{\gamma_t}\mathbb{E}(y^2_t-y^2_{t+1})\\
    &=\sum\limits_{t=1}^{T-1}(\frac{1}{\gamma_t}-\frac{1}{\gamma_{t-1}})\mathbb{E}(y^2_t)+\frac{1}{\gamma_{0}}\mathbb{E}(y^2_{0})-\frac{1}{\gamma_{T-1}}\mathbb{E}(y^2_{T})\\
    &\leq 4U^2\sum\limits_{t=1}^{T-1}(\frac{1}{\gamma_t}-\frac{1}{\gamma_{t-1}})+\frac{1}{\gamma_{0}}4U^2\\
    &\leq \frac{4U^2}{\gamma_{T-1}}\\
    &=4U^2\sqrt{T},
\end{align*}
where by the projection $(\Pi_U)$ and \Cref{cost_function}, it holds that $|y_t|=|\eta_t-J(K_t)|\leq 2U$.

For $I_2$, we get
\begin{align*}
    I_2&=\sum\limits_{t=0}^{T-1}\frac{2}{\gamma_t}\mathbb{E}(J(K_t)-J(K_{t+1}))^2\\
    &\leq 2l^2_1(\bar{K}+1)^2\bar{{\omega}}^2\sum\limits_{t=0}^{T-1}\frac{1}{\gamma_t}\alpha_t^2\\
    &=2l^2_1(\bar{K}+1)^2\bar{{\omega}}^2c^2\sum\limits_{t=0}^{T-1}\frac{1}{\sqrt{T}}\\
    &= 2l^2_1(\bar{K}+1)^2\bar{{\omega}}^2c^2\sqrt{T}.
\end{align*}
For $I_3$, we have
\begin{align*}
    I_3&=\sum\limits_{t=0}^{T-1}\gamma_t\mathbb{E}(c_t-\eta_t)^2\\
    &\leq \sum\limits_{t=0}^{T-1}\gamma_t\mathbb{E}(2c_t^2+2\eta_t^2)\\
    &\overset{(2)}{\leq} 2(C+U^2)\sum\limits_{t=0}^{T-1}\gamma_t\\
    &= 2(C+U^2)\sqrt{T}
\end{align*}
where (2) is due to the inequality $\mathbb{E}[c_t^2]\leq C$ derived by \Cref{new_bounded}.

Combining all terms, we get
\begin{align*}
    \sum\limits_{t=0}^{T-1}\mathbb{E}y^2_t\leq&\  2(l^2_1(\bar{K}+1)^2\bar{{\omega}}^2c^2+C+3U^2)\sqrt{T}+l_1c_\alpha\sum\limits_{t=0}^{T-1}\mathbb{E}\Vert z_t\Vert^2+l_1c_\alpha\sum\limits_{t=0}^{T-1}\mathbb{E}\Vert E_{K_t}\Vert^2. 
\end{align*}
Dividing by $T$, we have
\begin{align*}
    \frac{1}{T}\sum\limits_{t=0}^{T-1}\mathbb{E}y^2_t\leq &\ 2(l^2_1(\bar{K}+1)^2\bar{{\omega}}^2c^2+C+3U^2)\frac{1}{\sqrt{T}}+\frac{l_1c_\alpha}{T}\sum\limits_{t=0}^{T-1}\mathbb{E}\Vert z_t\Vert^2+\frac{l_1c_\alpha}{T}\sum\limits_{t=0}^{T-1}\mathbb{E}\Vert E_{K_t}\Vert^2.
\end{align*}
Thus we finish our proof.
\end{proof}
\subsection{Critic error analysis}
In this section, we derive an implicit bound for the critic error, in terms of the cost estimator error and the natural gradient norm. First, we need the following lemmas.
\begin{lemma}\label{p2}
 For all the $K_t$, there exists a constant $\mu>0$ such that
\begin{align*}
\sigma_{\text{min}}(A_{K_t})\ge\mu.
\end{align*}
\end{lemma}
\begin{lemma}\label{p3}
(Lipschitz continuity of $\omega^\ast_t$) For any $\omega^\ast_t,\omega^\ast_{t+1}$, we have
\begin{align}
    \Vert\omega^\ast_t-\omega^\ast_{t+1}\Vert\leq l_2\Vert K_t-K_{t+1}\Vert, \label{eq19}
\end{align}
where
\begin{equation}
    \begin{aligned}\label{newl2}
    l_2=&\ 6c_1d^{\frac{3}{2}}\bar{K}(\Vert A\Vert+\Vert B\Vert)^2 \sigma_{\text{min}}^{-1}(D_0)\frac{\Vert D_{\sigma}\Vert\Vert R\Vert}{1-(\frac{1+\lambda}{2})^2}(\bar{K}\Vert B\Vert(\Vert A\Vert +\bar{K}\Vert B\Vert+1) +1).
\end{aligned}
\end{equation}

\end{lemma}
\begin{theorem}\label{tt2}
Suppose that Assumptions \ref{a1} and \ref{a1} hold and choose $\alpha_t=\frac{c}{\sqrt{T}}, \beta_t=\gamma_t=\frac{1}{\sqrt{T}}$, where $c$ is a small positive constant. It holds that
\begin{equation}\label{critic}
    \begin{aligned}
\frac{1}{T}\sum\limits_{t=1}^{T-1}\mathbb{E}\Vert z_t\Vert^2\leq &\  \frac{4}{\mu}(C^2(1+\bar{\omega}^2)+\bar{\omega}^2+l^2_2c_3^2)\frac{1}{\sqrt{T}}+\frac{l_2c}{\mu T}\sum\limits_{t=0}^{T-1}\mathbb{E}\Vert E_{K_t}\Vert^2+\frac{2\sqrt{C}}{\mu}(\frac{1}{T}\sum\limits_{t=0}^{T-1}\mathbb{E}y^2_t)^{\frac{1}{2}}(\frac{1}{T}\sum\limits_{t=0}^{T-1}\mathbb{E}\Vert z_t\Vert^2)^{\frac{1}{2}}.
\end{aligned}
\end{equation}
\end{theorem}
\begin{proof}
Since we have $A_{K_t}\omega_t^\ast=b_{K_t}$, where $b_{K_t}=\mathbb{E}_{(x_t,u_t)}[(c(x_t,u_t)-J(K_t))\phi(x_t,u_t)]$, we can further get
\begin{align*}
    \Vert \omega_t^\ast\Vert &=\Vert A_{K_t}^{-1}b_{K_t}\Vert \\
    &\leq \frac{1}{\mu}\mathbb{E}|c(x_t,u_t)-J(K_t)|\|\phi(x_t,u_t)\|\\
    &\leq \frac{2}{\mu}\mathbb{E}(c_t^2+J(K_t)^2+\|\phi(x_t,u_t)\|^2)\\
    &\leq \frac{4(C+U^2)}{\mu}
\end{align*}
where the last inequality is due to \Cref{new_bounded} and \Cref{cost_function}.

Hence, we set
\begin{align}\label{rw}
    \bar{\omega}=\frac{4(C+U^2)}{\mu},
\end{align}
which justifies the projection introduced in the update of critic since $\omega_t^\ast$ lie within this projection radius for all $t$.

From update rule of critic in Algorithm \ref{alg1}, we have
\begin{align*}
    \omega_{t+1}=\Pi_{\bar{\omega}}(\omega_{t} + \beta_t \delta_t \phi(x_{t},u_{t})),
\end{align*}
which further implies
\begin{align*}
    \omega_{t+1}-\omega_{t+1}^\ast = \Pi_{\bar{\omega}}(\omega_{t} + \beta_t \delta_t \phi(x_{t},u_{t}))-\omega_{t+1}^\ast.
\end{align*}
By applying 1-Lipschitz continuity of projection map, we have
\begin{align*}
    \Vert\omega_{t+1}-\omega_{t+1}^\ast \Vert
    = &\ \Vert\Pi_{\bar{\omega}}(\omega_{t} + \beta_t \delta_t \phi(x_{t},u_{t}))-\omega_{t+1}^\ast\Vert\\
    =&\ \Vert\Pi_{\bar{\omega}}(\omega_{t} + \beta_t \delta_t \phi(x_{t},u_{t}))-\Pi_{\bar{\omega}}(\omega_{t+1}^\ast)\Vert\\
    \leq &\ \Vert \omega_{t} + \beta_t \delta_t \phi(x_{t},u_{t})-\omega_{t+1}^\ast\Vert\\
    =&\ \Vert \omega_t-\omega_t^\ast +\beta_t \delta_t \phi(s_{t},a_{t})+(\omega^\ast_t-\omega^\ast_{t+1})\Vert.
\end{align*}
This means
\begin{align*}
    \Vert z_{t+1}\Vert^2
    \leq &\ \Vert z_t +\beta_t \delta_t \phi(s_{t},a_{t})+(\omega^\ast_t-\omega^\ast_{t+1})\Vert^2\\
=&\ \Vert z_t+\beta_t(h(O_t,\omega_t,K_t)+\Delta h(O_t,\eta_t,K_t))+(\omega^\ast_t-\omega^\ast_{t+1})\Vert^2\\ 
=&\ \Vert z_t\Vert^2+2\beta_t\langle z_t,h(O_t,\omega_t,K_t)\rangle+2\beta_t\langle z_t,\Delta h(O_t,\eta_t,K_t)\rangle +2\langle z_t,\omega^\ast_t-\omega^\ast_{t+1}\rangle\\
&\  +\Vert \beta_t(h(O_t,\omega_t,K_t)+\Delta h(O_t,\eta_t,K_t))+(\omega^\ast_t-\omega^\ast_{t+1})\Vert^2\\     \nonumber
=&\ \Vert z_t\Vert^2+2\beta_t\langle z_t,\bar{h}(\omega_t,K_t)\rangle+2\beta_t\Lambda(O_t,\omega_t,K_t)+2\beta_t\langle z_t,\Delta h(O_t,\eta_t,K_t)\rangle+2\langle z_t,\omega^\ast_t-\omega^\ast_{t+1}\rangle\\
&\ +\Vert \beta_t(h(O_t,\omega_t,K_t)+\Delta h(O_t,\eta_t,K_t))+(\omega^\ast_t-\omega^\ast_{t+1})\Vert^2\\ \nonumber
\leq &\ \Vert z_t\Vert^2+2\beta_t\langle z_t,\bar{h}(\omega_t,K_t)\rangle+2\beta_t\Lambda(O_t,\omega_t,K_t)+2\beta_t\langle z_t,\Delta h(O_t,\eta_t,K_t)\rangle+2\langle z_t,\omega^\ast_t-\omega^\ast_{t+1}\rangle\\
&\ +2\beta_t^2\Vert h(O_t,\omega_t,K_t)+\Delta h(O_t,\eta_t,K_t))\Vert^2+2\Vert \omega^\ast_t-\omega^\ast_{t+1} \Vert^2.
\end{align*}
From lemma \ref{p2}, we know that $\sigma_{\text{min}}(A_{K_t})\ge\mu$ for all $K_t$. %According to inequality \ref{eq:4}, the second term can be upper bounded by
Then we have
\begin{align*}
\langle z_t,\bar{h}(\omega_t,K_t)\rangle&=\langle z_t,b_{K_t}-A_{K_t}\omega_t\rangle\\
&=\langle z_t,b_{K_t}-A_{K_t}w_t-(b_{K_t}-A_{K_t}\omega^\ast_t)\rangle\\ 
&=\langle z_t,-A_{K_t}z_t\rangle\\
&=-z_t^\top A_{K_t}z_t\\
&\leq -\mu\Vert z_t\Vert^2,
\end{align*}
where we use the fact $A_K\omega^\ast_{K_t}-b_{K_t}=0$. Hence, we have
\begin{align*}
    \Vert z_{t+1}\Vert^2\leq &\ (1-2\mu\beta_t)\Vert z_t\Vert^2+2\beta_t\Lambda(O_t,\omega_t,K_t)+2\beta_t\langle z_t,\Delta h(O_t,\eta_t,K_t)\rangle+2\langle z_t,\omega^\ast_t-\omega^\ast_{t+1}\rangle\\
    &\ +2\beta_t^2\Vert h(O_t,\omega_t,K_t)+\Delta h(O_t,\eta_t,K_t))\Vert^2+2\Vert \omega^\ast_t-\omega^\ast_{t+1} \Vert^2.
\end{align*}
Taking expectation up to $(x_t,u_t)$, we get
\begin{align*}
    \mathbb{E}\Vert z_{t+1}\Vert^2
\leq &\ (1-2\mu\beta_t)\mathbb{E}\Vert z_t\Vert^2+2\beta_t\mathbb{E}\langle z_t,\Delta h(O_t,\eta_t,K_t)\rangle+2\mathbb{E}\langle z_t,\omega^\ast_t-\omega^\ast_{t+1}\rangle \\
&\ +2\mathbb{E}\Vert \omega^\ast_t-\omega^\ast_{t+1} \Vert^2+2\beta_t^2\mathbb{E}\Vert h(O_t,\omega_t,K_t)+\Delta h(O_t,\eta_t,K_t))\Vert^2.
\end{align*}
It can be shown that
\begin{align*}
    \mathbb{E}[\Lambda(O_t,\omega_t,K_t)]
    =&\mathop{\mathbb{E}}\limits_{v_{0:t}}[\langle \omega_t-\omega^\ast_{K_t},h(O_t,\omega_t,K_t)-\bar{h}(\omega_t,K_t)\rangle]\\
    =&\mathop{\mathbb{E}}\limits_{v_{0:t-1}}\mathop{\mathbb{E}}\limits_{v_{0:t}}[\langle \omega_t-\omega^\ast_{K_t},h(O_t,\omega_t,K_t)-\bar{h}(\omega_t,K_t)\rangle|v_{0:t-1}]\\
    =&\mathop{\mathbb{E}}\limits_{v_{0:t-1}}\langle \omega_t-\omega^\ast_{K_t},\mathop{\mathbb{E}}\limits_{v_{t}}[h(O_t,\omega_t,K_t)-\bar{h}(\omega_t,K_t)|v_{0:t-1}]\rangle\\
    = &\ 0.
\end{align*}
For $\mathbb{E}\Vert g(O_t,\omega_t,K_t)+\Delta g(O_t,\eta_t,K_t))\Vert^2$, we have
\begin{align*}
    \mathbb{E}\Vert g(O_t,\omega_t,K_t)+\Delta g(O_t,\eta_t,K_t))\Vert^2
    \leq\  2\mathbb{E}\| (c_t-\eta_t)\phi(x_t,u_t)\|^2+2\mathbb{E}\|(\phi(x_t',u_t')-\phi(x_t,u_t))\phi(x_t,u_t)\|^2\|\omega_t\|^2.
\end{align*}
From lemma \ref{new_bounded}, we know that $\mathbb{E}\| (c_t-\eta_t)\phi(x_t,u_t)\|^2$ is bounded. Based on the proof of lemma \ref{new_bounded}, we know that $\|(\phi(x_t',u_t')-\phi(x_t,u_t))\phi(x_t,u_t)\|$ is the linear combination of the product of chi-square variables. From the fact that the expectation and variance of the product of chi-square variables are both bounded \cite[Corollary 5.4]{joarder2011statistical}, we know that $\mathbb{E}\|(\phi(x_t',u_t')-\phi(x_t,u_t))\phi(x_t,u_t)\|^2$ is also bounded. For simplicity, we set the constant $C$ large enough such that
\begin{align*}
    \mathbb{E}\Vert g(O_t,\omega_t,K_t)+\Delta g(O_t,\eta_t,K_t))\Vert^2
    \leq&\  2\mathbb{E}\| (c_t-\eta_t)\phi(x_t,u_t)\|^2+2\mathbb{E}\|(\phi(x_t',u_t')-\phi(x_t,u_t))\phi(x_t,u_t)\|^2\|\omega_t\|^2\\
    \leq&\  2C^2+2\Bar{\omega}^2C^2\\
    \leq &\  2C^2(1+\Bar{\omega}^2).
\end{align*}
We further have
\begin{equation}
    \begin{aligned}\label{eq18}
    \mathbb{E}\Vert z_{t+1}\Vert^2
\leq &\ (1-2\mu\beta_t)\mathbb{E}\Vert z_t\Vert^2+2\beta_t\mathbb{E}\langle z_t,\Delta h(O_t,\eta_t,K_t)\rangle+2\mathbb{E}\langle z_t,\omega^\ast_t-\omega^\ast_{t+1}\rangle\\
&\ +2\mathbb{E}\Vert \omega^\ast_t-\omega^\ast_{t+1} \Vert^2+2\beta_t^2\mathbb{E}\Vert h(O_t,\omega_t,K_t)+\Delta h(O_t,\eta_t,K_t))\Vert^2\\
\leq &\  (1-2\mu\beta_t)\mathbb{E}\Vert z_t\Vert^2+ 2\beta_t\sqrt{C}\mathbb{E}\Vert z_t\Vert |y_t|\\
&\ +2\mathbb{E}\langle z_t,\omega^\ast_t-\omega^\ast_{t+1}\rangle+2\mathbb{E}\Vert \omega^\ast_t-\omega^\ast_{t+1} \Vert^2+4C^2(1+\bar{\omega}^2)\beta_t^2.
\end{aligned}
\end{equation}

Based on \eqref{eq19}, we can rewrite the above inequality as
\begin{equation}
    \begin{aligned}\label{eqcr1}
    \mathbb{E}\Vert z_{t+1}\Vert^2
    \leq &\ (1-2\mu\beta_t)\mathbb{E}\Vert z_t\Vert^2+ 2\beta_t\sqrt{C}\mathbb{E}\Vert z_t\Vert |y_t|+2l_2\mathbb{E}\Vert z_t\Vert\Vert K_t-K_{t+1}\Vert +2\mathbb{E}\Vert \omega^\ast_t-\omega^\ast_{t+1} \Vert^2+4C^2(1+\bar{\omega}^2)\beta_t^2\nonumber \\
    \leq &\ (1-2\mu\beta_t)\mathbb{E}\Vert z_t\Vert^2+2\sqrt{C}\beta_t\mathbb{E}|y_t|\Vert z_t\Vert+2l_2\alpha_t\mathbb{E}\Vert z_t\Vert\Vert \widehat{E}_{K_t}\Vert+4C^2(1+\bar{\omega}^2)\beta_t^2+2l^2_2\mathbb{E}\Vert K_t-K_{t+1} \Vert^2\nonumber \\
     \leq &\ (1-2\mu\beta_t)\mathbb{E}\Vert z_t\Vert^2+2\sqrt{C}\beta_t\mathbb{E}|y_t|\Vert z_t\Vert+2l_2\alpha_t\mathbb{E}\Vert z_t\Vert\Vert \widehat{E}_{K_t}-E_{K_t}+E_{K_t}\Vert \nonumber \\
     &\ +4C^2(1+\bar{\omega}^2)\beta_t^2+2l^2_2\mathbb{E}\Vert K_t-K_{t+1} \Vert^2\nonumber \\
     \leq &\  (1-2\mu\beta_t)\mathbb{E}\Vert z_t\Vert^2+2l_2\alpha_t\mathbb{E}[\Vert z_t\Vert\Vert \widehat{E}_{K_t}-E_{K_t}\Vert\\
     &\ +\Vert z_t\Vert\Vert E_{K_t}\Vert]+2\sqrt{C}\beta_t\mathbb{E}|y_t|\Vert z_t\Vert+4C^2(1+\bar{\omega}^2)\beta_t^2+2l^2_2\mathbb{E}\Vert K_t-K_{t+1} \Vert^2\nonumber \\
     \leq &\  (1-2\mu\beta_t)\mathbb{E}\Vert z_t\Vert^2+2l_2\alpha_t\mathbb{E}[2(\bar{K}+1)\Vert z_t\Vert^2\\
     &\ +\frac{\Vert z_t\Vert^2}{2}+\frac{\Vert E_{K_t}\Vert^2}{2}]+2\sqrt{C}\beta_t\mathbb{E}|y_t|\Vert z_t\Vert+4C^2(1+\bar{\omega}^2)\beta_t^2+2l^2_2\mathbb{E}\Vert K_t-K_{t+1} \Vert^2\nonumber \\
     \leq &\  (1-2\mu\beta_t)\mathbb{E}\Vert z_t\Vert^2+2\sqrt{C}\beta_t\mathbb{E}|y_t|\Vert z_t\Vert+(4\bar{K}+5)l_2\alpha_t\mathbb{E}\Vert z_t\Vert^2 +l_2\alpha_t\mathbb{E}\Vert E_{K_t}\Vert^2+4(C^2(1+\bar{\omega}^2)+l^2_2c_3^2)\beta_t^2,
\end{aligned}
\end{equation}

where the last inequality is due to $\Vert K_t-K_{t+1}\Vert\leq \frac{c_3}{\sqrt{T}}=c_3\beta_t$ from \eqref{ktdiff}, where
\begin{align}\label{c3}
    c_3:=\frac{(1-(\frac{1+\lambda}{2})^2)\sigma_{\text{min}}(D_0)}{4c_1\Vert D_{\sigma}\Vert \Vert B\Vert (1+\Vert A\Vert +\bar{K}\Vert B\Vert)}.
\end{align}
Choose $c_\alpha$ small enough such that
\begin{align}\label{calpha3}
    (4\bar{K}+5)l_2c\leq \mu.
\end{align}
Thus we can rewrite \eqref{eqcr1} as
\begin{align*}
    \mathbb{E}\Vert z_{t+1}\Vert^2\leq &\ (1-\mu \beta_t)\mathbb{E}\Vert z_t\Vert^2 +2\sqrt{C}\beta_t\mathbb{E}|y_t|\Vert z_t\Vert+l_2\alpha_t\mathbb{E}\Vert E_{K_t}\Vert^2+4(C^2(1+\bar{\omega}^2)+l^2_2c_3^2)\beta_t^2.
\end{align*}
Rearranging the inequality and summing from $0$ to $T-1$ yields
\begin{align*}
    \mu\sum\limits_{t=1}^{T-1}\mathbb{E}\Vert z_t\Vert^2
    \leq &\ \sum\limits_{t=0}^{T-1}\frac{1}{ \beta_t}\mathbb{E}(\Vert z_t\Vert^2-\Vert z_{t+1}\Vert^2)+2\sqrt{C}\sum\limits_{t=0}^{T-1}\mathbb{E}|y_t|\Vert z_t\Vert+l_2c\sum\limits_{t=0}^{T-1}\mathbb{E}\Vert E_{K_t}\Vert^2+4(C^2(1+\bar{\omega}^2)+l^2_2c_3^2)\sum\limits_{t=0}^{T-1}\beta_t\\
    \leq &\ \underbrace{\sum\limits_{t=0}^{T-1}\frac{1}{\beta_t}\mathbb{E}(\Vert z_t\Vert^2-\Vert z_{t+1}\Vert^2)}_{I_1}+2\sqrt{C}\underbrace{\sum\limits_{t=0}^{T-1}\mathbb{E}|y_t|\Vert z_t\Vert}_{I_2}+l_2c\sum\limits_{t=0}^{T-1}\mathbb{E}\Vert E_{K_t}\Vert^2+4(C^2(1+\bar{\omega}^2)+l^2_2c_3^2)\sqrt{T}.
\end{align*}
In the following, we need to control $I_1$ and $I_2$, respectively.

For term $I_1$, from Abel summation by parts, we have
\begin{align*}
    I_1=&\ \sum\limits_{t=0}^{T-1}\frac{1}{\beta_t}\mathbb{E}(\Vert z_t\Vert^2-\Vert z_{t+1}\Vert^2)\\
    =&\ \sum\limits_{t=1}^{T-1}(\frac{1}{\beta_t}-\frac{1}{\beta_{t-1}})\mathbb{E}\Vert z_t\Vert^2+\frac{1}{\beta_{0}}\mathbb{E}\Vert z_{0}\Vert^2-\frac{1}{\beta_{T-1}}\mathbb{E}\Vert z_{T}\Vert^2\\
    \leq &\ \sum\limits_{t=1}^{T-1}(\frac{1}{\beta_t}-\frac{1}{\beta_{t-1}})\mathbb{E}\Vert z_t\Vert^2+\frac{1}{\beta_{0}}\mathbb{E}\Vert z_{0}\Vert^2\\
    \leq &\ 4\bar{\omega}^2(\sum\limits_{t=1}^{T-1}(\frac{
    1}{\beta_t}-\frac{1}{\beta_{t-1}})+\frac{1}{\beta_{0}})\\
    =&\ 4\bar{\omega}^2\frac{1}{\beta_{T-1}}\\
    =&\ 4\bar{\omega}^2\sqrt{T}.
\end{align*}
For $I_2$, from Cauchy-Schwartz inequality, we have
\begin{align*}
    I_2=&\ \sum\limits_{t=0}^{T-1}\mathbb{E}|y_t|\Vert z_t\Vert\\
    \leq &\ 
    \sum\limits_{t=0}^{T-1}(\mathbb{E}y_t^2)^{\frac{1}{2}}(\mathbb{E}\Vert z_t\Vert^2)^{\frac{1}{2}}\\
    \leq &\ (\sum\limits_{t=0}^{T-1}\mathbb{E}y^2_t)^{\frac{1}{2}}(\sum\limits_{t=0}^{T-1}\mathbb{E}\Vert z_t\Vert^2)^{\frac{1}{2}}.
\end{align*}
Combining the upper bound of the above two items, we can get
\begin{align*}
    \sum\limits_{t=1}^{T-1}\mathbb{E}\Vert z_t\Vert^2 \leq &\ \frac{4}{\mu}(C^2(1+\bar{\omega}^2)+\bar{\omega}^2+l^2_2c_3^2)\sqrt{T}+\frac{l_2c}{\mu}\sum\limits_{t=0}^{T-1}\mathbb{E}\Vert E_{K_t}\Vert^2+\frac{2\sqrt{C}}{\mu}(\sum\limits_{t=0}^{T-1}\mathbb{E}y^2_t)^{\frac{1}{2}}(\sum\limits_{t=0}^{T-1}\mathbb{E}\Vert z_t\Vert^2)^{\frac{1}{2}}.
\end{align*}
Dividing by $T$, we have
\begin{align*}
    \frac{1}{T}\sum\limits_{t=1}^{T-1}\mathbb{E}\Vert z_t\Vert^2\leq &  \frac{4}{\mu}(C^2(1+\bar{\omega}^2)+\bar{\omega}^2+l^2_2c_3^2)\frac{1}{\sqrt{T}}+\frac{l_2c}{\mu T}\sum\limits_{t=0}^{T-1}\mathbb{E}\Vert E_{K_t}\Vert^2+\frac{2\sqrt{C}}{\mu}(\frac{1}{T}\sum\limits_{t=0}^{T-1}\mathbb{E}y^2_t)^{\frac{1}{2}}(\frac{1}{T}\sum\limits_{t=0}^{T-1}\mathbb{E}\Vert z_t\Vert^2)^{\frac{1}{2}},
\end{align*}
which concludes he convergence of critic.
\end{proof}
\subsection{Natural gradient norm analysis}
In this subsection, we derive an implicit bound for the natural gradient norm in terms of the the critic error. Before proceeding, we need the following two lemmas, which characterize two important properties of LQR system.
\begin{lemma}\label{l6}
(Almost Smoothness). For any two stabilizing
policies $K$ and $K'$, $J(K)$ and $J(K')$ satisfy:
\begin{align*}
    J(K')-J(K)
    =-2\text{Tr}(D_{K'}(K-K')^\top E_K)+\text{Tr}(D_{K'}(K-K')^\top(R+B^\top P_KB)(K-K')).
\end{align*}
\end{lemma}
\begin{lemma}\label{lem:l7}
(Gradient Domination). Let $K^\ast$ be an optimal policy. Suppose $K$ has finite cost. Then, it holds that
\begin{align*}
   J(K)-J(K^\ast)\leq \frac{1}{\sigma_{\text{min}}(R)}\Vert D_{K^\ast}\Vert \text{Tr}(E_K^\top E_K).
\end{align*}
\end{lemma}
\begin{theorem}\label{tt3}
Suppose that Assumptions \ref{a1} and \ref{a1} hold and choose $\alpha_t=\frac{c}{\sqrt{T}}, \beta_t=\gamma_t=\frac{1}{\sqrt{T}}$, where $c$ is a small positive constant. It holds that
\begin{equation}
\begin{aligned}\label{actor}
 &\frac{1}{T}\sum\limits_{t=0}^{T-1} \mathbb{E}\Vert E_{K_t}\Vert^2
 \leq (\frac{U+2c_4c_\alpha^2}{2\sigma_{\text{min}}(D_0)c})\frac{1}{\sqrt{T}}+\frac{c_5(\bar{K}+1)}{\sigma_{\text{min}}(D_0)}(\frac{1}{T}\sum\limits_{t=0}^{T-1}\mathbb{E}\Vert z_t\Vert^2)^{\frac{1}{2}}(\frac{1}{T}\sum\limits_{t=0}^{T-1}\mathbb{E}\Vert E_{K_t}\Vert)^{\frac{1}{2}}.
\end{aligned}
\end{equation}
\end{theorem}
\begin{proof}
Combining the almost smoothness property, we get
\begin{align*}
    J(K_{t+1})-J(K_t)=&-2\text{Tr}(D_{K_{t+1}}(K_t-K_{t+1})^\top E_{K_t})+\text{Tr}(D_{K_{t+1}}(K_t-K_{t+1})^\top(R+B^\top P_{K_t} B)(K_t-K_{t+1}))\\
    =&-2\alpha_t\text{Tr}(D_{K_{t+1}}\hat{E}_{K_t}^\top E_{K_t})+\alpha_t^2\text{Tr}(D_{K_{t+1}}\hat{E}^\top_{K_t}(R+B^\top P_{K_t} B)\hat{E}_{K_t})\\
    =&-2\alpha_t\text{Tr}(D_{K_{t+1}}(\hat{E}_{K_t}-E_{K_t})^\top E_{K_t})-2\alpha_t\text{Tr}(D_{K_{t+1}}E_{K_t}^\top E_{K_t})+\alpha_t^2\text{Tr}(D_{K_{t+1}}\hat{E}^\top_{K_t}(R+B^\top P_{K_t} B)\hat{E}_{K_t}).
    % &\leq -2\alpha_t\text{Tr}(\Sigma_{K_{t+1}}(\hat{E}_{K_t}-E_{K_t})^\top E_{K_t})-2\alpha_t\text{Tr}(\Sigma_{K_{t+1}}E_{K_t}^\top E_{K_t})+C_7\alpha_t^2\\
    % &\leq 2C_8\alpha_t\Vert E_{K_t}\Vert \Vert \hat{E}_{K_t}-E_{K_t}\Vert-2\alpha_t\sigma_{\text{min}}(D_0)\Vert E_{K_t}\Vert^2+C_7\alpha_t^2
\end{align*}
By the similar trick to the proof of lemma \ref{pa1.2.1}, we can bound $P_{K_t}$ by
\begin{align*}
    \Vert P_{K_t}\Vert\leq &\frac{\hat{c}_1}{1-(\frac{1+\lambda}{2})^2}\Vert Q+K^\top RK\Vert \\
    \leq &\frac{\hat{c}_1(\sigma_{\max}(Q)+\bar{K}^2\sigma_{\max}(R))}{1-(\frac{1+\lambda}{2})^2},
\end{align*}
where $\hat{c}_1$ is a constant. Hence we further have
\begin{align*}
    \text{Tr}(D_{K_{t+1}}\hat{E}^\top_{K_t}(R+B^\top P_{K_t} B)\hat{E}_{K_t})
    \leq &d\Vert D_{K_{t+1}}\Vert \Vert R+B^\top P_{K_t}B\Vert \Vert \hat{E}_{K_t}\Vert^2_{\text{F}}\\
    \leq &d(\bar{K}+1)^2\bar{\omega}^2 \frac{c_1\Vert D_{\sigma}\Vert}{1-(\frac{1+\lambda}{2})^2}(\sigma_{\max}(R)\\
    &+\sigma^2_{\max}(B)\frac{\hat{c}_1(\sigma_{\max}(Q)+\bar{K}^2\sigma_{\max}(R))}{1-(\frac{1+\lambda}{2})^2}),
\end{align*}
where we use $\Vert \hat{E}_{K_t}\Vert_{\text{F}}\leq (\bar{K}+1)\bar{\omega}$. Hence we define $c_4$ as follows
\begin{equation}
    \begin{aligned}\label{c4}
    c_4:= &d(\bar{K}+1)^2\bar{\omega}^2 \frac{c_1\Vert D_{\sigma}\Vert}{1-(\frac{1+\lambda}{2})^2}(\sigma_{\max}(R)+\sigma^2_{\max}(B)\frac{\hat{c}_1(\sigma_{\max}(Q)+\bar{K}^2\sigma_{\max}(R))}{1-(\frac{1+\lambda}{2})^2}).
\end{aligned}
\end{equation}
Then we get
\begin{align*}
    J(K_{t+1})-J(K_t)
    \leq& -2\alpha_t\text{Tr}(D_{K_{t+1}}(\hat{E}_{K_t}-E_{K_t})^\top E_{K_t})-2\alpha_t\text{Tr}(D_{K_{t+1}}E_{K_t}^\top E_{K_t})+c_4\alpha_t^2\nonumber \\
    \leq& \alpha_t\frac{2c_1d^{\frac{3}{2}}\Vert D_{\sigma}\Vert}{1-(\frac{1+\lambda}{2})^2}\Vert E_{K_t}\Vert \Vert \hat{E}_{K_t}-E_{K_t}\Vert-2\alpha_t\sigma_{\text{min}}(D_0)\Vert E_{K_t}\Vert^2+c_4\alpha_t^2\\
    =&c_5\alpha_t\Vert E_{K_t}\Vert \Vert \hat{E}_{K_t}-E_{K_t}\Vert-2\alpha_t\sigma_{\text{min}}(D_0)\Vert E_{K_t}\Vert^2+c_4\alpha_t^2,
\end{align*}
where
\begin{align}\label{c_5}
    c_5:=\frac{2c_1d^{\frac{3}{2}}\Vert D_{\sigma}\Vert}{1-(\frac{1+\lambda}{2})^2}.
\end{align}
Taking expectation up to $(x_t,u_t)$ and rearranging the above inequality, we have
\begin{align*}
    \mathbb{E}\Vert E_{K_t}\Vert^2\leq &\frac{\mathbb{E}[J(K_t)-J(K_{t+1})]}{2\alpha_t\sigma_{\text{min}}(D_0)}+\frac{c_5}{2\sigma_{\text{min}}(D_0)}\mathbb{E}\Vert E_{K_t}\Vert \Vert \hat{E}_{K_t}-E_{K_t}\Vert+\frac{c_4\alpha_t}{2\sigma_{\text{min}}(D_0)}.
\end{align*}
Summing over $t$ from $0$ to $T-1$ gives
\begin{align*}
    \sum\limits_{t=0}^{T-1}\mathbb{E} \Vert E_{K_t}\Vert^2\leq &\ \underbrace{\sum\limits_{t=0}^{T-1} \frac{\mathbb{E}[J(K_t)-J(K_{t+1})]}{2\alpha_t\sigma_{\text{min}}(D_0)}}_{I_1}+\frac{c_5}{2\sigma_{\text{min}}(D_0))}\underbrace{\sum\limits_{t=0}^{T-1} \mathbb{E}\Vert E_{K_t}\Vert \Vert \hat{E}_{K_t}-E_{K_t}\Vert}_{I_2}+\frac{c_4c_\alpha}{\sigma_{\text{min}}(D_0)}\sqrt{T}.
\end{align*}
For term $I_1$, using Abel summation by parts, we have
\begin{align*}
    \sum\limits_{t=0}^{T-1} \frac{\mathbb{E}[J(K_t)-J(K_{t+1})]}{2\alpha_t\sigma_{\text{min}}(D_0)}
    =&\frac{1}{2\sigma_{\text{min}}(D_0)}(\sum\limits_{t=1}^{T-1} (\frac{1}{\alpha_{t}}-\frac{1}{\alpha_{t-1}})\mathbb{E}[J(K_t)] +\frac{1}{\alpha_{0}}\mathbb{E}[J(K_{0})]-\frac{1}{\alpha_{T-1}}\mathbb{E}[J(K_{T})])\\
    \leq &\frac{U}{2\sigma_{\text{min}}(D_0)}(\sum\limits_{t=1}^{T-1} (\frac{1}{\alpha_{t}}-\frac{1}{\alpha_{t-1}})+\frac{1}{\alpha_{0}})\\
    =&\frac{U}{2\sigma_{\text{min}}(D_0)}\frac{1}{\alpha_{T-1}}\\
    =&\frac{U}{2c_\alpha \sigma_{\text{min}}(D_0)}\sqrt{T}.
\end{align*}
For term $I_2$, by Cauchy-Schwartz inequality, we have
\begin{align*}
    \sum\limits_{t=0}^{T-1}\mathbb{E}\Vert E_{K_t}\Vert \Vert \hat{E}_{K_t}-E_{K_t}\Vert
    \leq (\sum\limits_{t=0}^{T-1}\mathbb{E}\Vert E_{K_t}\Vert^2)^{\frac{1}{2}}(\sum\limits_{t=0}^{T-1}\mathbb{E}\Vert \hat{E}_{K_t}-E_{K_t}\Vert^2)^{\frac{1}{2}}.
\end{align*}
Combining the results of $I_1$ and $I_2$, we have
\begin{align*}
    \sum\limits_{t=0}^{T-1} \mathbb{E}\Vert E_{K_t}\Vert^2
    \leq &\ (\frac{U+2c_4c_\alpha^2}{2\sigma_{\text{min}}(D_0)c})\sqrt{T}+\frac{c_5}{2\sigma_{\text{min}}(D_0)}(\sum\limits_{t=0}^{T-1}\mathbb{E}\Vert E_{K_t}\Vert^2)^{\frac{1}{2}}(\sum\limits_{t=0}^{T-1}\mathbb{E} \Vert \hat{E}_{K_t}-E_{K_t}\Vert^2)^{\frac{1}{2}}\nonumber \\
    \leq &\ (\frac{U+2c_4c_\alpha^2}{2\sigma_{\text{min}}(D_0)c})\sqrt{T}+\frac{c_5(\bar{K}+1)}{\sigma_{\text{min}}(D_0)}(\sum\limits_{t=0}^{T-1}\mathbb{E}\Vert z_t\Vert^2)^{\frac{1}{2}}(\sum\limits_{t=0}^{T-1}\mathbb{E}\Vert E_{K_t}\Vert)^{\frac{1}{2}}.
\end{align*}
Dividing by $T$, we get
\begin{align*}
    \frac{1}{T}&\sum\limits_{t=0}^{T-1} \mathbb{E}\Vert E_{K_t}\Vert^2\leq (\frac{U+2c_4c_\alpha^2}{2\sigma_{\text{min}}(D_0)c})\frac{1}{\sqrt{T}}+\frac{c_5(\bar{K}+1)}{\sigma_{\text{min}}(D_0)}(\frac{1}{T}\sum\limits_{t=0}^{T-1}\mathbb{E}\Vert z_t\Vert^2)^{\frac{1}{2}}(\frac{1}{T}\sum\limits_{t=0}^{T-1}\mathbb{E}\Vert E_{K_t}\Vert)^{\frac{1}{2}}.
\end{align*}
Thus we conclude our proof.
\end{proof}
\subsection{Interconnected iteration system analysis}\label{inter-system}
We know that
\begin{align*}
    A_T=\frac{1}{T}\sum\limits_{t=0}^{T-1} \mathbb{E}y_t^2,\ 
    B_T=\frac{1}{T}\sum\limits_{t=0}^{T-1} \mathbb{E}\Vert z_t\Vert^2,\ 
    C_T=\frac{1}{T}\sum\limits_{t=0}^{T-1} \mathbb{E}\Vert E_{K_t}\Vert^2.
\end{align*}
In the following, we give an interconnected iteration system analysis with respect to $A_T$, $B_T$ and $C_T$.
\begin{theorem}
Combining \eqref{average cost}, \eqref{critic} and \eqref{actor}, we have
\begin{align}
    A_T=\mathcal{O}(\frac{1}{\sqrt{T}}),\ B_T=\mathcal{O}(\frac{1}{\sqrt{T}}),\ C_T=\mathcal{O}(\frac{1}{\sqrt{T}}).
\end{align}
\end{theorem}
% According to the definition in \eqref{notation}, we have
% \begin{align*}
%     A_T=\frac{1}{T}\sum\limits_{t=0}^{T-1}\mathbb{E}y^2_t, \ B_T=\frac{1}{T}\sum\limits_{t=0}^{T-1}\mathbb{E}\Vert z_t\Vert^2,\ C_T=\frac{1}{T}\sum\limits_{t=0}^{T-1}\mathbb{E}\Vert E_{K_t}\Vert.
% \end{align*}
% \noindent \textbf{Proof of Theorem \ref{t0}}:
\begin{proof}
From \eqref{average cost}, \eqref{critic} and \eqref{actor}, we have
\begin{align*}
    A_T\leq&\  2(l^2_1(\bar{K}+1)^2\bar{{\omega}}^2c^2+C+3U^2)\frac{1}{\sqrt{T}}+l_1c_\alpha B_T+l_1c_\alpha C_T,\\
    B_T\leq &\  \frac{4}{\mu}(C^2(1+\bar{\omega}^2)+\bar{\omega}^2+l^2_2c_3^2)\frac{1}{\sqrt{T}}+\frac{2\sqrt{C}}{\mu}\sqrt{A_TB_T}+\frac{l_2c}{\mu}C_T,\\
    C_T\leq &\  (\frac{U+2c_4c_\alpha^2}{2\sigma_{\text{min}}(D_0)c})\frac{1}{\sqrt{T}}+\frac{c_5(\bar{K}+1)}{\sigma_{\text{min}}(D_0)}\sqrt{B_TC_T}.
\end{align*}
For simplicity, we denote
\begin{equation}
    \begin{aligned}\label{constants}
    h_1:=&\ 4(l^2_1(\bar{K}+1)^2\bar{{\omega}}^2c^2+C+2U^2)\frac{1}{\sqrt{T}},\\
    h_2:=&\ l_1c_\alpha,\\
    h_3:=&\ \frac{4}{\mu}(C^2(1+\bar{\omega}^2)+\bar{\omega}^2+l^2_2c_3^2)\frac{1}{\sqrt{T}},\\
    h_4:=&\ \frac{2\sqrt{C}}{\mu},\\
    h_5:=&\ \frac{l_2c}{\mu},\\
    h_6:=&\ (\frac{U+2c_4c_\alpha^2}{2\sigma_{\text{min}}(D_0)c})\frac{1}{\sqrt{T}},\\
    h_7:=&\ \frac{c_5(\bar{K}+1)}{\sigma_{\text{min}}(D_0)}.
\end{aligned}
\end{equation}
Thus we further have
\begin{align}\label{eq98}
    A_T\leq &\ h_1+h_2B_T+h_2C_T,\\
    B_T\leq &\ h_3+h_4\sqrt{A_TB_T}+h_5C_T,\nonumber \\
    C_T\leq &\ h_6+h_7\sqrt{B_TC_T}\nonumber .
\end{align}
Then we have
\begin{align}\label{eq99}
    B_T&\leq h_3+\frac{1}{2}(h_4^2A_T+B_T)+h_5C_T,\nonumber \\
    B_T&\leq 2h_3+h_4^2A_T+2h_5C_T.
\end{align}
For $C_T$, we get
\begin{align}\label{eq100}
    C_T&\leq h_6+\frac{1}{2}(h_7^2B_T+C_T),\nonumber \\
    C_T&\leq 2h_6+h_7^2B_T 
\end{align}
Combining \eqref{eq98}, \eqref{eq99} and \eqref{eq100}, we have 
\begin{align*}
    B_T\leq &\ 2h_3+h_4^2(h_1+h_2B_T+h_2(2h_6+h_7^2B_T))+2h_5(2h_6+h_7^2B_T)\\
    =&\ 2h_3+h_1h_4^2+2h_2h_4^2h_6+4h_5h_6+(h_2h_4^2+h_2h_4^2h_7^2+2h_5h_7^2)B_T.
\end{align*}
If $h_2h_4^2+h_2h_4^2h_7^2+2h_5h_7^2<1$, we have
\begin{align*}
    B_T\leq \frac{2h_3+h_1h_4^2+2h_2h_4^2f+4ef}{1-h_2h_4^2-h_2h_4^2h_7^2-2h_5h_7^2}.
\end{align*}
Note that
\begin{align*}
   h_2h_4^2+h_2h_4^2h_7^2+2h_5h_7^2
   =&\ l_1c\frac{4C}{\mu^2}+l_1c\frac{4C}{\mu^2}\frac{c^2_5(\bar{K}+1)^2}{\sigma^2_{\text{min}}(D_0)}+\frac{2l_2c}{\mu}\frac{c^2_5(\bar{K}+1)^2}{\sigma^2_{\text{min}}(D_0)}\\
    =&\ c(l_1\frac{4C}{\mu^2}+l_1\frac{4C}{\mu^2}\frac{c^2_5(\bar{K}+1)^2}{\sigma^2_{\text{min}}(D_0)}+\frac{2l_2c^2_5(\bar{K}+1)^2}{\mu \sigma^2_{\text{min}}(D_0)}).
\end{align*}
Thus we can achieve $h_2h_4^2+h_2h_4^2h_7^2+2h_5h_7^2<1$ by choosing the stepsize ratio smaller than the following constant threshold:
\begin{align}\label{threshold}
    1/(\frac{4l_1C}{\mu^2}+\frac{4l_1C}{\mu^2}\frac{c^2_5(\bar{K}+1)^2}{\sigma^2_{\text{min}}(D_0)}+\frac{2l_2c^2_5(\bar{K}+1)^2}{\mu \sigma^2_{\text{min}}(D_0)}).
\end{align}
Therefore, we get
\begin{align*}
    B_T\leq &\ \frac{2h_3+h_1h_4^2+2h_2h_4^2h_6+4h_5h_6}{1-h_2h_4^2-h_2h_4^2h_7^2-2h_5h_7^2}=\mathcal{O}(\frac{1}{\sqrt{T}}),\nonumber \\
    C_T\leq &\  2h_6+h_7^2B_T=\mathcal{O}(\frac{1}{\sqrt{T}}),\nonumber \\
    A_T\leq &\  h_1+h_2B_T+C_T=\mathcal{O}(\frac{1}{\sqrt{T}}).
\end{align*}
Thus we have
\begin{align*}
    A_T=\mathcal{O}(\frac{1}{\sqrt{T}}),\ B_T=\mathcal{O}(\frac{1}{\sqrt{T}}),\ C_T=\mathcal{O}(\frac{1}{\sqrt{T}}),
\end{align*}
which concludes the proof.
\end{proof}

\subsection{Global convergence analysis}

\noindent \textbf{Proof of Theorem \ref{t0}}
\begin{proof}
From gradient domination, we know that
\begin{equation}
    \begin{aligned}\label{n2}
    \mathbb{E}(J(K_t)-J(K^\ast))\leq &\frac{1}{\sigma_{\text{min}}(R)}\Vert D_{K^\ast}\Vert \mathbb{E}[\text{Tr}(E_{K_t}^\top E_{K_t})]\\
    \leq &\frac{d\Vert D_{K^\ast}\Vert}{\sigma_{\text{min}}(R)} \mathbb{E}\Vert E_{K_t}\Vert^2.
\end{aligned}
\end{equation}

From the convergence of $C_T$, we know that
\begin{align*}
    \frac{1}{T}\sum\limits_{t=0}^{T-1}\mathbb{E}\Vert E_{K_t}\Vert=\mathcal{O}(\frac{1}{\sqrt{T}})
\end{align*}
Hence, we have
\begin{align*}
    \mathop{\text{min}}\limits_{0\leq t< T}\frac{d\Vert D_{K^\ast}\Vert}{\sigma_{\text{min}}(R)} \mathbb{E}\Vert E_{K_t}\Vert^2\leq&  \frac{d\Vert D_{K^\ast}\Vert}{\sigma_{\text{min}}(R)} \frac{1}{T}\sum\limits_{t=0}^{T-1}\mathbb{E}\Vert E_{K_t}\Vert^2
    =\mathcal{O}(\frac{1}{\sqrt{T}}).
\end{align*}
Therefore, from \ref{n2} we get
\begin{align*}
    \mathop{\text{min}}\limits_{0\leq t< T}&\mathbb{E}(J(K_t)-J(K^\ast))=\mathcal{O}(\frac{1}{\sqrt{T}}).
\end{align*}
Thus we conclude the proof of Theorem \ref{t0}.

\end{proof}
\section{Proof of lemmas}\label{Supplementary Material2}
\noindent \textbf{Proof of lemma \ref{coe}}:

\begin{proof}
The following proof is a slight modification of Lemma 3 in \cite{duan2023optimization}. From the fact that
    \begin{align*}
        J(K)=&\text{Tr}((Q+K^\top RK)D_K)+\sigma^2\text{Tr}(R)\\
        \geq & \sigma_{\rm min}(D_0)\sigma_{\rm min}(R)\|K\|^2,
    \end{align*}
which directly leads to that $J(K)\to\infty$ when $\|K\|\to\infty$. Since $P_K=\sum\limits_{j=0}^\infty (A-BK)^{j\top}(Q+K^\top RK)(A-BK)^j$, then we have
\begin{align*}
    J(K)=&\text{Tr}(\sum\limits_{j=0}^\infty (A-BK)^{j\top}(Q+K^\top RK)(A-BK)^jD_{\sigma})+\sigma^2\text{Tr}(R)\\
    \geq & \sigma_{\rm min}(D_0)\sigma_{\rm min}(Q)\sum\limits_{j=0}^\infty \|(A-BK)^j\|_{\rm F}^2\\
    \geq & \sigma_{\rm min}(D_0)\sigma_{\rm min}(Q)\sum\limits_{j=0}^\infty \rho(A-BK)^{2j}\\
    = & \sigma_{\rm min}(D_0)\sigma_{\rm min}(Q)\frac{1-\rho(A-BK)^\infty}{1-\rho(A-BK)^2},
\end{align*}
which implies $J(K)\to\infty$ when $\rho(A-BK)\to 1$. Overall, we conclude our proof.
\end{proof}
To establish the Lemma \ref{pro3}, we need the following lemma, the proof of which can be found in \cite{nagar1959bias,magnus1978moments}.
\begin{lemma}\label{lemmab1}
Let $g\sim\mathcal{N}(0,I_n)$ be the standard Gaussian random variable in $\mathbb{R}^n$ and let $M,N$ be two symmetric matrices. Then we have
\begin{align*}
    \mathbb{E}[g^\top Mgg^\top N g]=2\text{Tr}(MN)+\text{Tr}(M)\text{Tr}(N).
\end{align*}
\end{lemma}
\noindent \textbf{Proof of lemma \ref{pro3}}:

\begin{proof}
This lemma is a slight modification of lemma 3.2 in \cite{yang2019provably} and the proof is inspired by the proof of this lemma.

% By lemma B.2 in \cite{yang2019provably}, we have the following fact
% \begin{align}
%     A_K=\tilde{D}_{K}\otimes_s\tilde{D}_{K}-(\tilde{D}_{K}\otimes_s\tilde{D}_{K})(L^\top\otimes_sL^\top)=(\tilde{D}_K\otimes_s\tilde{D}_K)(I-L^\top \otimes_s L^\top)
% \end{align}
% where 
% \begin{align}
%     \tilde{D}_K=\begin{pmatrix}
%      D_K & -D_KK^\top \\ -KD_K & KD_KK^\top+\sigma^2I_k
%      \end{pmatrix}=\begin{pmatrix}
%      0 & 0 \\ 0 & \sigma^2I_k
%     \end{pmatrix}+\begin{pmatrix}
%      I_d \\ -K
%      \end{pmatrix}D_K\begin{pmatrix}
%      I_d \\ -K
%      \end{pmatrix}^\top 
% \end{align}     
% \begin{align}     
%      L:=\begin{pmatrix}
%      A & B\\ -KA & -KB
%      \end{pmatrix}=\begin{pmatrix}
%      I_d \\ -K
%      \end{pmatrix}\begin{pmatrix}
%       A & B
%      \end{pmatrix}
% \end{align}
% and $\mathcal{N}(0,\tilde{D}_K)$ is the stationary distribution of the state-action pair when following policy $\pi_K$. Since $\tilde{D}_K$ is positive definite and $\rho(L)=\rho(A-BK)<1$, then $I-L^\top \otimes_s L^\top$ is positive definite, which implies $A_K$ is invertible.

For any state-action pair $(x,u)\in\mathbb{R}^{d+k}$, we denote the successor state-action pair following policy $\pi_K$ by $(x',u')$. With this notation, as we defined in \eqref{policy}, we have
\begin{align*}
    x'=Ax+Bu+\epsilon, \quad u'= -Kx'+\sigma\zeta.
\end{align*}
where $\epsilon\sim\mathcal{N}(0,D_0)$ and $\zeta\sim\mathcal{N}(0,I_k)$. We further denote $(x,u)$ and $(x',u')$ by $\vartheta$ and $\vartheta'$ respectively. Therefore, we have
\begin{align}\label{markov2}
    \vartheta'=L\vartheta+\varepsilon,
\end{align}
where
\begin{align*}
     L:=\begin{bmatrix}
     A & B\\ -KA & -KB
    \end{bmatrix}=\begin{bmatrix}
     I_d \\ -K
    \end{bmatrix}\begin{bmatrix}
     A & B
    \end{bmatrix},\ 
    \varepsilon:=\begin{bmatrix}
    \epsilon \\ -K\epsilon+\sigma\zeta
    \end{bmatrix}.
\end{align*}
Therefore, by definition, we have $\varepsilon\sim\mathcal{N}(0,\tilde{D}_0)$ where
\begin{align*}
    \tilde{D}_0=\begin{bmatrix}
    D_0 & -D_0 K^\top \\ -KD_0 & KD_0 K^\top+\sigma^2I_k
    \end{bmatrix}.
\end{align*}
Since for any two matrices $M$ and $N$, it holds that $\rho(MN)=\rho(NM)$. Then we get $\rho(L)=\rho(A-BK)<1$. Consequently, the Markov chain defined in \eqref{markov2} have a stationary distribution $\mathcal{N}(0,\tilde{D}_K)$ denoted by $\tilde{\rho}_K$, where $\tilde{D}_K$ is the unique positive definite solution of the following Lyapunov equation
\begin{align}\label{lyap2}
    \tilde{D}_K=L\tilde{D}_KL^\top + \tilde{D}_0
\end{align}
Meanwhile, from the fact that $x\sim\mathcal{N}(0,D_K)$ and $u=-Kx+\sigma\zeta$, by direct computation we have
\begin{align*}
    \tilde{D}_K=&\begin{bmatrix}
    D_K & -D_KK^\top \\ -KD_K & KD_KK^\top+\sigma^2I_k
    \end{bmatrix}\\
    =&\begin{bmatrix}
    0 & 0 \\ 0 & \sigma^2I_k
    \end{bmatrix}+\begin{bmatrix}
    I_d \\ -K
    \end{bmatrix}D_K\begin{bmatrix}
    I_d \\ -K
    \end{bmatrix}^\top.
\end{align*}
From the fact that $\Vert AB\Vert_{\text{F}}\leq \Vert A\Vert_{\text{F}}\Vert B\Vert$ and $\Vert A\Vert\leq \Vert A\Vert_{\text{F}}$, we have
\begin{align*}
    \Vert\tilde{D}_K\Vert\leq \Vert \tilde{D}_K\Vert_{\text{F}}\leq \sigma^2k+\Vert D_{K}\Vert(d+\Vert K\Vert^2_{\text{F}}).
\end{align*}
Then we get
\begin{align*}
    \mathbb{E}_{(x,u)}[\phi(x,u)\phi(x,u)^\top]=\mathbb{E}_{\vartheta\sim\tilde{\rho}_K}[\phi(\vartheta)\phi(\vartheta)^\top].
\end{align*}
Let $M,N$ be any two symmetric matrices with appropriate dimension, we have
\begin{align*}
    \text{svec}(M)^\top \mathbb{E}_{\vartheta\sim\tilde{\rho}_K}[\phi(\vartheta)\phi(\vartheta)^\top]\text{svec}(N)
    =&\mathbb{E}_{\vartheta\sim\tilde{\rho}_K}[\text{svec}(M)^\top \phi(\vartheta)\phi(\vartheta)^\top \text{svec}(N)]\\
    =&\mathbb{E}_{\vartheta\sim\tilde{\rho}_K}[\langle \vartheta\vartheta^\top, M \rangle \langle \vartheta\vartheta^\top,N\rangle]\\
    =&\mathbb{E}_{\vartheta\sim\tilde{\rho}_K}[\vartheta^\top M\vartheta \vartheta^\top N\vartheta]\\
    =&\mathbb{E}_{g\sim\mathcal{N}(0,I_{d+k})}[g^\top \tilde{D}_K^{1/2} M\tilde{D}_K^{1/2}gg^\top\tilde{D}_K^{1/2}N\tilde{D}_K^{1/2}g],
\end{align*}
where $\tilde{D}_K^{1/2}$ is the square root of $\tilde{D}_K$. By applying Lemma \ref{lemmab1}, we have
\begin{align*}
    \text{svec}(M)^\top \mathbb{E}_{\vartheta\sim\tilde{\rho}_K}[\phi(\vartheta)\phi(\vartheta)^\top]\text{svec}(N)
    =&\mathbb{E}_{g\sim\mathcal{N}(0,I_{d+k})}[g^\top \tilde{D}_K^{1/2} M\tilde{D}_K^{1/2}gg^\top\tilde{D}_K^{1/2}N\tilde{D}_K^{1/2}g]\\
    =&2\text{Tr}(\tilde{D}_K^{1/2}M\tilde{D}_KN\tilde{D}_K^{1/2})+\text{Tr}(\tilde{D}_K^{1/2}M\tilde{D}_K^{1/2})\text{Tr}(\tilde{D}_K^{1/2}N\tilde{D}_K^{1/2})\\
    =&2\langle M,\tilde{D}_KN\tilde{D}_K\rangle+\langle M,\tilde{D}_K\rangle \langle N,\tilde{D}_K\rangle\\
    =&\text{svec}(M)^\top (2\tilde{D}_K\otimes_s \tilde{D}_K+\text{svec}(\tilde{D}_K)\text{svec}(\tilde{D}_K)^\top)\text{svec}(N),
\end{align*}
where the last equality follows from the fact that 
\begin{align*}
    \text{svec}(\frac{1}{2}(NSM^\top+MSN^\top))=(M\otimes_s N)\text{svec}(S).
\end{align*}
for any two matrix $M,N$ and a symmetric matrix $S$ \cite{schacke2004kronecker}. Thus we have
\begin{align}\label{eq226}
    \mathbb{E}_{\vartheta\sim\tilde{\rho}_K}[\phi(\vartheta)\phi(\vartheta)^\top] =2\tilde{D}_K\otimes_s \tilde{D}_K+\text{svec}(\tilde{D}_K)\text{svec}(\tilde{D}_K)^\top.
\end{align}
Similarly
\begin{align*}
    \phi(\vartheta')&=\text{svec}[(L\vartheta+\varepsilon)(L\vartheta+\varepsilon)^\top]\\
    &=\text{svec}(L\vartheta\vartheta^\top L^\top+L\vartheta\varepsilon^\top-\varepsilon \vartheta^\top L^\top+\varepsilon\varepsilon^\top).
\end{align*}
Since $\epsilon$ is independent of $\vartheta$, we get
\begin{align*}
    \mathbb{E}_{\vartheta\sim\tilde{\rho}_K}&[\phi(\vartheta)\phi(\vartheta')^\top]=\mathbb{E}_{\vartheta\sim\tilde{\rho}_K}[\phi(\vartheta)\text{svec}(L\vartheta\vartheta^\top L^\top+\tilde{D}_0)].
\end{align*}
By the same argument, we have
\begin{align*}
    \text{svec}(M)^\top \mathbb{E}_{\vartheta\sim\tilde{\rho}_K}[\phi(\vartheta)\phi(\vartheta')^\top] \text{svec}(N)
    =&\mathbb{E}_{\vartheta\sim\tilde{\rho}_K} [\langle \vartheta\vartheta^\top,M\rangle \langle L\vartheta\vartheta^\top L^\top+\tilde{D}_0,N\rangle]\\
    =&\mathbb{E}_{\vartheta\sim\tilde{\rho}_K}[\vartheta^\top M\vartheta\vartheta^\top L^\top NL\vartheta] +\langle M,\tilde{D}_K\rangle \langle \tilde{D}_0,N\rangle]\\
    =&\mathbb{E}_{g\in\mathcal{N}(0,I_{d+k})}[g^\top\tilde{D}_K^{\frac{1}{2}}M\tilde{D}_K^{\frac{1}{2}}gg^\top\tilde{D}_K^{\frac{1}{2}}L^\top NL\tilde{D}_K^{\frac{1}{2}}g]+\langle M,\tilde{D}_K,\rangle \langle \tilde{D}_0,N\rangle]\\
    =&2\text{Tr}(M\tilde{D}_KL^\top NL\tilde{D}_K)+\text{Tr}(M\tilde{D}_K)\text{Tr}(L^\top NL \tilde{D}_K)+\langle M,\tilde{D}_K\rangle \langle \tilde{D}_0,N\rangle\\
    =&2\langle M,\tilde{D}_KL^\top NL\tilde{D}_K\rangle+\langle M,\tilde{D}_K\rangle \langle L \tilde{D}_KL^\top, N\rangle+\langle M,\tilde{D}_K\rangle \langle \tilde{D}_0,N\rangle\\
    =&2\langle M,\tilde{D}_KL^\top NL\tilde{D}_K\rangle+\langle M,\tilde{D}_K\rangle\langle \tilde{D}_K,N\rangle \\
    =&\text{svec}(M)^\top (2\tilde{D}_KL^\top \otimes_s \tilde{D}_KL^\top+\text{svec}(\tilde{D}_K)\text{svec}(\tilde{D}_K)^\top)\text{svec}(N),
\end{align*}
where we make use of the Lyapunov equation \eqref{lyap2}. Thus we get
\begin{equation}
    \begin{aligned}\label{eq239}
    \mathbb{E}_{\vartheta\sim\tilde{\rho}_K}[\phi(\vartheta)\phi(\vartheta')^\top]=&2\tilde{D}_KL^\top \otimes_s \tilde{D}_KL^\top+\text{svec}(\tilde{D}_K)\text{svec}(\tilde{D}_K)^\top.
\end{aligned}
\end{equation}

Therefore, combining \eqref{eq226} and \eqref{eq239}, we have
\begin{align*}
    A_K&=2(\tilde{D}_K\otimes_s \tilde{D}_K-\tilde{D}_KL^\top \otimes_s \tilde{D}_KL^\top)\\
    &=2(\tilde{D}_K\otimes_s \tilde{D}_K)(I-L^\top\otimes_s L^\top),
\end{align*}
where in the last equality we use the fact that
\begin{align*}
    (A\otimes_s B)(C\otimes_s D)=\frac{1}{2}(AC\otimes_s BD+AD\otimes_s BC)
\end{align*}
for any matrices $A,B,C,D$. Since $\rho(L)<1$, then $I-L^\top\otimes_s L^\top$ is positive definite, which further implies $A_K$ is invertible.

From Bellman equation of $Q_K$, we have
 \begin{align*}
     \langle \phi(x,u), \text{svec}(\Omega_K)\rangle =&c(x,u)-J(K)+\langle \mathbb{E}[\phi(x',u')|x,u],\text{svec}(\Omega_K)\rangle.
 \end{align*}
Multiply each side by $\phi(x,u)$ and take a expectation with respect to $(x,u)$, we get
\begin{align*}
    \mathbb{E}[\phi(x,u)(\phi(x,u)-\mathbb{E}[\phi(x',u')|x,u])^\top]\text{svec}(\Omega_K)
    =\mathbb{E}[\phi(x,u)(c(x,u)-J(K))].
\end{align*}
We further have
\begin{align*}
    \mathbb{E}[\phi(x,u)(\phi(x,u)-\mathbb{E}[\phi(x',u')|x,u])^\top]=\mathbb{E}[\phi(x,u)(\phi(x,u)-\phi(x',u'))^\top]=A_K,
\end{align*}
where the first equality comes from the low of total expectation and
\begin{align*}
    \mathbb{E}[\phi(x,u)(c(x,u)-J(K))]=b_K
\end{align*}
Therefore, we get
 \begin{align*}
     A_K\text{svec}(\Omega_K)=b_K,
 \end{align*}
which implies $\omega^\ast_K=\text{svec}(\Omega_K)$. Thus we conclude our proof.
% Furthermore, according to the property of feature function, we have $\phi(x,u)=\phi(\vartheta)=\text{svec}(\vartheta\vartheta^\top)$, which means that
% \begin{align}
%     \phi(x,u)-\phi(x',u')&=\text{svec}[\vartheta\vartheta^\top-(L\vartheta+\varepsilon)(L\vartheta+\varepsilon)^\top]\\
%     &=\text{svec}(\vartheta\vartheta^\top-L\vartheta\vartheta^\top L^\top-L\vartheta\varepsilon^\top-\varepsilon \vartheta^\top L^\top-\varepsilon\varepsilon^\top)
% \end{align}
% Since $\varepsilon$ is independent of $\vartheta$, we have
% \begin{align}
%     A_K=\mathbb{E}_{\vartheta\sim\tilde{\rho}_K}[\phi(z)\text{svec}(\vartheta\vartheta^\top-L\vartheta\vartheta^\top L^\top-\tilde{\Sigma}_0)^\top]
% \end{align}
\end{proof}
\noindent \textbf{Proof of lemma \ref{pa1.2.1}}:
\begin{proof}
Since $D_{K_t}$ satisfies the Lyapunov equation defined in \eqref{lyap1}, we have
\begin{align*}
    D_{K_t}=\sum\limits_{k=0}^{\infty}(A-BK_t)^kD_{\sigma}((A-BK_t)^\top)^k.
\end{align*}
From Assumption \ref{a1}, we know that $\rho(A-BK_t)\leq\lambda<1$. Thus for any $\epsilon>0$, there exists a sub-multiplicative matrix norm $\Vert \cdot \Vert_{\ast}$ such that
\begin{align*}
    \Vert A-BK_t\Vert_{\ast}\leq \rho(A-BK_t)+\epsilon.
\end{align*}
Choose $\epsilon=\frac{1-\lambda}{2}$, we get
\begin{align*}
    \Vert A-BK_t\Vert_{\ast}\leq \frac{1+\lambda}{2}<1.
\end{align*}
Therefore, we can bound the norm of $D_{K_t}$ by
\begin{align*}
    \Vert D_{K_t}\Vert_{\ast} &\leq \sum\limits_{k=0}^\infty \Vert A-BK_t\Vert^{2k}_{\ast}\Vert D_{\sigma}\Vert_{\ast}\\
    &\leq \Vert D_{\sigma}\Vert_{\ast}\sum\limits_{k=0}^{\infty}(\frac{1+\lambda}{2})^{2k}\\
    &\leq \Vert D_{\sigma} \Vert_{\ast}\frac{1}{1-(\frac{1+\lambda}{2})^2}.
\end{align*}
Since all norms are equivalent on the finite dimensional Euclidean space, there exists a constant $c_1$ satisfies
\begin{align*}
    \Vert D_{K_t}\Vert \leq \frac{c_1}{1-(\frac{1+\lambda}{2})^2}\Vert D_{\sigma}\Vert,
\end{align*}
which concludes our proof.
\end{proof}
\noindent \textbf{Proof of lemma \ref{new_bounded}}:
\begin{proof}
    We first bound $\mathbb{E}[c_t^2]$. Note that from the proof of lemma \ref{pro3}, we have $\vartheta_t=(x_t^\top,u_t^\top)^\top\sim \mathcal{N}(0,\Tilde{D}_{K_t})$, where $\Tilde{D}_{K_t}$ is upper bounded by \eqref{updk}. Combining with lemma \ref{pa1.2.1}, we know that $\Tilde{D}_{K_t}$ is norm bounded.
Define
\begin{align*}
    \Sigma:=\begin{bmatrix}
        Q &  \\  & R
    \end{bmatrix}.
\end{align*}
It holds that 
\begin{align*}
    c_t=x_t^\top Qx_t+u_t^\top Ru_t=\vartheta^\top \Sigma \vartheta.
\end{align*}
Then we have
\begin{align*}
    \mathbb{E}[c_t^2]=&\ \mathbb{E}[(\vartheta^\top \Sigma \vartheta)^2]\\
    =&\  \text{Var}(\vartheta^\top \Sigma \vartheta)+[\mathbb{E}(\vartheta^\top \Sigma \vartheta)]^2\\
    =&\  2\text{Tr}(\Sigma \Tilde{D}_{K_t}\Sigma \Tilde{D}_{K_t})+(\text{Tr}(\Sigma \Tilde{D}_{K_t}))^2,
\end{align*}
where we use the fact that if $\vartheta\sim\mathcal{N}(\mu,D)$ is a multivariate Gaussian distribution and $\Sigma$ is a symmetric matrix, we have \cite{rencher2008linear}
\begin{align*}
    \mathbb{E}[\vartheta^\top \Sigma \vartheta]=&\ \text{Tr}(\Sigma D)+\mu^\top \Sigma \mu,\\
    \text{Var}(\vartheta^\top \Sigma \vartheta)=&\  2\text{Tr}(\Sigma D\Sigma D)+4\mu^\top D\Sigma D \mu.
\end{align*}
Since $\Sigma$ and $\Tilde{D}_{K_t}$ are both uniform bounded, $\mathbb{E}[c_t^2]$ is also uniform bounded.

It reminds to bound $\mathbb{E}[\Vert \phi(x_t,u_t)\Vert^2]$. We know that
    \begin{align*}
        \Vert \phi(x_t,u_t)\Vert^2=&\ \langle \text{svec}(\vartheta_t \vartheta_t^\top), \text{svec}(\vartheta_t \vartheta_t^\top)\rangle\\
        =&\ \Vert \vartheta_t\vartheta_t^\top\Vert_{\text{F}}^2\\
        =&\ \sum\limits_{1\leq i,j\leq d+k}(\vartheta_t^i\vartheta_t^j)^2,
    \end{align*}
where $\vartheta_t^i$ and $\vartheta_j^j$ are i-th and j-th component of $\vartheta_t$ respectively. Therefore, we can further get
    \begin{align*}
        \mathbb{E}[\Vert \phi(x_t,u_t)\Vert^2]=& \sum\limits_{1\leq i,j\leq d+k}\mathbb{E}(\vartheta_t^i\vartheta_t^j)^2.
    \end{align*}
It can be shown that
    \begin{align*}
        \vartheta_t^i\vartheta_t^j=\frac{1}{4}(\vartheta_t^i+\vartheta_t^j)^2-\frac{1}{4}(\vartheta_t^i-\vartheta_t^j)^2.
    \end{align*}
Since both $\vartheta_t^i$ and $\vartheta_t^j$ are univariate Gaussian distributions, we have
\begin{align*}
    \vartheta_t^i\vartheta_t^j=\frac{\text{Var}(\vartheta_t^i+\vartheta_t^j)}{4}X-\frac{\text{Var}(\vartheta_t^i-\vartheta_t^j)}{4}Y,
\end{align*}
where $X,Y\sim \chi_1^2$ and we use the fact that the squared of a standard Gaussian random variable has a chi-squared distribution. From $\Vert\Tilde{D}_{K_t}\Vert_{\text{F}}$ is bounded, we know that $\text{Var}(\vartheta_t^i+\vartheta_t^j)$ and $\text{Var}(\vartheta_t^i-\vartheta_t^j)$ are both bounded. Define $c_1:=\frac{\text{Var}(\vartheta_t^i+\vartheta_t^j)}{4}$ and $c_2:=\frac{\text{Var}(\vartheta_t^i-\vartheta_t^j)}{4}$, we can show have
\begin{align*}
    \mathbb{E}[(\vartheta_t^i\vartheta_t^j)^2]=&\ \text{Var}(\vartheta_t^i\vartheta_t^j)+(\mathbb{E}(\vartheta_t^i\vartheta_t^j))^2\\
    =&\ \text{Var}(c_1X-c_2Y )+(\mathbb{E}[c_1X-c_2Y])^2.
\end{align*}
Since $EX=EY=1, \text{Var}(X)=\text{Var}(Y)=2$, it holds that
\begin{align*}
    \mathbb{E}[(\vartheta_t^i\vartheta_t^j)^2]=&\ \text{Var}(c_1X-c_2Y )+(\mathbb{E}[c_1X-c_2Y])^2\\
    =&\ 2c_1^2+2c_2^2-2c_1c_2\text{Cov}(X,Y)+(c_1-c_2)^2\\
    \leq&\  4c_1^2+4c_2^2+2c_1c_2\sqrt{\text{Var}(X)\text{Var}(Y)}\\
    =&\ 4c_1^2+4c_2^2+4c_1c_2.
\end{align*}
Therefore, we get
\begin{align*}
    \mathbb{E}[\Vert \phi(x_t,u_t)\Vert^2]=& \sum\limits_{1\leq i,j\leq d+k}\mathbb{E}(v_iv_j)^2\\
    \leq &(d+k)^2(4c_1^2+4c_2^2+4c_1c_2),
\end{align*}
which is bounded.

Overall, we have shown that there exists a constant $C>0$ such that
\begin{align*}
    \mathbb{E}[c^2_t]\leq C,\ \mathbb{E}[\Vert \phi(x_t,u_t)\Vert^2]\leq C.
\end{align*}
\end{proof}
\noindent \textbf{Proof of lemma \ref{cost_function}}:
\begin{proof}
It can be shown that
    \begin{align*}
    J(K_t)=&\ \mathbb{E}_{(x_t,u_t)}[c(x_t,u_t)] \\
    =&\ \mathbb{E}[x_t^\top Qx_t+u_t^\top Ru_t] \\
    =&\ \mathbb{E}[x_t^\top Qx_t+(-Kx_t+\sigma \zeta_t)^\top R(-Kx_t+\sigma\zeta_t)] \\
    =&\ \mathbb{E}_{x_t\sim \rho_{K_t}}\mathbb{E}_{\zeta_t\sim \mathcal{N}(0,I_k)}[x_t^\top(Q+K_t^\top RK_t)x_t-\sigma x_t^\top K_t^\top R\zeta_t -\sigma\zeta_t^\top RK_tx_t+\sigma^2\zeta_t^\top R\zeta_t ] \\
    =&\ \mathbb{E}_{x_t\sim\rho_{K_t}}[x_t^\top (Q+K_t^\top RK_t)x_t]+\sigma^2\text{Tr}(R)  \\
    =&\ \text{Tr}((Q+K_t^\top RK_t)D_{K_t})+\sigma^2\text{Tr}(R)\\
    \leq &\  \Vert (Q+K_t^\top RK_t)D_{K_t}\Vert_{\text{F}}+\sigma^2\text{Tr}(R)\\
    \leq &\  \Vert Q\Vert_{\text{F}}+\Vert K_t \Vert^2_{\text{F}}+\Vert R\Vert_{\text{F}}+\Vert D_{K_t}\Vert_{\text{F}}+\sigma^2\text{Tr}(R)\\
    \leq &\  \Vert Q\Vert_{\text{F}}+d\Bar{K}^2+\Vert R\Vert_{\text{F}}+\sqrt{d}\Vert D_{K_t}\Vert +\sigma^2\text{Tr}(R)\\
    \leq &\ \Vert Q\Vert_{\text{F}}+d\Bar{K}^2+\Vert R\Vert_{\text{F}}+\sigma^2\text{Tr}(R)+\frac{c_1\sqrt{d}}{1-(\frac{1+\lambda}{2})^2}\Vert D_{\sigma}\Vert\\
    := &\  U,
\end{align*}
where the last inequality comes from lemma \ref{pa1.2.1}.
\end{proof}
\noindent \textbf{Proof of lemma \ref{p1}}:
\begin{proof}
\begin{align*}
    |J(K_{t+1})-J(K_t)|
    =&|\text{Tr}((P_{K_{t+1}}-P_{K_t})D_{\sigma})| \\
    \leq &d\Vert D_{\sigma} \Vert \Vert P_{K_{t+1}}-P_{K_t}\Vert\\
    \leq &6d\Vert D_{\sigma} \Vert \sigma_{\text{min}}^{-1}(D_0)\Vert D_{K_t}\Vert \Vert K_t\Vert \Vert R\Vert(\Vert K_t\Vert \Vert B\Vert \Vert A-BK_t\Vert+\Vert K_t\Vert \Vert B\Vert +1)\Vert K_{t+1}-K_t\Vert\\
    \leq& 6c_1d\bar{K}\sigma_{\text{min}}^{-1}(D_0)\frac{\Vert D_{\sigma} \Vert^2}{1-(\frac{1+\lambda}{2})^2}\Vert R\Vert(\bar{K}\Vert B\Vert(\Vert A\Vert +\bar{K}\Vert B\Vert+1) +1)\Vert K_{t+1}-K_t\Vert\\
    =&l_1\Vert K_{t+1}-K_t\Vert,
\end{align*}
where the second inequality is due to the perturbation of $P_K$ in Lemma \ref{l2} and
\begin{align*}
    l_1:=&\ 6c_1d\bar{K}\sigma_{\text{min}}^{-1}(D_0)\frac{\Vert D_{\sigma} \Vert^2}{1-(\frac{1+\lambda}{2})^2}\Vert R\Vert(\bar{K}\Vert B\Vert(\Vert A\Vert+\bar{K}\Vert B\Vert+1) +1).
\end{align*}
Thus we finish our proof.
\end{proof}
\noindent \textbf{Proof of lemma \ref{p2}}:
\begin{proof}
From lemma \ref{pro3}, we know that 
\begin{align*}
    A_{K_t}=2(\tilde{D}_{K_t}\otimes_s\tilde{D}_{K_t})(I-L^\top\otimes_sL^\top).
\end{align*}
By Assumption \ref{a1}, we have $\rho(L)=\rho(A-BK_t)\leq \lambda<1$.
Then we have 
\begin{align*}
    \Vert A_{K_t}^{-1}\Vert &=\frac{1}{2}\Vert (I-L^\top\otimes_sL^\top)^{-1}(\tilde{D}_{K_t}\otimes_s\tilde{D}_{K_t})^{-1}\Vert\\
    &\leq \frac{1}{2}\Vert (I-L^\top\otimes_sL^\top)^{-1}\Vert \Vert(\tilde{D}_{K_t}\otimes_s\tilde{D}_{K_t})^{-1}\Vert\\
    &\leq \frac{1}{2(1-\lambda^2)}\Vert \tilde{D}_{K_t}^{-1}\Vert^2\\
    &= \frac{1}{2(1-\lambda^2)\sigma^2_{\text{min}}(\tilde{D}_{K_t})}.\\
\end{align*}
To bound $\sigma_{\text{min}}(\tilde{D}_{K_t})$, for any $a\in\mathbb{R}^d$ and $b\in\mathbb{R}^k$, we have
\begin{align*}
    \begin{pmatrix}
    a^\top & b^\top
    \end{pmatrix} \tilde{D}_{K_t}\begin{pmatrix}
    a\\b
    \end{pmatrix}
    =&\mathbb{E}_{(x,u)\sim\mathcal{N}(0,\tilde{D}_{K_t})}[\begin{pmatrix}
    a^\top & b^\top
    \end{pmatrix} \begin{pmatrix}
    x \\ u
    \end{pmatrix}\begin{pmatrix}
    x^\top & u^\top
    \end{pmatrix}\begin{pmatrix}
    a\\b
    \end{pmatrix}]\\
    =&\mathbb{E}_{(x,u)\sim\mathcal{N}(0,\tilde{D}_{K_t})}[((a^\top-b^\top K_t)x+\sigma b^\top \zeta)((a^\top-b^\top K_t)x+\sigma b^\top \zeta)^\top]\\
    =&\mathbb{E}_{x\sim\mathcal{N}(0,D_{K_t}), \zeta\sim\mathcal{N}(0,I_k)}[(a^\top -b^\top K_t)xx^\top (a-K_t^\top b)+\sigma^2 b^\top \zeta \zeta^\top b]\\
    \ge& \sigma_{\text{min}}(D_{K_t})\Vert a-K_t^\top b\Vert^2+\sigma^2\Vert b\Vert^2.\\
    % \ge &\text{min}\{\sigma_{\text{min}}(D_0), \frac{\sigma^2}{2M^2}\} \Vert a-K_t^\top b\Vert^2+\sigma^2\Vert b\Vert^2   \\
    % \ge &\frac{\sigma^2}{2}\Vert b\Vert^2 + \text{min}\{\sigma_{\text{min}}(D_0), \frac{\sigma^2}{2M^2}\}\Vert a\Vert^2\\
    % \ge &\text{min}\{\sigma_{\text{min}}(D_0), \frac{\sigma^2}{2M^2}, \frac{\sigma^2}{2}\}(\Vert a\Vert^2+\Vert b\Vert^2).
\end{align*}
For $\Vert a-K_t^\top b\Vert^2$, we have
\begin{align*}
    \Vert a-K_t^\top b\Vert^2\ge &\ \Vert a\Vert^2+\Vert K^\top_tb\Vert^2-2\Vert a\Vert \Vert K_t^\top\Vert\Vert b\Vert\\
    \ge&\ \Vert a\Vert^2-2\bar{K}\Vert a\Vert\Vert b\Vert\\
    \ge &\ \Vert a\Vert^2-\frac{1}{2}(\Vert a\Vert^2+4\bar{K}^2\Vert b\Vert^2)\\
    =&\ \frac{1}{2}\Vert a\Vert^2-2\bar{K}^2\Vert b\Vert^2.
\end{align*}
Hence we get
\begin{align*}
    \begin{pmatrix}
    a^\top & b^\top
    \end{pmatrix} \tilde{D}_{K_t}\begin{pmatrix}
    a\\b
    \end{pmatrix}
    &\ge \sigma_{\text{min}}(D_{K_t})\Vert a-K_t^\top b\Vert^2+\sigma^2\Vert b\Vert^2\\
    &\ge \sigma_{\text{min}}(D_{K_t})(\frac{1}{2}\Vert a\Vert^2-2\bar{K}^2\Vert b\Vert^2)+\sigma^2\Vert b\Vert^2\\
    &\ge \min \{\sigma_{\min} (D_0),\frac{\sigma^2}{4\bar{K}^2} \}(\frac{1}{2}\Vert a\Vert^2-2\bar{K}^2\Vert b\Vert^2)+\sigma^2\Vert b\Vert^2\\
    &\ge \min \{\frac{\sigma_{\min} (D_0)}{2},\frac{\sigma^2}{8\bar{K}^2},\frac{\sigma^2}{2} \}(\Vert a\Vert^2+\Vert b\Vert^2).
\end{align*}
Thus we have 
\begin{align*}
    \sigma_{\text{min}}(\tilde{D}_{K_t})\ge \min \{\frac{\sigma_{\min} (D_0)}{2},\frac{\sigma^2}{8\bar{K}^2},\frac{\sigma^2}{2} \}>0,
\end{align*}
which further implies
\begin{align*}
    \Vert A^{-1}_{K_t}\Vert &\leq \frac{1}{2(1-\lambda^2)\sigma^2_{\min}(\tilde{D}_{K_t})}\\
    &\leq \frac{1}{2(1-\lambda^2)(\min \{\frac{\sigma_{\min} (D_0)}{2},\frac{\sigma^2}{8\bar{K}^2},\frac{\sigma^2}{2} \})^2}.
\end{align*}
We define
\begin{align*}
    \mu:=2(1-\lambda^2)(\min \{\frac{\sigma_{\min} (D_0)}{2},\frac{\sigma^2}{8\bar{K}^2},\frac{\sigma^2}{2} \})^2
\end{align*}
such that we get
\begin{align*}
    \sigma_{\text{min}}(A_{K_t})\ge \mu,
\end{align*}
which concludes the proof.
\end{proof}
\noindent \textbf{Proof of lemma \ref{p3}}:
\begin{proof}
\begin{align*}
    \Vert\omega^\ast_t-\omega^\ast_{t+1}\Vert
    =& \Vert \text{svec}(\Omega_{K_t}-\Omega_{K_{t+1}})\Vert \\
    =&\Vert \Omega_{K_t}-\Omega_{K_{t+1}}\Vert_{\text{F}}\\
=&\Vert \begin{bmatrix}
A^\top (P_{K_t}-P_{K_{t+1}}) A & A^\top (P_{K_t}-P_{K_{t+1}})B \\ B^\top (P_{K_t}-P_{K_{t+1}}) A & B^\top (P_{K_t}-P_{K_{t+1}})B
\end{bmatrix}\Vert_{\text{F}}\\
=&\Vert A^\top (P_{K_t}-P_{K_{t+1}}) A\Vert_{\text{F}}+\Vert A^\top (P_{K_t}-P_{K_{t+1}}) B\Vert_{\text{F}}+\Vert B^\top (P_{K_t}-P_{K_{t+1}}) A\Vert_{\text{F}} + \Vert B^\top (P_{K_t}-P_{K_{t+1}}) B\Vert_{\text{F}}\\
\leq& d^{\frac{3}{2}}(\Vert A\Vert+\Vert B\Vert)^2\Vert P_{K_t}-P_{K_{t+1}}\Vert\\
\leq& 6d^{\frac{3}{2}}(\Vert A\Vert+\Vert B\Vert)^2\sigma_{\text{min}}^{-1}(D_0)\Vert D_{K_t}\Vert \Vert K_t\Vert \Vert R\Vert(\Vert K_t\Vert \Vert B\Vert \Vert A-BK_t\Vert+\Vert K_t\Vert \Vert B\Vert +1)\Vert K_{t+1}-K_t\Vert\\
\leq& 6c_1d^{\frac{3}{2}}(\Vert A\Vert+\Vert B\Vert)^2 \sigma_{\text{min}}^{-1}(D_0)\frac{\Vert D_{\sigma}\Vert}{1-(\frac{1+\lambda}{2})^2}\bar{K}\Vert R\Vert(\bar{K}\Vert B\Vert(\Vert A\Vert +\bar{K}\Vert B\Vert+1) +1)\Vert K_{t+1}-K_t\Vert\\
=&l_2 \Vert K_{t+1}-K_t\Vert,
\end{align*}
where
\begin{equation}
    \begin{aligned}
    l_2:=&\ 6c_1d^{\frac{3}{2}}\bar{K}(\Vert A\Vert+\Vert B\Vert)^2 \sigma_{\text{min}}^{-1}(D_0)\frac{\Vert D_{\sigma}\Vert\Vert R\Vert}{1-(\frac{1+\lambda}{2})^2}(\bar{K}\Vert B\Vert(\Vert A\Vert +\bar{K}\Vert B\Vert+1) +1).
\end{aligned}
\end{equation}

\end{proof}
\section{Proof of Auxiliary Lemmas}\label{Supplementary Material3}
The following lemmas are well known and have been established in several papers \cite{yang2019provably,fazel2018global}. We include the proof here only for completeness.

\noindent \textbf{Proof of Lemma \ref{pro1}}:
\begin{proof}
Since we focus on the family of linear-Gaussian policies defined in \eqref{policy}, we have
\begin{align}
    J(K)=&\ \mathbb{E}_{(x,u)}[c(x,u)] \nonumber\\
    =&\ \mathbb{E}_{(x,u)}[x^\top Qx+u^\top Ru] \nonumber\\
    =&\ \mathbb{E}_{(x,u)}[x^\top Qx+(-Kx+\sigma \zeta)^\top R(-Kx+\sigma\zeta)] \nonumber \\
    =&\ \mathbb{E}_{x\sim \rho_K}\mathbb{E}_{\zeta\sim I_k}[x^\top(Q+K^\top RK)x-\sigma x^\top K^\top R\zeta -\sigma\zeta^\top RKx+\sigma^2\zeta^\top R\zeta ] \nonumber \\
    =&\ \mathbb{E}_{x\sim\rho_K}[x^\top (Q+K^\top RK)x]+\sigma^2\text{Tr}(R) \nonumber \\
    =&\ \text{Tr}((Q+K^\top RK)D_K)+\sigma^2\text{Tr}(R). \label{300}
\end{align}
Furthermore, for $K\in\mathbb{R}^{k\times d}$ such that $\rho(AB-K)<1$ and positive definite matrix $S\in\mathbb{R}^{d\times d}$, we define the following two operators
\begin{align}\label{operator}
    \Gamma_K(S)=\sum\limits_{t\ge 0}(A-BK)^tS[(A-BK)^t]^\top, \nonumber \\ \Gamma^\top_K(S)=\sum\limits_{t\ge 0}[(A-BK)^t]^\top S(A-BK)^t.
\end{align}
Hence, $\Gamma_K(S)$ and $\Gamma_K^\top (S)$ satisfy Lyapunov equations
\begin{align}
    &\Gamma_K(S)=S+(A-BK)\Gamma_K(S)(A-BK)^\top, \label{op1}\\
    &\Gamma^\top_K(S)=S+(A-BK)^\top\Gamma^\top_K(S)(A-BK) \label{op2}
\end{align}
respectively. Therefore, for any positive definite matrices $S_1$ and $S_2$, we get
\begin{align*}
    \text{Tr}(S_1\Gamma_K(S_2))&=\sum\limits_{t\ge 0}\text{Tr}(S_1(A-BK)^tS_2[(A-BK)^t]^\top)\\
    &=\sum\limits_{t\ge 0}\text{Tr}([(A-BK)^t]^\top S_1(A-BK)^tS_2)\\
    &=\text{Tr}(\Gamma_K^\top(S_1)S_2).
\end{align*}
Combining \eqref{lyap1}, \eqref{lyap2}, \eqref{op1} and \eqref{op2}, we know that
\begin{align}\label{operatorpk}
    D_K=\Gamma_K(D_{\sigma}), \quad P_K=\Gamma^\top_K(Q+K^\top RK).
\end{align}
Thus \eqref{300} implies
\begin{align*}
    J(K)&=\text{Tr}((Q+K^\top RK)D_K)+\sigma^2\text{Tr}(R)\\
    &=\text{Tr}((Q+K^\top RK)\Gamma_K(D_{\sigma}))+\sigma^2\text{Tr}(R)\\
    &=\text{Tr}(\Gamma_K^\top (Q+K^\top RK)D_{\sigma})+\sigma^2\text{Tr}(R)\\
    &=\text{Tr}(P_KD_{\sigma})+\sigma^2\text{Tr}(R).
\end{align*}
It remains to establish the gradient of $J(K)$. Based on \eqref{300}, we have
\begin{align*}
    \nabla_K J(K)=&\ \nabla_K \text{Tr}((Q+K^\top RK)C))|_{C=D_K}+\nabla_K\text{Tr}(CD_K)|_{C=Q+K^\top RK},
\end{align*}
where we use $C$ to denote that we compute the gradient with respect to $K$ and then substitute the expression of $C$. Hence we get
\begin{align}\label{302}
    \nabla_K J(K)=2RKD_K+\nabla_K \text{Tr}(C_0D_K)|_{C_0=Q+K^\top RK}.
\end{align}
Furthermore, we have
\begin{align*}
    \nabla_K \text{Tr}(C_0D_K)
    =&\ \nabla_K \text{Tr}(C_0\Gamma_K(D_{\sigma}))\\
    =&\ \nabla_K \text{Tr}(C_0D_{\sigma}+C_0(A-BK)\Gamma_K(D_{\sigma})(A-BK)^\top)\\
    =&\ \nabla_K \text{Tr}(C_0D_{\sigma})+\nabla_K \text{Tr}((A-BK)^\top C_0(A-BK)\Gamma_K(D_{\sigma}))\\
    =&\ -2B^\top C_0(A-BK)\Gamma_K(D_{\sigma})+\nabla_K \text{Tr}(C_1 \Gamma_K(D_{\sigma}))|_{C_1=(A-BK)^\top C_0(A-BK)}.
\end{align*}
Then it reduces to compute $\nabla_K \text{Tr}(C_1 \Gamma_K(D_{\sigma}))|_{C_1=(A-BK)^\top C_0(A-BK)}$. Applying this iteration for $n$ times, we get
\begin{equation}
    \begin{aligned}\label{301}
    \nabla_K \text{Tr}(C_0D_K)=-2B^\top\sum\limits_{t=0}^nC_t(A-BK)\Gamma_K(D_{\sigma})+\nabla_K \text{Tr}(C_n \Gamma_K(D_{\sigma}))|_{C_n=[(A-BK)^n]^\top C_0(A-BK)^n}.
\end{aligned}
\end{equation}
Meanwhile, by Lyapunov equation defined in \eqref{lyap3}, we have
\begin{align*}
    \sum\limits_{t=0}^{\infty}C_t=&\sum\limits_{t=0}^{\infty}[(A-BK)^t]^\top (Q+K^\top RK)(A-BK)^t
    =P_K.
\end{align*}
Since $\rho(A-BK)<1$, we further get
\begin{align*}
    \lim\limits_{n\to\infty}\text{Tr}(C_n\Gamma_K(D_{\sigma}))
    \leq \lim\limits_{n\to\infty}\Vert (Q+K^\top RK)\Vert \rho(A-BK)^{2n}\text{Tr}(\Gamma_K(D_{\sigma}))= 0.
\end{align*}
Thus by letting $n$ go to infinity in \eqref{301}, we get
\begin{align*}
    \nabla_K \text{Tr}(C_0D_K)|_{C_0=Q+K^\top RK}
    =-2B^\top P_K(A-BK)\Gamma_K(D_{\sigma})\\
    =-2B^\top P_K(A-BK)D_K.
\end{align*}
Hence, combining \eqref{302}, we have
\begin{align*}
    \nabla_K J(K)&=2RKD_K-2B^\top P_K(A-BK)D_K\\
    &=2[(R+B^\top P_KB)K-B^\top P_KA]D_K,
\end{align*}
which concludes our proof.
\end{proof}
\noindent \textbf{Proof of Lemma \ref{pro2}}:
\begin{proof}
By definition, we have the state-value function as follows
\begin{align}
    V_{\theta}(x):&=\sum\limits_{t=0}^{\infty}\mathbb{E}_{\theta}[(c(x_t,u_t)-J(\theta))|x_0=x]\nonumber \\
    &=\mathbb{E}_{u\sim\pi_{\theta}(\cdot|x)}[Q_{\theta}(x,u)],
\end{align}
Therefore, we have
\begin{align}
V_K(x)\
    =&\sum\limits_{t=0}^{\infty}\mathbb{E}[c(x_t,u_t)-J(K)|x_0=x,u_t=-Kx_t+\sigma \zeta_t]\nonumber \\
    =&\sum\limits_{t=0}^{\infty}\mathbb{E}\{[x_t^\top (Q+K^\top RK)x_t]+\sigma^2\text{Tr}(R)-J(K)\}.\label{303}
\end{align}
Combining the linear dynamic system in \eqref{eq:6} and the form of \eqref{303}, we see that $V_K(x)$ is a quadratic function, which can be denoted by
\begin{align*}
    V_K(x)=x^\top P_Kx+C_K,
\end{align*}
where $P_K$ is defined in \eqref{lyap3} and $C_K$ only depends on $K$. Moreover, by definition, we know that $\mathbb{E}_{x\sim\rho_K}[V_K(x)]=0$, which implies
\begin{align*}
    \mathbb{E}_{x\sim\rho_K}[x^\top P_Kx+C_K]=\text{Tr}(P_KD_K)+C_K=0.
\end{align*}
Thus we have $C_K=-\text{Tr}(P_KD_K)$. Hence, the expression of $V_K(x)$ is given by
\begin{align*}
    V_K(x)=x^\top P_K x-\text{Tr}(P_KD_K).
\end{align*}
Therefore, the action-value function $Q_K(x,u)$ can be written as
\begin{align*}
    Q(x,u)=&\ c(x,u)-J(K)+\mathbb{E}[V_K(x')|x,u]\\
    =&\ c(x,u)-J(K)+(Ax+Bu)^\top P_K(Ax+Bu)+\text{Tr}(P_KD_0)-\text{Tr}(P_KD_K)\\
    =&\ x^\top Qx+u^\top Ru+(Ax+Bu)^\top P_K(Ax+Bu)-\sigma^2\text{Tr}(R+P_KBB^\top)-\text{Tr}(P_K\Sigma_K).
\end{align*}
Thus we finish the proof.
\end{proof}
% \noindent \textbf{Proof of \Cref{l2}}:
% \begin{proof}
%     See Lemma 5.7 in \cite{yang2019provably} for a detailed proof.
    % For the operator in \ref{operator}, we define the induced norm as follows
    % \begin{align*}
    %     \Vert \Gamma_K\Vert =\mathop{\text{sup}}_{X}\frac{\Vert \Gamma_K(X)\Vert}{\Vert X\Vert},
    % \end{align*}
    % where the supremum is taken over all symmetric matrices whose $\ell_2$ norm is non-zero. For any unit norm vector $v$ and unit $\ell_2$-norm $S$, we have
    % \begin{align*}
    %     v^\top(\Gamma_K(S))v&=\sum\limits_{t\ge 0}v^\top (A-BK)^t S[(A-BK)^t]^\top v\\
    %     &=\sum\limits_{t\ge 0}\text{Tr}([(A-BK)^t]^\top vv^\top (A-BK)^tS)\\
    %     &=\sum\limits_{t\ge 0}\text{Tr}([D_{\sigma}^{\frac{1}{2}}(A-BK)^t]^\top vv^\top (A-BK)^tD_{\sigma}^{\frac{1}{2}}D_{\sigma}^{-\frac{1}{2}}SD_{\sigma}^{-\frac{1}{2}})\\
    %     &\leq \sum\limits_{t=0}^{\infty}\text{Tr}([D_{\sigma}^{\frac{1}{2}}(A-BK)^t]^\top vv^\top (A-BK)^tD_{\sigma}^{\frac{1}{2}}])\Vert D_{\sigma}^{-\frac{1}{2}}SD_{\sigma}^{-\frac{1}{2}})\Vert\\
    %     &=\Vert D_{\sigma}^{-\frac{1}{2}}SD_{\sigma}^{-\frac{1}{2}}\Vert (v^\top \Gamma_K(D_{\sigma})v)\\
    %     &\leq \frac{1}{\sigma_{\text{min}}(D_0)} \Vert \Gamma_K(D_{\sigma})\Vert\\
    %     &=\frac{1}{\sigma_{\text{min}}(D_0)} \Vert D_K\Vert.
    % \end{align*}
    % Thus we have
    % \begin{align*}
    %     \Vert \Gamma_K\Vert \leq \frac{\Vert D_K\Vert }{\sigma_{\text{min}}(D_0)}.
    % \end{align*}
    % Haven't finished. Need to make some complementary.
% \end{proof}
\noindent \textbf{Proof of \Cref{l6}}:
\begin{proof}
    By the definition of operator in \eqref{operator} and \eqref{operatorpk}, we have
    \begin{align*}
        x^\top P_{K'}x
        =&\ x^\top \Gamma_{K'}^\top(Q+K'^\top RK')x\\
        =&\ \sum\limits_{t\ge 0}x^\top [(A-BK')^t]^\top (Q+K'^\top RK')(A-BK')^tx.
    \end{align*}
    Hereafter, we define $(A-BK')^tx=x_t'$ and $u_t'=-K'x'_t$. Hence, we further have
    \begin{align*}
        x^\top P_{K'}x=&\ \sum\limits_{t\ge 0}x'^\top_t(Q+K'^\top RK')x'_t\\
        =&\ \sum\limits_{t\ge 0}(x'^\top_tQx'_t+u'^\top_tRu'_t).
    \end{align*}
    Therefore, we get
    \begin{align*}
        x^\top P_{K'}x-x^\top P_Kx
        =&\sum\limits_{t\ge 0}[(x'^\top_tQx'_t+u'^\top Ru'_t)+x'^\top_tP_Kx'_t-x'^\top_tP_Kx'_t]-x'^\top_0P_Kx'_0\\
        =&\sum\limits_{t\ge 0}[(x'^\top_tQx'_t+u'^\top Ru'_t)+x'^\top_{t+1}P_Kx'_{t+1}-x'^\top_tP_Kx'_t]\\
        =&\sum\limits_{t\ge 0}[(x'^\top_t Qx'_t+u'^\top_t Ru'_t)+[(A-BK')x'_t]^\top P_K(A-BK')x'_t-x'_tP_Kx'_t]\\
        =&\sum\limits_{t\ge 0}\{x'^\top_t[Q+(K'-K+K)^\top R(K'-K+K)]x'_t\\
        &\ +x'^\top_t [A-BK-B(K'-K)^\top P_K [A-BK-B(K'-K)]x'_t-x'_tP_Kx'_t\}\\
        =&\sum\limits_{t\ge 0}\{2x_t^\top(K'-K)^\top [(R+B^\top P_KB)K-B^\top P_KA]x'_t\\
        &+x'^\top_t(K'-K)^\top (R+B^\top P_KB)(K'-K)x'_t\}\\
        =&\sum\limits_{t\ge 0}[2x'^\top_t(K'-K)^\top E_Kx'_t+x'^\top_t(K'-K)^\top (R+B^\top P_KB)(K'-K)x'_t].
    \end{align*}
    Define
    \begin{equation}
        \begin{aligned}\label{ak2}
        A_{K, K'}(x)&:=2x^\top(K'-K)^\top E_Kx+x^\top(K'-K)^\top(R+B^\top P_KB)(K'-K)x.
    \end{aligned}
    \end{equation}  
    Then, from the expression of $J(K)$ in \eqref{eq.cost_formula}, we have
    \begin{align*}
        J(K')-J(K)
        =&\mathbb{E}_{x\sim\mathcal{N}(0,D_{\sigma})}[x^\top (P_{K'}-P_K)x] \\
        =&\mathbb{E}_{x'_0\sim\mathcal{N}(0,D_{\sigma})}\sum\limits_{t\ge 0}A_{K,K'}(x_t) \\
        =&\mathbb{E}_{x'_0\sim\mathcal{N}(0,D_{\sigma})}\sum\limits_{t\ge 0}[2x'^\top_t(K'-K)^\top E_Kx'_t+x'^\top_t(K'-K)^\top (R+B^\top P_KB)(K'-K)x'_t] \nonumber \\
        =&\text{Tr}(2\mathbb{E}_{x'_0\sim\mathcal{N}(0,D_{\sigma})}[\sum\limits_{t\ge 0}x'^\top_tx'_t](K'-K)^\top E_K)\\
        &+\text{Tr}(\mathbb{E}_{x'_0\sim\mathcal{N}(0,D_{\sigma})}[\sum\limits_{t\ge 0}x'^\top_tx'_t](K'-K)^\top(R+B^\top P_K B)(K'-K))\nonumber \\
        =&-2\text{Tr}(D_{K'}(K-K')^\top E_K)+\text{Tr}(D_{K'}(K-K')^\top (R+B^\top P_K B)(K-K')). \nonumber 
    \end{align*}
    where the last equation is due to the fact that
    \begin{align*}
        \mathbb{E}_{x'_0\sim\mathcal{N}(0,D_{\sigma})}[\sum\limits_{t\ge 0}x'_t(x'_t)^\top]=\mathbb{E}_{x\sim\mathcal{N}(0,D_{\sigma})}\{ \sum\limits_{t\ge 0}(A-BK')^txx^\top [(A-BK')^t]^\top \}=\Gamma_{K'}(D_{\sigma})=D_{K'}.
    \end{align*}
    Hence, we finish our proof.
\end{proof}
\noindent \textbf{Proof of \Cref{lem:l7}}:
\begin{proof}
    By definition of $A_{K,K'}$ in \eqref{ak2}, we have
    \begin{align*}
        A_{K, K'}(x)
        =&2x^\top(K'-K)^\top E_Kx+x^\top(K'-K)^\top (R+B^\top P_KB)(K'-K)x\\
        =&\text{Tr}(xx^\top [K'-K+(R+B^\top P_K B)^{-1}E_K]^\top(R+B^\top P_KB)[K'-K +(R+B^\top P_K B)^{-1}E_K])\\
        &-\text{Tr}(xx^\top E_K^\top (R+B^\top P_K B)^{-1}E_K)\\
        \ge& -\text{Tr}(xx^\top E_K^\top (R+B^\top P_K B)^{-1}E_K),
    \end{align*}
    where the equality is satisfied when $K'=K-(R+B^\top P_K B)^{-1}E_K$. Therefore, we have
    \begin{align*}
        J(K)-J(K^\ast)&=-\mathbb{E}_{x'_0\sim\mathcal{N}(0,D_{\sigma})}\sum\limits_{t\ge 0}A_{K,K^\ast}(x_t)\\
        &\leq \text{Tr}(D_{K^\ast}E_K^\top (R+B^\top P_K B)^{-1}E_K)\\
        &\leq \Vert D_{K^\ast}\Vert \text{Tr}(E_K^\top (R+B^\top P_K B)^{-1}E_K)\\
        &\leq \Vert D_{K^\ast}\Vert \Vert (R+B^\top P_K B)^{-1}\Vert \text{Tr}(E_K^\top E_K)\\
        &\leq \frac{1}{\sigma_{\text{min}}(R)}\Vert D_{K^\ast}\Vert\text{Tr}(E_K^\top E_K).
    \end{align*}
    Thus we complete the proof of upper bound.
    
    It remains to establish the lower bound. Since the equality is attained at $K'=K-(R+B^\top P_K B)^{-1}E_K$, we choose this $K'$ such that
    \begin{align*}
        J(K)-J(K^\ast)&\ge J(K)-J(K')\\
        &=-\mathbb{E}_{x'_0\sim\mathcal{N}(0,D_{\sigma})}[\sum\limits_{t\ge 0}A_{K,K'}(x'_t)]\\
        &=\text{Tr}(D_{K'}E^\top_K(R+B^\top P_K B)^{-1}E_K)\\
        &\ge \sigma_{\text{min}}(D_0)\Vert R+B^\top P_K B\Vert ^{-1}\text{Tr}(E_K^\top E_K).
    \end{align*}
    Overall, we have
    \begin{align*}
        J(K)-J(K^\ast)\leq \frac{1}{\sigma_{\text{min}}(R)}\Vert D_{K^\ast}\Vert\text{Tr}(E_K^\top E_K),
    \end{align*}
    which concludes our proof.
\end{proof}
\newpage
\section{Experimental details}\label{exper}

\begin{example}
\label{ex.1}
Consider a two-dimensional system with 
\begin{align*}
    \text{A}=\begin{bmatrix}
    0 & 1\\ 1 & 0
    \end{bmatrix}, \text{B}=\begin{bmatrix}
    0 & 1\\ 1 & 0
    \end{bmatrix}, \text{Q}=\begin{bmatrix}
    9 & 2\\ 2 & 1
    \end{bmatrix}, \text{R}=\begin{bmatrix}
    1 & 2\\ 2 & 8
    \end{bmatrix}.
\end{align*}
\end{example}

\begin{example}
\label{ex.2}
Consider a four-dimensional system with 
\begin{align*}
    &\text{A}=\begin{bmatrix}
    0.2 & 0.1 & 1 & 0\\ 0.2 & 0.1 & 0.1 & 0\\ 0 & 0.1 & 0.5 & 0\\ 0 & 0 & 0 & 0.5
    \end{bmatrix}, \text{B}=\begin{bmatrix}
    0.3 & 0 & 0\\ 0.2 & 0 & 0.3\\ 1 & 1 & 0.3 \\ 0.3 & 0.1 & 0.1
    \end{bmatrix},\\
    &\text{Q}=\begin{bmatrix}
    1 & 0 & 0.2 & 0\\ 0 & 1 & 0.1 & 0\\ 0.2 & 0.1 & 1 & 0.1 \\ 0 & 0 & 0.1 & 1
    \end{bmatrix}, \text{R}=\begin{bmatrix}
    1 & 0.1 & 1\\ 0.1 & 1 & 0.5 \\ 1 & 0.5 & 2
    \end{bmatrix}.
\end{align*}
\end{example}

We compare our considered single-sample single-timescale AC with two other baseline algorithms that have been analyzed in the state-of-the-art theoretical works: the zeroth-order method~\cite{fazel2018global} (listed in~\Cref{alg2}) and the double loop AC~\cite{yang2019provably} (listed in~\Cref{alg3} on the next page). 

For the considered single-sample single-timescale AC, we set for both examples $\alpha_t=\frac{0.005}{\sqrt{T}}, \beta_t=\frac{0.01}{\sqrt{T}}, \gamma_t=\frac{0.1}{\sqrt{T}}, \sigma=1, T=10^6$. Note that multiplying small constants to these stepsizes does not affect our theoretical results.

For the zeroth-order method proposed in~\cite{fazel2018global}, we set $z=5000, l=20, r=0.1$, stepsize $\eta=0.01$ and iteration number $J=1000$ for the first numerical example; while in the second example, we set $z=20000, l=50, r=0.1, \eta=0.01, J=1000$. We choose different parameters based on the trade-off between better performance and fewer sample complexity.

For the double loop AC proposed in \cite{yang2019provably}, we set for both examples $\alpha_t=\frac{0.01}{\sqrt{1+t}}, \sigma=0.2, \eta=0.05$, inner-loop iteration number $T=500000$ and outer-loop iteration number $J=100$. We note that the algorithm is fragile and sensitive to the practical choice of these parameters. Moreover, we found that it is difficult for the algorithm to converge without an accurate critic estimation in the inner-loop. In our implementation, we have to set the inner-loop iteration number to $T=500000$ to barely get the algorithm converge to the global optimum. This nevertheless demands a significant amount of computation. Higher $T$ iterations can yield more accurate critic estimation, and consequently more stable convergence, but at a price of even longer running time. We run the outer-loop for 100 times for each run of the algorithm. We run the whole algorithm 10 times independently to get the results shown in Figure. With parallel computing implementation, it takes more than 2 weeks on our desktop workstation (Intel Xeon(R) W-2225 CPU @ 4.10GHz $\times$ 8) to finish the computation. In comparison, it takes about 0.5 hour to run the single-sample single-timescale AC and 5 hours for the zeroth-order method.
\newpage
\begin{algorithm}[H]
\caption{Zeroth-order Natural Policy Gradient}
\label{alg2}
\begin{algorithmic}
   \STATE Input: stabilizing policy gain $K_0$ such that $\rho(A-BK_0)<1$, number of trajectories $z$, roll-out length $l$, perturbation amplitude $r$, stepsize $\eta$
   \WHILE{updating current policy}
   \STATE \textbf{Gradient Estimation:}
   \FOR{$i=1,\cdots, z$}
   \STATE Sample $x_0$ from $\mathcal{D}$
   \STATE Simulate $K_{j}$ for $l$ steps starting from $x_0$ and observe $y_0, \cdots, y_{l-1}$ and $c_0, \cdots, c_{l-1}$. 
   \STATE Draw $U_i$ uniformly over matrices such that $\|U_i\|_F=1$, and generate a policy $K_{j,U_i}=K_j+rU_i$.
   \STATE Simulate $K_{j,U_i}$ for $l$ steps starting from $x_0$ and observe $c_0', \cdots, c_{l-1}'$.
   \STATE Calculate empirical estimates:
   \begin{equation}
   \nonumber
   \widehat{J_{K_{j}}^i}=\sum_{t=0}^{l-1} c_t,\; \widehat{\mathcal{L}_{K_{j}}^i}=\sum_{t=0}^{l-1} y_ty_t^\top,\; \widehat{J_{K_{j,U_i}}}=\sum_{t=0}^{l-1} c_t'.
   \end{equation}
   \ENDFOR
\STATE Return estimates:
   \begin{equation}
   \nonumber
    \widehat{\nabla J(K_j)} = \frac{1}{z}\sum_{i=1}^z\frac{\widehat{J_{K_{j,U_i}}}-\widehat{J_{K_{j}}^i}}{r}U_i,\; \widehat{\mathcal{L}_{K_j}}=\frac{1}{z}\sum_{i=1}^z\widehat{\mathcal{L}_{K_{j}}^i}.
   \end{equation}
    \STATE \textbf{Policy Update:}
    %\STATE Vanilla policy gradient $K_{j+1}=K_j-\eta \widehat{\nabla J(K_j)}.$
    \STATE $K_{j+1}=K_j-\eta \widehat{\nabla J(K_j)}\widehat{\mathcal{L}_{K_j}}^{-1}.$
    \STATE $j=j+1$.
\ENDWHILE
\end{algorithmic}
\end{algorithm}

\begin{algorithm}[H]
\caption{Double-loop Natural Actor-Critic}
\label{alg3}
\begin{algorithmic}
   \STATE Input: Initial policy $\pi_{K_0}$ such that $\rho(A-BK_0)<1$, stepsize $\gamma$ for policy update. 
   \WHILE{updating current policy}
   \STATE \textbf{Gradient Estimation:}
   \STATE Initialize the primal and dual variables by $v_0\in\mathcal{X}_{\Theta}$ and $\omega_0\in\mathcal{X}_{\Omega}$, respectively.
   \STATE Sample the initial state $x_0\in\mathbb{R}^d$ from stationary distribution $\rho_{K_j}$. Take action $u_0\sim \pi_{K_j}(\cdot|x_0)$ and obtain the reward $c_0$ and the next state $x_1$.
   \FOR{$i=1,2,\cdots, T$}
   \STATE Take action $u_t$ according to policy $\pi_{K_j}$, observe the reward $c_t$ and the next state $x_{t+1}$.
   \STATE $\delta_t=v^1_{t-1}-c_{t-1}+[\phi(x_{t-1},u_{t-1})-\phi(x_t,u_t)]^\top v^2_{t-1}$.
   \STATE $v^1_{t}=v^1_{t-1}-\alpha_t[\omega^1_{t-1}+\phi(x_{t-1},u_{t-1})^\top \omega^2_{t-1}]$.
   \STATE $v^2_{t}=v^2_{t-1}-\alpha_t[\phi(x_{t-1},u_{t-1})-\phi(x_t,u_t)]\cdot \phi(x_{t-1},u_{t-1})^\top\omega^2_{t-1}$.
   \STATE $\omega^1_t=(1-\alpha_t)\omega_t^1+\alpha_t(v^1_{t-1}-c_{t-1})$.
   \STATE $\omega^2_t=(1-\alpha_t)\omega^2_t+\alpha_t\delta_t\phi(x_{t-1},u_{t-1})$.
   \STATE Project $v_t$ and $\omega_t$ to $v_0\in\mathcal{X}_{\Theta}$ and $\omega_0\in\mathcal{X}_{\Omega}$.
   \ENDFOR
\STATE Return estimates:
   \begin{equation}
   \nonumber
    \widehat{v}^2 = (\sum\limits_{t=1}^T\alpha_tv^2_t)/(\sum\limits_{t=1}^T\alpha_t),\; \widehat{\Theta}=\text{smat}(\widehat{v}^2).
   \end{equation}
    \STATE \textbf{Policy Update:}
    %\STATE Vanilla policy gradient $K_{j+1}=K_j-\eta \widehat{\nabla J(K_j)}.$
    \STATE $K_{j+1}=K_j-\eta(\widehat{\Theta}^{22}K_j-\widehat{\Theta}^{21}).$
    \STATE $j=j+1$.
\ENDWHILE
\end{algorithmic}
\end{algorithm}

\end{document}